%% file: main.tex
\documentclass{article}

\input{packages}
\input{macros}
\input{environments}

\title{Learning on the Edge: Online Learning with Stochastic Feedback Graphs}

\author{%
  Emmanuel Esposito\thanks{Equal contribution.} \\
  Dept.\ of Computer Science\\
  Università degli Studi di Milano, Italy\\
  \& Istituto Italiano di Tecnologia, Italy\\
  \texttt{emmanuel@emmanuelesposito.it} \\
  \And
  Federico Fusco\samethanks \\
  Dept.\ of Computer, Control \\
  and Management Engineering \\
  Sapienza Università di Roma, Italy \\
  \texttt{fuscof@diag.uniroma1.it} \\
  \And
  Dirk van der Hoeven\samethanks \\
  Dept.\ of Computer Science\\
  Università degli Studi di Milano, Italy\\
  \texttt{dirk@dirkvanderhoeven.com} \\
  \And
  Nicolò Cesa-Bianchi \\
  Dept.\ of Computer Science\\
  Università degli Studi di Milano, Italy\\
  \texttt{nicolo.cesa-bianchi@unimi.it} \\
}

\begin{document}

\maketitle
\setcounter{footnote}{0}

\begin{abstract}
    The framework of feedback graphs is a generalization of sequential decision-making with bandit or full information feedback. In this work, we study an extension where the directed feedback graph is stochastic, following a distribution similar to the classical Erdős-Rényi model. Specifically, in each round every edge in the graph is either realized or not with a distinct probability for each edge. We prove nearly optimal regret bounds of order $\min\bigl\{\min_{\e} \sqrt{(\alpha_\e/\e) T},\, \min_{\e} (\delta_\e/\e)^{1/3} T^{2/3}\bigr\}$ (ignoring logarithmic factors), where $\alpha_{\e}$ and $\delta_{\e}$ are graph-theoretic quantities measured on the support of the stochastic feedback graph $\scG$ with edge probabilities thresholded at~$\e$. Our result, which holds without any preliminary knowledge about $\scG$, requires the learner to observe only the realized out-neighborhood of the chosen action.
    When the learner is allowed to observe the realization of the entire graph (but only the losses in the out-neighborhood of the chosen action), we derive a more efficient algorithm featuring a dependence on weighted versions of the independence and weak domination numbers that exhibits improved bounds for some special cases.
\end{abstract}

\section{Introduction}

    In this work we study an online learning framework for decision-making with partial feedback. In each decision round, the learner chooses an action in a fixed set and is charged a loss. In our setting, the loss of any action in all decision rounds is preliminarily chosen by an adversary, but the feedback received by the learner at the end of each round $t$ is stochastic. More specifically, the loss of each action $i$ (including $I_t$, the one selected by the learner at round $t$) is independently observed with a certain probability $p(I_t,i)$, where the probabilities $p(i,j)$ for all pairs $i,j$ are fixed but unknown.
    
    This feedback model can be viewed as a stochastic version of the feedback graph model for online learning \citep{Mannor2011}, where the feedback received by the learner at the end of each round is determined by a directed graph defined over the set of actions. In this model, the learner deterministically observes the losses of all the actions in the out-neighborhood of the action selected in that round. In certain applications, however, deterministic feedback is not realistic. Consider for instance a sensor network for monitoring the environment, where the learner can decide which sensor to probe in order to maximize some performance measure. Each probed sensor may also receive readings of other sensors, but whether a sensor successfully transmits to another sensor depends on a number of environmental factors, which include the position of the two sensors, but also their internal state (e.g., battery levels) and the weather conditions. Due to the variability of some of these factors, the possibility of reading from another sensor can be naturally modeled as a stochastic event.
    
    Online learning with adversarial losses and stochastic feedback graphs has been studied before, but under fairly restrictive assumptions on the probabilities $p(i,j)$.
    Let $\scG$ be a stochastic feedback graph, represented by its probability matrix $p(i,j)$ for $i,j \in V$ where $V$ is the action set.
    When $p(i,j) = \e$ for all distinct $i,j\in V$ and for some $\e > 0$, then $\scG$ follows the Erdős-Rényi random graph model. Under the assumption that $\e$ is known and $p(i,i) = 1$ for all $i\in V$ (all self-loops occur w.p.\ $1$), \citet{AlonCGMMS17} show that the optimal regret after $T$ rounds is of order $\sqrt{T/\e}$, up to logarithmic factors. This result has been extended by \citet{Kocak2016}, who prove a regret bound of order $\sqrt{\sum_t (1/\e_t)}$ when the parameter $\e_t$ of the random graph is unknown and allowed to change over time. However, their result holds only under rather strong assumptions on the sequence $\e_t$ for $t \ge 1$. 
    In a recent work, \citet{Ghari2021} show a regret bound of order $(\alpha/\e)\sqrt{KT}$, ignoring logarithmic factors, when each (unknown) probability $p(i,j)$ in $\scG$ is either zero or at least $\e$ for some known $\e > 0$, and all self-loops $(i,i)$ have probability $p(i,i) \ge \e$. Here $\alpha$ is the independence number (computed ignoring edge orientations) of the support graph $\supp{\scG}$; i.e., the directed graph with adjacency matrix $A(i,j) = \ind{\{p(i,j) > 0\}}$.
    Their bound holds under the assumption that $\supp{\scG}$ is preliminarily known to the learner.
    
    Our analysis does away with a crucial assumption that was key to prove all previous results. Namely, we do not assume any special property of the matrix $\scG$, and we do not require the learner to have any preliminary knowledge of this matrix.
    The fact that positive edge probabilities are not bounded away from zero implies that the learner must choose a threshold $\e \in (0,1]$ below which the edges are deemed to be too rare to be exploitable for learning. If $\e$ is too small, then waiting for rare edges slows down learning. On the other hand, if $\e$ is too large, then the feedback becomes sparse and the regret increases.
    
    To formalize the intuition of rare edges, we introduce the notion of thresholded graph $\supp{\scG_{\e}}$ for any $\e > 0$. This is the directed graph with adjacency matrix $A(i,j) = \ind{\{p(i,j) \ge \e\}}$. As the thresholded graph is a deterministic feedback graph $G$, we can refer to \citet{AlonCDK15} for a characterization of minimax regret $R_T$ based on whether $G$ is not observable ($R_T$ of order $T$), weakly observable ($R_T$ of order $\delta^{1/3}T^{2/3}$), or strongly observable ($R_T$ of order $\sqrt{\alpha T}$).\footnote{All these rates ignore logarithmic factors.} Here $\alpha$ and $\delta$ are, respectively, the independence and the weak domination number of $G$; see \Cref{s:setting} for definitions.
    Let $\alpha_{\e}$ and $\delta_{\e}$ respectively denote the independence number and the weak domination number of $\supp{\scG_{\e}}$. As $\alpha_{\e}$ and $\delta_{\e}$ both grow when $\e$ gets larger, the ratios $\alpha_{\e}/\e$ and $\delta_{\e}/\e$ capture the trade-off involved in choosing $\e$. We define the optimal values for $\e$ as follows:
        \begin{align}
        \label{def:eps_s}    
            \e^*_s &= \argmin_{\e \in (0,1]} \left\{\frac{\alpha_{\e}}{\e} \,:\, \text{$\supp{\scG_{\e}}$ is strongly observable}\right\}, \\ 
        \label{def:eps_w}
            \e^*_w &= \argmin_{\e \in (0,1]} \left\{\frac{\delta_{\e}}{\e} \,:\, \text{$\supp{\scG_{\e}}$ is observable} \right\}.
        \end{align}
    We adopt the convention that the minimum of an empty set is infinity and the relative $\argmin$ is set to $0$. The $\argmin$ are well defined: there are at most $K^2$ values of $\e$ for which the support of $\scG_\e$ varies, and the minimum is attained in one of these values. 
    For simplicity, we let $\alpha^* = \alpha_{\e^*_s}$ and $\delta^* = \delta_{\e^*_w}$. Our first result can be informally stated as follows.
    \begin{theorem}[Informal]
    \label{th:informal}
    Consider the problem of online learning with an unknown stochastic feedback graph $\scG$ on $T$ time steps. 
    If $\supp{\scG_{\e}}$ is not 
    observable for $\e = \tilde\Theta(K^3/T)$, then any learning algorithm suffers regret linear in $T$.
    Otherwise, there exists an algorithm whose regret satisfies (ignoring polylog factors in $K$ and $T$)
    \[
        R_T
    \le
        \min\left\{
            \sqrt{\frac{\alpha^*}{\e_s^*}T}, 
            \left(\frac{\delta^*}{\e_w^*}\right)^{1/3} T^{2/3}
        \right\}.
    \]
    This bound is tight (up to polylog factors). 
    \end{theorem}
    This result shows that, without any preliminary knowledge of $\scG$, we can obtain a bound that optimally trades off between the strongly observable rate $\sqrt{({\alpha^*}/{\e_s^*})T}$, for the best threshold $\e$ for which $\supp{\scG_{\e}}$ is strongly observable, and the (weakly) observable rate $({\delta^*}/{\e_w^*})^{1/3} T^{2/3}$, for the best threshold $\e$ for which $\supp{\scG_{\e}}$ is (weakly) observable.
    Note that this result improves on \citet{Ghari2021} bound $({\alpha_\e}/{\e})\sqrt{KT}$, who additionally assume that $\supp{\scG_{\e}}$ and $\e$ (a lower bound on the self-loop probabilities) are both preliminarily available to the learner.
    On the other hand, the algorithm achieving the bound of \Cref{th:informal} need not receive any information (neither prior nor during the learning process) besides the stochastic feedback.
    
    We obtain positive results in \Cref{th:informal} via an elaborate reduction to online learning with deterministic feedback graphs. Our algorithm works in two phases: first, it learns the edge probabilities in a round-robin procedure, then it commits to a carefully chosen estimate of the feedback graph and feeds it to an algorithm for online learning with deterministic feedback graphs. There are two main technical challenges the algorithm faces: on the one hand, it needs to switch from the first to the second phase at the right time in order to achieve the optimal regret. On the other hand, in order for the reduction to work, it needs to simulate the behaviour of a deterministic feedback graph using only feedback from a stochastic feedback graph (with unknown edge probabilities).
    We complement the positive results in \Cref{th:informal} with matching lower bounds that are obtained by a suitable modification of the hard instances in \cite{AlonCDK15,AlonCGMMS17} so as to consider stochastic feedback graphs. 
    
    Our last result is an algorithm that, at the cost of an additional assumption on the feedback (i.e., the learner additionally observes the realization of the entire feedback graph at the end of each round), has regret which is never worse and may be considerably better than the regret of the algorithm in Theorem~\ref{th:informal}. While the bounds in Theorem~\ref{th:informal} are tight up to log factors, we show that the factors ${\alpha^*}/{\e_s^*}$ and ${\delta^*}/{\e_w^*}$ can be improved for specific feedback graphs. Specifically, we design weighted versions of the independence and weak domination numbers, where the weights of a given node depend on the probabilities of seeing the loss of that node. On the technical side, we design a new importance-weighted estimator which uses a particular version of upper confidence bound estimates of the edge probabilities $p(i,j)$, rather than the true edge probabilities, which are unknown. We show that the cost of using this estimator is of the same order as the regret bound achievable had we known $p(i,j)$. Additionally, the algorithm that obtains these improved bounds is more efficient than the algorithm of Theorem~\ref{th:informal}. The improvement in efficiency comes from the following idea: we start with an optimistic algorithm that assumes that the support of $\scG$ is strongly observable and only switches to the assumption that the support of $\scG$ is (weakly) observable when it estimates that the regret under this second assumption is smaller. The algorithm learns which regime is better by keeping track of a bound on the regret of the optimistic algorithm while simultaneously estimating the regret in the (weakly) observable case, which it can do efficiently.

    \paragraph{Additional related work.}
    The problem of adversarial online learning with feedback graphs was introduced by \citet{Mannor2011}, in the special case where all nodes in the feedback graph have self-loops.
    The results of \citet{AlonCDK15} (also based on prior work by \citet{AlonCCM2013,Kocak2014}) have been recently slightly improved by \citet{Chen21}, with tighter constants in the regret bound. Variants of the adversarial setting have been studied by \citet{FengL18,AroraMM19,RangiF19} and \citet{VanderHoeven2021beyond}, who study online learning with feedback graphs and switching costs and online multiclass classification with feedback graphs, respectively.  
    There is also a considerable amount of work in the stochastic setting \citep{LiuBS18,Cortes2019,Li2020}. 
    Finally, \citet{rouyer22bobw} and \citet{ito22bobw} independently designed different best-of-both-worlds learning algorithms achieving nearly optimal (up to polylogarithmic factors in $T$) regret bounds in the adversarial and stochastic settings.
    
    Following \citet{Mannor2011}, we can consider a more general scenario where the feedback graph is not fixed but changes over time, resulting in a sequence $G_1, \dots, G_T$ of feedback graphs.
    \cite{Cohen2016} study a setting where the graphs are adversarially chosen and only the local structure of the feedback graph is observed. They show that, if the losses are generated by an adversary and all nodes always have a self-loop, one cannot do better than $\sqrt{KT}$ regret, and we might as well simply employ a standard bandit algorithm. Furthermore, removing the guarantee on the self-loops induces an $\Omega(T)$ regret. In \Cref{sec:block}, we are in a similar situation, as we also observe only local information about the feedback graph and the losses are generated by an adversary. However, we show that if the graphs are stochastically generated with a strongly observable support for some threshold $\varepsilon$, there is a $\sqrt{\alpha T / \varepsilon}$ regret bound. As a consequence, for $\varepsilon$ not too small, observing only the local information about the feedback graphs is in fact sufficient to obtain better results than in the bandit setting. Similarly, if there are no self-loops in the support but the support is weakly observable, then our regret bounds are sublinear rather than linear in $T$. \cite{AlonCCM2013,AlonCGMMS17} and \cite{Kocak2014} also consider adversarially generated sequences $G_1,G_2,\dots$ of deterministic feedback graphs. In the case of directed feedback graphs, \cite{AlonCCM2013} investigate a model in which $G_t$ is revealed to the learner at the beginning of each round~$t$. \cite{AlonCGMMS17} and \cite{Kocak2014} extend this analysis to the case when $G_t$ is strongly observable and made available only at the end of each round $t$. In comparison, in our setting the graphs (or the local information about the graph) revealed to the learner (at the end of each round) may not even be observable, let alone strongly observable. Despite this seemingly challenging setting for previous works, we nevertheless obtain sublinear regret bounds. Finally, \cite{Kocak2016noisy} study a feedback model where the losses of other actions in the out-neighborhood of the action played are observed with an edge-dependent noise. In their setting, the feedback graphs $G_t$ are weighted and revealed at the beginning of each round. They introduce edge weights $s_t(i,j) \in [0,1]$ that determine the feedback according to the following additive noise model: $s_t(I_t,j)\ell_t(j) + (1-s_t(I_t,j))\xi_t(j)$, where $\xi_t(j)$ is a zero-mean bounded random variable. Hence, if $s_t(i,j)=1$, then $I_t=i$ allows to observe the loss of action $j$ without any noise. If $s_t(i,j)=0$, then only noise is observed. Note that they assume $s_t(i,i)=1$ for each $i$, implying strong observability. Although similar in spirit to our feedback model, our results do not directly compare with theirs.
    
    Further work also takes into account a time-varying probability for the revelation of side-observations \citep{Kocak2016}. While the idea of a general probabilistic feedback graph has been already considered in the stochastic setting \citep{Li2020,Cortes2020}, the recent work by \citet{Ghari2021} seems to be the first one in the adversarial setting that generalizes from the Erdős-R\'enyi model to a more flexible distribution where they allow ``edge-specific'' probabilities.
    We remark, however, that the assumptions of \citet{Ghari2021} exclude some important instances of feedback graphs.
    For example, we cannot hope to employ their algorithm for efficiently solving the revealing action problem (see for example \citep{AlonCDK15}). In a spirit similar to ours, \citet{ReslerM19} studied a version of the problem where the topology of the graph is fixed and known a priori, but the feedback received by the learner is perturbed when traversing edges.

\section{Problem Setting}
\label{s:setting}
A feedback graph over a set $V = [K]$ of actions is any directed graph $G = (V,E)$, possibly with self-loops. For any vertex $i \in V$, we use $\inneigh{G}{i}= \{j \in V : (j, i) \in E\}$ to denote the in-neighborhood of $i$ and $\outneigh{G}{i} = \{j \in V : (i, j) \in E\}$ to denote its out-neighborhood (we may omit the subscript when the graph is clear from the context).
The independence number $\alpha(G)$ of a feedback graph $G$ is the cardinality of the largest subset $S$ of $V$ such that, for all distinct $i,j \in S$, it holds that $(i,j)$ and $(j,i) $ are not in $E$. 
We also use the following terminology for directed graphs $G = (V,E)$ \citep{AlonCDK15}. Any $i \in V$ is: observable if $\inneigh{G}{i} \neq \emptyset$, strongly observable if $i \in \inneigh{G}{i}$ or $V\setminus\{i\} \subseteq \inneigh{G}{i}$, and weakly observable if it is observable but not strongly. The graph $G$ is: observable if all $i\in V$ are observable, strongly observable if all $i\in V$ are strongly observable, and weakly observable if it is observable but not strongly. The weak domination number $\delta(G)$ of $G$ is the cardinality of the smallest subset $S$ of $V$ such that for all weakly observable $i \in V \setminus S$ there exists $j \in S$ such that $(j,i) \in E$.

In the online learning problem with a stochastic feedback graph, an oblivious adversary privately chooses a stochastic feedback graph $\scG$ and a sequence $\ell_1,\ell_2,\ldots$ of loss functions $\ell_t \colon V \to [0,1]$. At each round $t=1,2,\ldots$, the learner selects an action $I_t \in V$ to play and, independently, the adversary draws a feedback graph $G_t$ from $\scG$ (denoted by $G_t \sim \scG)$. The learner then incurs loss $\ell_t(I_t)$ and observes the feedback $\big\{(i, \ell_t(i)) : i \in \outneigh{G_t}{I_t}\big\}$.
In some cases we consider a richer feedback, where at the end of each round $t$ the learner also observes the realized graph $G_t.$ The learner's performance is measured using the standard notion of regret,
\[
   R_T = \max_{k \in V}\Enb\left[\sumT \big(\ell_t(I_t) - \ell_t(k)\big)\right]
\]
where $I_1,\ldots,I_T$ are the actions played by the learner, and the expectation is computed over both the sequence $G_1,\ldots,G_T$ of feedback graphs drawn i.i.d.\ from $\scG$ and the learner's internal randomization.

Fix any stochastic feedback graph $\scG = \{p(i,j) \,:\, i,j \in V\}$. We sometimes use $e$ to denote a pair $(i,j)$, in which case we write $p_e$ to denote the probability $p(i,j)$. When $G_t = (V,E_t)$ is drawn from $\scG$, each pair $(i,j) \in V \times V$ independently becomes an edge (i.e., $(i,j) \in E_t$) with probability $p(i,j)$.
For any $\e > 0$, we define the thresholded version $\scG_{\e}$ of $\scG$ represented by $\{p'(i,j) \,:\, i,j \in V\}$, where $p'(i,j) = p(i,j)\ind{\{p(i,j) \ge \e\}}$.
We also define the support feedback graph of $\scG$ as the graph $\supp{\scG} = (V,E)$ having $E = \{ (i,j) \in V \times V \,:\, p(i,j) > 0\}$. To keep the notation tidy, we write $\alpha(\scG)$ instead of $\alpha(\supp{\scG})$ and similarly for $\delta$.


\section{Block Decomposition Approach}\label{sec:block}
    
    In this section, we present an algorithm for online learning with stochastic feedback graphs via a reduction to online learning with deterministic feedback graphs. Our algorithm \etc (\Cref{alg:explore-commit}) has an initial exploration phase followed by a commit phase. In the exploration phase, the edge probabilities are learned online in a round-robin fashion. A carefully designed stopping criterion then triggers the commit phase, where we feed the support of the estimated stochastic feedback graph to an algorithm for online learning with (deterministic) feedback graphs.
    
    \subsection{Estimating the Edge Probabilities}
    
    As a first step we design a routine, \round (\Cref{alg:round-robin}), that sequentially estimates the stochastic feedback graph until a certain stopping criterion is met. The stopping criterion depends on a nonnegative function $\Phi$ that takes in input a stochastic feedback graph $\scG$ together with a time horizon. Let $\ftau \le T/K$ be the index of the last iteration of the outer for loop in \Cref{alg:round-robin}. We want to make sure that, for all $\tau \le \ftau$, the stochastic feedback graphs $\hat \scG_{\tau}$ are valid estimates of the underlying $\scG$ up to a $\Theta(\e_{\tau})$ precision. To formalize this notion of approximation, we introduce the following definition.
    
    \begin{algorithm}[t]
    \textbf{Environment}: stochastic feedback graph $\scG$, sequence of losses $\ell_1, \ell_2, \dots, \ell_T$\;
    \textbf{Input:} time horizon $T$, stopping function $\Phi$, actions $V = \{1,2,\dots,K\}$\;
    $n_e \gets 0$, for all $e \in V^2$\;
    \For{\textup{each} $\tau =1,2,\dots, \lfloor T/K \rfloor$}{
         \For{\textup{each} $ i=1,2, \dots K$}{
            Play action $i$ and observe $N^\text{out}_{G_t}(i)$ from $G_t \sim \scG$\tcp*{$t$ is the time step}
            $n_e \gets n_e + 1$ for all $e \in N^\text{out}_{G_t}(i)$\;
        }
        $\hat p_e^{\tau} \gets n_e/\tau $ for all edges $e \in V^2$\;
        $\e_{\tau} \gets 60 \ln(KT)/\tau$\;
        $\hat\scG_{\tau} \gets \bigl(V,\{e \in V^2 \,:\, \hat p_e^{\tau} \ge \e_{\tau}\}\bigr)$ with weights $\hat p_e^{\tau}$\tcp*{estimated feedback graph}
        \If{$\Phi(\hat  \scG_{\tau},T) \le \tau K$}{
        \textbf{output} $\hat  \scG_{\tau}$, $\e_{\tau}$\;
        }
    }
    \textbf{output} $\hat \scG_{\tau}$, $\e_{\tau}$\;
    \caption{\round} 
    \label{alg:round-robin}
    \end{algorithm}
    
    \begin{definition}[$\e$-good approximation] \label{def:eps-good-approx}
        A stochastic feedback graph $\hat \scG = \{\hat p_e \,:\, e \in V^2\}$ is an $\e$-good approximation of $\scG = \{p_e \,:\, e \in V^2\}$ for some $\e \in (0,1]$, if the following holds:
        \begin{enumerate}[topsep=0pt,parsep=0pt,itemsep=2pt]
            \item All the edges $e\in \supp{\scG}$ with $p_e \ge 2\e$ belong to $\supp{\hat\scG}$;
            \item For all edges $e\in  \supp{\hat \scG}$ with $p_e \ge \e/2$ it holds that $|\hat p_e - p_e| \le p_e/2$;
            \item No edge $e\in V^2$ with $p_e < \e/2$ belongs to $\supp{\hat\scG}$.   
        \end{enumerate} 
    \end{definition}
    We can now state the following theorem; we defer the proof in \Cref{sec:block-alg-appendix}. The proof follows from an application of the multiplicative Chernoff bound on edge probabilities. 
    
    \begin{restatable}{retheorem}{retheoremroundrobinestimates} \label{thm:round-robin-estimates}
        If \round (\Cref{alg:round-robin}) is run on the stochastic feedback graph $\scG$, then, with probability at least $1-1/T$, the estimate $\hat \scG_{\tau}$ is an $\e_{\tau}$-good approximation of $\scG$ simultaneously for all $\tau \le \ftau$, where $\ftau \le T/K$ is the index of the last iteration of the outer for loop in \Cref{alg:round-robin}.
    \end{restatable}
    
    \subsection{Block Decomposition: Reduction to Deterministic Feedback Graph}
    
    As a second step, we present \block (\Cref{alg:blocks-reduction}) which reduces the problem of online learning with stochastic feedback graph to the corresponding problem with deterministic feedback graph. Surprisingly enough, in order for this reduction to work, we do not need the exact edge probabilities: an $\e$-good approximation is sufficient for this purpose. 
    
    The intuition behind \block is simple: given that each edge $e$ in $\supp{\scG_\e}$ appears in $G_t$ with probability $p_e \ge \e$ at each time step $t$, if we wait for $\Theta \bigl((1/\e) \ln(T)\bigr)$ time steps it will appear at least once with high probability. Applying a union bound over all edges, we can argue that considering $\Delta = \Theta \bigl((1/\e) \ln(KT)\bigr)$ realizations of the stochastic feedback graph, then all the edges in $\supp{\scG_\e}$ are realized at least once with high probability. 
    
    Imagine now to play a certain action $a$ consistently during a block $B_{\tau}$ of $\Delta$ time steps. We want to reconstruct the average loss suffered by $a'$ in $B_{\tau}$:
    \begin{equation} \label{eq:c_tau}
        c_{\tau}(a')=\sum_{t \in B_\tau} \frac{\ell_t(a')}{\Delta} \enspace,
    \end{equation}
    and we want to do it for all $a'$ in the out-neighborhood of $a$. 
    Let $\Delta^\tau_{(a,a')}$ be the number of times that the loss of $a'$ is observed by the learner; i.e., the number of times that $(a,a')$ is realized in the $\Delta$ time steps. With this notation, we can define the natural estimator $\hat c_{\tau}(a')$:
    \begin{equation} \label{eq:blocks-estimator}
        \hat c_{\tau}(a') = \sum_{t \in B_\tau} \ell_t(a') \frac{\ind{\{(a,a') \in E_t\}}}{\Delta^\tau_{(a,a')}} \enspace.
    \end{equation}
    Conditioning on the event $\calE^\tau_{(a,a')}$ that the edge $(a,a')$ in $\hat{G}$ is observed at least once in block $B_\tau$, we show in \Cref{lem:unbiased-estimator-blocks} in \Cref{sec:block-alg-appendix} that $\hat c_{\tau}(a')$ is an unbiased estimator of $c_{\tau}(a')$.
    
    Therefore, we can construct unbiased estimators of the average of the losses on the blocks as if the stochastic feedback graph were deterministic. This allows us to reduce the original problem to that of online learning with deterministic feedback graph on the meta-instance given by the blocks. The details of \block are reported in \Cref{alg:blocks-reduction}, while the theoretical properties are summarized in the next result, whose proof can be found in \Cref{sec:block-alg-appendix}. 
    \begin{algorithm}[t]
    \textbf{Environment}: stochastic feedback graph $\scG$, sequence of losses $\ell_1, \ell_2, \dots, \ell_T$\;
    \textbf{Input}: time horizon $T$, threshold $\e$, estimate $\hat \scG$ of $\scG$, learning algorithm $\A$\;
    $\Delta \gets \lceil\frac{2}{\e} \ln(KT)\rceil$, $N \gets \lfloor T/\Delta \rfloor$, $\hat G\gets \supp{\hat \scG}$\;
    \textbf{Initialize} $\A$ with time horizon $N$ and graph $\hat G$\;
    $B_{\tau} \gets \{({\tau}-1) \Delta + 1, \dots, {\tau} \Delta\}, $ for all ${\tau} = 1, \dots, N$\;
    \For{\textup{each round} ${\tau}=1,2,\dots,N$}{
    Draw action $a_{\tau}$ from the probability distribution over actions output by $\A$\;
    \For{\textup{each round} $t\in B_{\tau}$}{
    Play action $a_{\tau}$ and observe the revealed feedback\tcp*{$G_t \sim \scG$}
    }
    For all $a'\in N_{\hat G}^\text{out}(a_\tau)$, compute $\hat c_{\tau}(a')$ as in (\ref{eq:blocks-estimator}), and feed them to $\A$\;
    }   
    Play arbitrarily the remaining $T - \Delta N$ rounds\;
    \caption{\block} \label{alg:blocks-reduction}
    \end{algorithm}

    \begin{restatable}{retheorem}{retheoremblocksreduction}\label{thm:blocks-reduction}
        Consider the problem of online learning with stochastic feedback graph $\scG$, and let $\hat \scG$ be an $\e$-good approximation of $\scG$.
        Let $\A$ be an algorithm for online learning with arbitrary deterministic feedback graph $G$ with regret bound $R^\A_N(G)$ over any sequence of $N$ losses in $[0,1]$.
        Then, the regret of \block (\Cref{alg:blocks-reduction}) run with input $(T,\e/2,\hat \scG,\A)$ is at most $\Delta R^\A_N\big(\supp{\hat \scG}\big) + \Delta $, where $N = \lfloor T/\Delta \rfloor$ and $\Delta = \lceil\frac{4}{\e} \ln(KT)\rceil$.
    \end{restatable}
    For online learning with deterministic feedback graphs we use the variants of \expg contained in \citet{AlonCDK15}. Together with \Cref{thm:blocks-reduction}, this gives the following corollary; the details of the proof are in \Cref{sec:block-alg-appendix}.

    \begin{restatable}{recor}{recorblock}\label{cor:blocks-reduction}
    Consider the problem of online learning with stochastic feedback graph $\scG$, and let $\hat \scG$ be an $\e$-good approximation of $\scG$ for $\e \ge 1/T$ and with support $\hat G$.
    \begin{itemize}[topsep=0pt,parsep=0pt,itemsep=0.5pt,leftmargin=14pt]
    \item
    If $\hat G$ is strongly observable with independence number $\alpha$, then the regret of \block run with parameter $\e/2$ using \expg for strongly observable graphs as base algorithm $\A$ satisfies:
        $
            R_T \le 4C_s\sqrt{(\alpha/\e)T} \bigl(\ln(KT)\bigr)^{3/2},
        $
    where $C_s > 0$ is a constant in the regret bound of~$\A$.
    \item
    If $\hat G$ is (weakly) observable with weak domination number $\delta$, then the regret of \block run with parameter $\e/2$ using \expg for weakly observable graphs as base algorithm $\A$ satisfies:
        $
            R_T \le 4 C_w (\delta/\e)^{1/3}\bigl(\ln(KT)\bigr)^{2/3}T^{2/3},
        $
    where $C_w > 0$ is a constant in the regret bound of $\A$.
    \end{itemize}
    \end{restatable}
    
    Note that we can explicitly compute valid constants $C_s = 12+2\sqrt{2}$ and $C_w=8$ directly from the bounds in \citet{AlonCDK15}.

    \subsection{Explore then Commit to a Graph}
    
    We are now ready to combine the two routines we presented, \round and \block, in our final learning algorithm, \etc. \etc has two phases: in the first phase, \round is used to quickly obtain an $\e$-good approximation $\hat \scG$ of the underlying feedback graph $\scG$, for a suitable $\e$. In the second phase, the algorithm commits to $\hat \scG$ and feeds it to \block. The crucial point is when to commit to a certain (estimated) stochastic feedback graph. If we commit too early, we might not observe a denser support graph, which implies missing out on a richer 
    feedback. If we wait for too long, then the exploration phase ends up dominating the regret. To balance this trade-off, we use the stopping function $\Phi$. This function takes as input a probabilistic feedback graph together with a time horizon and outputs the regret bound that \block would guarantee on this pair. It is defined as
    \begin{equation}
        \label{eq:def_Phi}
        \Phi(\scG,T) = \min\left\{4 C_s \sqrt{\frac{\alpha^*}{\e^*_s} T} \big(\ln(KT)\big)^{3/2},\; 4 C_w \left(\frac{\delta^*}{\e^*_w}\big(\ln(KT)\big)^2\right)^{1/3}T^{2/3} \right\} 
    \end{equation}
    for the specific choice of \textsc{Exp3.G} as the learning algorithm $\A$ adopted by \block.
    Note that the dependence of $\Phi$ on the feedback graph $\scG$ is contained in the topological parameters $\alpha^*$ and $\delta^*$ and the corresponding thresholds $\e^*_s$ and $\e^*_w$, defined in \Cref{def:eps_s,def:eps_w}; see \Cref{app:computation-notes} for more details on their computation.
    If we apply $\Phi$ to a stochastic feedback graph that is observable w.p.\ zero, its value is conventionally set to infinity.
    Observe that, otherwise, the min is achieved for a specific $\opte$ and a specific $\optG=\scG_{\opte}$. In \Cref{sec:block-alg-appendix}, we provide a sequence of lemmas (\Cref{lem:prop_Phi,lem:hatG} in particular) showing that, if \round outputs an $\e$-good approximation of the graph, then the regret is bounded by a multiple of the stopping criterion evaluated at $\scG$.
    Combined with Theorem~\ref{thm:round-robin-estimates}, which tells us that \round does in fact output an $\e$-good approximation of the graph with high probability, this proves our main result for this section.

    \begin{algorithm}[t]
    \textbf{Environment}: stochastic feedback graph $\scG$, sequence of losses $\ell_1, \ell_2, \dots, \ell_T$\;
    \textbf{Input}: time horizon $T$ and actions $V = \{1,2,\dots,K\}$\;
    Let $\Phi$ defined as in \Cref{eq:def_Phi}\;
    Run $\round(T,\Phi,V)$ and obtain $\hat \scG$ and $\hat \e$\;
    Compute $\hat \e_s^*$ and $\hat \e_w^*$ for graph $\hat \scG$ as in \Cref{def:eps_s,def:eps_w}\;
    Let $\hat \e^*$ be the best threshold as in \Cref{eq:def_Phi}\; 
    \lIf{$\hat \e^* = \hat \e_s^*$}{
        Let $\A$ be \expg for strongly observable feedback graph}
    ~~~~~~~~~~~~~
    \lElse{
        Let $\A$ be \expg for weakly observable feedback graph}
    Let $T'=T-\ftau K$ be the remaining time steps\tcp*{$\ftau$ as in \textnormal{\round}}
    Run $\block(T', \hat \e^*/2, \hat\scG_{\hat\e^*}, \A)$\;
    \caption{\explore}
    \label{alg:explore-commit}
    \end{algorithm}

    \begin{restatable}{retheorem}{retheoremblockresult}\label{thm:block-result}
        Consider the problem of online learning with stochastic feedback graph $\scG$ on $T$ time steps.
        If $\supp{\scG_{\e(K,T)}}$ is observable for $\e(K,T) = C K^3 (\ln(KT))^2 / T$ for a given constant $C > 0$, then there exists an algorithm whose regret $R_T$ satisfies (ignoring polylog factors in $K$ and $T$) $
            R_T
        \le
            \min\left\{
                \sqrt{({\alpha^*}/{\e_s^*})T}, 
                \bigl({\delta^*}/{\e_w^*}\bigr)^{1/3} T^{2/3}
            \right\}.
        $
    \end{restatable}

\section{Lower Bounds}\label{sec:lowerbounds} 

In this section, we provide lower bounds that match the regret bound guaranteed by \etc, shown in \Cref{thm:block-result}, up to polylogarithmic factors in $K$ and $T$.
These lower bounds are valid even if the learner is allowed to observe the realization of the entire feedback graph at every time step, and knows a priori the ``correct'' threshold $\e$ to work with.
\Cref{thm:lower-bounds-summary} summarizes the lower bounds whose proofs can be found in \Cref{sec:lower-bounds-proofs}.

\begin{theorem}[Informal]\label{thm:lower-bounds-summary}
    Let $\A$ be a possibly randomized algorithm for the online learning problem with stochastic feedback graphs.
    Consider any directed graph $G = (V, E)$ with $\abs{V} \ge 2$ and any $\e \in (0, 1]$.
    There exists a stochastic feedback graph $\scG$ with $\supp{\scG} = G$ and, for a sufficiently large time horizon $T$, there is a sequence $\ell_1, \dots, \ell_T$ of loss functions on which the expected regret of $\A$ with respect to the stochastic generation of $G_1, \dots, G_T \sim \scG$ is
    \begin{itemize}[topsep=0pt,parsep=0pt,itemsep=3pt]
    	\item $\Omega(\sqrt{(\alpha(\scG_\e)/\e)T})$ if $G$ is strongly observable;
    	\item $\tilde{\Omega}((\delta(\scG_\e)/\e)^{1/3}T^{2/3})$ if $G$ is weakly observable;
    	\item $\Omega(T)$ if $G$ is not observable.
    \end{itemize}
\end{theorem}

The lower bound in the non-observable case is the same as \citet[Theorem~6]{AlonCDK15} with a deterministic feedback graph.
The remaining lower bounds are nontrivial adaptations of the corresponding bounds for the deterministic case by \citet{AlonCDK15,AlonCGMMS17}. The main technical hurdle is due to the stochastic nature of the feedback graph, which needs to be taken into account in the proofs. The rationale behind the constructions used for proving the lower bounds is as follows: since each edge is realized only with probability $\e$, any algorithm requires $1/\e$ rounds in expectation in order to observe the loss of an action in the out-neighborhood of the played action, whereas one round would suffice with a deterministic feedback graph. 
Note that \Cref{thm:lower-bounds-summary} implies that, if $\supp{\scG_{\e(K,T)}}$ is not observable for $\e(K,T)$ as in \Cref{thm:block-result}, then there is no hope to achieve sublinear regret, as the lower bounds for both strongly and weakly observable supports are linear in $T$ for all $\e \le \e(K,T)$.


\section{Refined Graph-Theoretic Parameters}\label{sec:topo}
Although the results from Section \ref{sec:block} are worst-case optimal up to log factors, we may find that the factors $\sqrt{{\alpha(G_\e)}/{\e}}$ and $\left({\delta(G_\e)}/{\e}\right)^{1/3}$ for strongly and weakly observable $\supp{\scG_\e} = G_\e$, respectively, may be improved upon in certain cases. In particular, we show that, under additional assumptions on the feedback that we receive, we can obtain better regret bounds. To understand our results, we need some initial definitions. The weighted independence number for a graph $H = (V, E)$ and positive vertex weights $w(i)$ for $i \in V$ is defined as 
\begin{align*}
    \walpha(H, w) = \max_{S \in \I(H)} \sum_{i \in S}\nolimits w(i) \enspace,
\end{align*}
where $\I(H)$ denotes the family of independent sets in $H$.
We consider two different weight assignments computed in terms of any stochastic feedback graph $\scG$ with edge probabilities $p(i,j)$ and $\supp{\scG} = G$. They are defined as $w_{\scG}^-(i) = \bigl(\min_{j \in N^\text{in}_G(i)} p(j,i)\bigr)^{-1}$ and $w_{\scG}^+(i) = \bigl(\min_{j \in N^\text{out}_G(i)} p(i,j)\bigr)^{-1}$.
Then, the two corresponding weighted independence numbers are $\walpha^-(\scG) = \walpha(G, w_{\scG}^-)$ and $\walpha^+(\scG) = \walpha(G, w_{\scG}^+)$.
The parameter of interest for the results in this section is $\walpha(\scG) = \walpha^-(\scG) + \walpha^+(\scG)$. For more details on the weighted independence number see \Cref{sec:weighted-alpha}. 
We also use the following definitions of weighted weak domination number $\wdelta$ for a graph $H$ and positive vertex weights $w$, and self-observability parameter $\sigma$:
\begin{align*}
    \wdelta(H, w) = \min_{D \in \D(H)} \sum_{i \in D}\nolimits w(i) \enspace,
    && \sigma(\scG) = \sum_{i: i \in \inneigh{G}{i}}\nolimits ({p(i,i)})^{-1} \enspace,
\end{align*}
where $\D(H)$ denotes the family of weakly dominating sets in $H$. 
In this section, we focus on the weighted weak domination number $\wdelta(\scG) = \wdelta(G, w_{\scG}^+)$. To gain some intuition on the graph-theoretic parameters introduced above, consider the graph with only self-loops, also used in Example~\ref{ex:onlyselfloops} below. If all $p(i,i) = \e$, the learner needs to pull a single arm $1/\e$ times for one observation in expectation, and $K/\e$ times to see the losses of all arms once. However, when the edge probabilities are different we need to pull arms for $\sumK 1/{p(i, i)}$ times.
The weighted independence number, weighted weak domination and self-observability generalize this intuition and tell us how many observations the learner needs to see all losses at least once in expectation.
We now state the main result of this section. 
%
%
\begin{theorem}[Informal]\label{th:informalUCB}
    There exists an algorithm with per-round running time of $O(K^4)$ and whose regret is bounded (ignoring logarithmic factors) by
    \begin{align*}
        \min \Biggl\{ T,\, &\min_{\e} \Bigl\{ \sqrt{\walpha(\scG_\e)T} \,:\, \supp{\scG_\e} \textnormal{ is strongly observable}\Bigr\}, \\[-5pt]
        & \min_{\e}  \Bigl\{\left(\wdelta(\scG_\e)\right)^{1/3}T^{2/3} + \sqrt{\sigma(\scG_\e)T} \,:\, \supp{\scG_\e} \textnormal{ is observable}\Bigr\}\Biggr\} \enspace, 
    \end{align*}
\end{theorem}

The regret bound in \Cref{th:informalUCB} follows from \Cref{th:finalbeyondsanity} in \Cref{app:OTC}. The running time bound is determined by approximating $\wdelta$ for all $K^2$ possible thresholds.
In each of the thresholded graphs, we can compute a $(\ln(K) + 1)$-approximation for the weighted weak domination number in $O(K^2)$ time by reduction to set cover \citep{vazirani2001approximation}.
Doing so only introduces an extra factor of order $(\ln(K))^{1/3}$ in the regret bound. 

An important property of the bound in Theorem~\ref{th:informalUCB} is that it is never worse than the bounds obtained before. The following example shows that the regret bound in Theorem \ref{th:informalUCB} can also be better than previously obtained regret bounds. 
%
%
\begin{example}[Faulty bandits]\label{ex:onlyselfloops}
    Consider a stochastic feedback graph $\scG$ for the $K$-armed bandit setting: $p(i,i) = \e_i \in (0, 1]$ for all $i \in V$ and $p(i,j) = 0$ for all $i \neq j$.
    In this case, the regret of \explore
    is $\tilde{O}\bigl(\sqrt{KT/(\min_{i} \e_i)})$.
    On the other hand, \Cref{th:informalUCB} provides the bound $\tilde O\bigl(\sqrt{T\sum_{i}(1/{\e_i})}\bigr)$, as $\walpha(\scG) = 2\sum_i 1/\e_i$.
    In the special case when $\e_i = \e \in (0,1]$ for some $i \in V$ while $\e_j=1$ for all $j \neq i$, the regret of \explore
    is  $\tilde O(\sqrt{KT/\e})$, while Theorem~\ref{th:informalUCB} guarantees a $\tilde O(\sqrt{(K + 1/\e) T})$ regret bound.
\end{example}

To derive these tighter bounds, we exploit the additional assumption that the realized feedback graph $G_t$ is observed at the end of each round. This allows us to \emph{simultaneously} estimate the feedback graph and control the regret, rather than performing these two tasks sequentially 
as in \Cref{sec:block}. In particular, we use this extra information to construct a novel importance-weighted estimator for the loss, which for $t > 1$ is defined to be
%
%
\begin{align}\label{eq:IWestimator}
    \tilde{\ell}_t(i) = \frac{\ell_t(i)}{\hat{\Pp}_t(i)} \idseeit \enspace,
\end{align}
where $\hat{\Pp}_t(i) = \sum_{j \in N^\text{in}_{\hat{G}_t}(i)}\pi_t(j)\phattji$ is the estimated probability of observing the loss of arm $i$ at round $t$, $\pi_t(i)$ is the distribution we sample $I_t$ from, and $\hat{G}_t$ is the support of the estimated graph $\hat{\scG}_t$. Note that we ignore losses that we receive due to missing edges in $\hat{G}_t$. We show that we pay an additive term in the regret for wrongly estimating an edge, which is why it is important to control which edges are in $\hat{G}_t$. 
Ideally, we would use $\Pp_t(i) = \sum_{j \in N^\text{in}_{\hat{G}_t}(i)}\pi_t(j)\pji$ rather than $\hat{\Pp}_t(i)$, as this is the true probability of observing the loss of arm $i$ in round $t$. However, since  we do not have access to $\pji$, we use instead an upper confidence estimate of $\pji$ for rounds $t \geq 2$ given by
\begin{align*}
    \phattji = \ptilt(j,i) + \UCBt \enspace,
\end{align*}
where $\ptilt(j,i) = \ptiltji$.
We denote by $\hat{\scG}_t^{\text{UCB}}$ the stochastic graph with edge probabilities $\phattji$. Note that the support of $\hat{\scG}_t^{\text{UCB}}$ is a complete graph because $\phattji > 0$ for all $(j,i) \in V \times V$.
These estimators for the edge probabilities are sufficiently good for our purposes whenever event $\kappa$ occurs, which we define as the event that, for all $t \geq 2$,
\[
\abs*{\ptilt(j,i) - \pji} \leq \UCBt, ~~ \forall (j,i) \in V \times V \enspace.
\]
An important property of $\ltilt$ can be found in \Cref{lem:informalIWproperty} below. It tells us that we may~treat $\ltilt$ as if event $\kappa$ is always realized, i.e., $\phattji$ is always an upper bound estimator on $\pji$. The proof of \Cref{lem:informalIWproperty} is implied by \Cref{lem:IWproperty} in \Cref{app:OTC}. 
\begin{lemma}[Informal]\label{lem:informalIWproperty}
    Let $e_k$ denote the basis vector with $e_k(i) = \ind{\{i=k\}}$ as $i$-th entry for each $i \in [K]$.
    The loss estimate $\ltilt$ defined in~\eqref{eq:IWestimator} satisfies
    %
    \begin{align}\label{eq:IWpropertyinf}
        R_T
        & =  \tilde{O}\Biggl(\E{\sumTt  \sqrt{\sumK  \frac{\pi_t(i)}{(t-1)\hat{\Pp}_t(i)}}~\middle|~\kappa} + \max_{k \in V} \E{\sumTt \sumK \bigl(\pi_t(i) - e_k(i)\bigr)\ltilt(i)~\middle|~\kappa}\Biggr).
    \end{align}
\end{lemma}
\Cref{lem:informalIWproperty} shows that we only suffer $\tilde{O}\Bigl(\sqrt{\sumTt \sumK \frac{\pi_t(i)}{\Pp_t(i)}}\Bigr)$ additional regret compared to when we know $\pji$. \Cref{lem:informalIWproperty} also shows that $\ltilt$ behaves nicely in the sense that, conditioned on $\kappa$, we have $\tilde{\ell}_t(i)  \leq \frac{\ell_t(i)}{\Pp_t(i)} \idseeit$.
This is a crucial property to bound the regret of our algorithm. We show that, with a modified version of \expg~\citep{AlonCDK15}, the second sum on the right-hand side of~\eqref{eq:IWpropertyinf} is bounded by a term of order $\sqrt{\sumTt \sumK \frac{\pi_t(i)}{\Pp_t(i)}}$, meaning that the regret is also bounded similarly.
Our final step is to prove that the above 
term is bounded in terms of the minimum of the weighted independence number and the weighted weak domination number plus self-observability.

\begin{ack}
Nicolò Cesa-Bianchi, Federico Fusco and Dirk van der Hoeven gratefully acknowledge partial support by the MIUR PRIN grant Algorithms, Games, and Digital Markets (ALGADIMAR). Nicolò Cesa-Bianchi and Emmanuel Esposito were also supported by the EU Horizon 2020 ICT-48 research and innovation action under grant agreement 951847, project ELISE. Federico Fusco was also supported by the ERC Advanced Grant 788893 AMDROMA “Algorithmic and Mechanism Design Research in Online Markets”.
\end{ack}

\bibliographystyle{plainnat}
\DeclareRobustCommand{\VAN}[3]{#3} 
\bibliography{references.bib}
\DeclareRobustCommand{\VAN}[3]{#2} 

\clearpage

\appendix

\section{On the Computation of the Optimal Probability Thresholds} \label{app:computation-notes}

    The tasks of finding the independence number and (weak) domination number in a graph are notoriously $\np$-hard problems.
    In particular, while for the domination number, by a reduction to set cover, a simple greedy approach yields a logarithmic (in the number $K$ of nodes) approximation \citep{vazirani2001approximation}, for the independence number it is known that even computing a $K^{1-\epsilon}$-approximation is hard, for any $\epsilon > 0$ \citep{hastad1999clique,Zuckerman2007inapprox}.
    
     Our algorithm \optimist solves these computational aspects directly, whereas the hardness of finding $\alpha^*$ and $\delta^*$ may limit the applicability of \explore in instances with a large and complex action space.
     In fact, the computation of the stopping function $\Phi$ involves finding the best thresholds $\e^*_s$ and $\e^*_w$, defined in \Cref{def:eps_s,def:eps_w}, and therefore repeatedly solving $\np$-hard problems.
     In what follows, we present some observations that clarify to which extent (and at which cost) \explore can still be implemented efficiently.
     
     First, it is important to note that our algorithm is robust with respect to approximate knowledge of the topological parameters: the definition of $\Phi$ can be tweaked as to consider the approximation factor at the cost of having the same factor showing up in the regret bound (with the same order as the approximated graph parameter).
     This partly solves the problem for weakly observable graphs (as the $(\log(K)+1)$-approximation only gives and extra $\polylog(K)$ in the regret) and for the classes of graphs where it is possible to efficiently compute good approximations of the independence number, e.g., planar graphs \citep{Baker94} or bounded-degree graphs \citep{halldorsson1994greed}. 
     
     Another approach consists in considering the fractional solutions of the independence and domination number linear programs.
     While for the former we obtain an approximation given by the integrality gap, for the latter we can show a tight dependence on the fractional weak domination number (thus improving the regret bound), as in \citet{Chen21}.
     
     Furthermore, note that it is always possible to ignore the $\alpha$ and $\delta$ terms in the definition of~$\Phi$; it is not hard to see that such an approach yields a regret bound (ignoring polylog terms) of the type $\min\{\sqrt{(K/\e_1)T}, (K/\e_2)^{1/3} T^{2/3}\}$, where $\e_1$, respectively $\e_2$, is the largest $\e$ such that $\supp{\scG_{\e}}$ is strongly, respectively weakly, observable.
     Although suboptimal, this drastic approach gives a regret bound with an optimal dependence on the $T$ and $\e$ terms (as $\e^*_s \le \e_1$ and $\e^*_w \le \e_2$).
     
    Finally, we conclude by discussing how it is possible to drastically reduce the number of times that \explore calls the routine to compute $\alpha$ and $\delta$, at the cost of losing a small multiplicative factor in the regret.
    Crucially, we do not need to check the stopping condition involving $\Phi$ in every single round: it suffices to do so for a logarithmic number of times. 
    Assume, in fact, to check the stopping condition in \round only when $\tau$ is a power of $2$, i.e., $\tau=2^b$ for some integer $b$.
    This single check covers all rounds $\tau'$ such that $\tau/2 = 2^{b-1} \le \tau' \le 2^b = \tau$.
    On the stochastic graph estimate $\hat{\mathcal{G}}_{\tau}$ we can compute $\alpha_{\varepsilon_\tau}/\varepsilon_\tau$ and $\delta_{\varepsilon_\tau}/\varepsilon_\tau$, which are also $2$-approximations for the best respective ratios on any thresholded graph corresponding to rounds of \round between $\tau/2$ and $\tau$ (note that such an approach would also improve the dependency of $\e_\tau$ and $\Delta$ on $T$ in Theorems~\ref{thm:round-robin-estimates} and~\ref{thm:blocks-reduction}, and thus in the regret bound, from $\log(T)$ down to $\log(\log(T))$ due to an improved union bound).
    

\section{Missing Results from Section~\ref{sec:block}} \label{sec:block-alg-appendix}

\subsection{Proof of Theorem~\ref{thm:round-robin-estimates}}

\retheoremroundrobinestimates*
    \begin{proof}[Proof of \Cref{thm:round-robin-estimates}]
    For all edges $e$ and time steps $\tau \le \ftau$, we define the following two events: the event $\calE_{e}^{\tau} = \left\{\hat p^{\tau}_e \ge \e_{\tau} \right\}$ that $e$ belongs to the support of $\hat \scG_{\tau}$, and the event $\F_e^{\tau} = \left\{|\hat p_e^{\tau} - p_e| \le p_e/2\right\}$ that $\hat p^{\tau}_e$ is well estimated. For all $\tau \le \ftau$, we also define large and small edges in $E$ according to their probabilities: $\lEtau = \{e \in V^2 \,:\, p_e \ge 2\e_{\tau}\}$ and $\sEtau = \{e \in V^2 \,:\, p_e < \e_{\tau}/2\}$.
    
    First, we look at the complementary event of $\calE_{e}^{\tau}$ for any $\tau \le \ftau$ and $e \in \lEtau$. We have:
    \[
        \P{\overline{\calE}_e^{\tau}} = \P{\hat p_e^{\tau} < \e_{\tau}} \le \P{\hat p_e^{\tau} \le p_e/2} = \P{\hat p_e^{\tau} - p_e\le - p_e/2} \le \mye^{-\frac{\tau}{8}p_e } \le \mye^{-\frac{ \tau}{4}\e_{\tau}}  \le \frac{1}{4KT^2} \enspace.
    \]
    Note that in the first and second to last inequalities we used the fact that $p_e \ge 2 \e_{\tau}$, in the last inequality the definition of $\e_{\tau}$ and the fact that $K \ge 2$, while in the second inequality we applied the Chernoff lower bound (multiplicative version, see \citet[part~$2$ of Theorem~4.5]{Eli05}) on the estimator $\hat p_e^{\tau}$.
    
    If we call $\calE$ the event corresponding to part~1 of \Cref{def:eps-good-approx},
    we have the following:
    \begin{equation}
    \label{eq:calE}
        \P{\calE} =  \P{\bigcap_{\tau \le \ftau} \bigcap_{e \in \lEtau} \calE_e^{\tau}} \ge 1 - \sum_{\tau \le \ftau}\sum_{e \in \lEtau}\P{\hat p_e^{\tau} < \e_{\tau}}\ge 1 - \sum_{\tau\le\ftau} \frac{|\lEtau|}{4KT^2} \ge 1 - \frac{1}{4T} \enspace,
    \end{equation}
    where we used that $|\lEtau| \le K^2$ for all $\tau \le \ftau \le T/K$ with probability $1.$
    
    Next, we study the complementary event of $\F_e^{\tau}$ for $e \not \in \sEtau$. For such $e$ and any $\tau \le \ftau$, we can directly use the two-sided Chernoff bound (multiplicative version, as in
    \citet[Corollary~4.6]{Eli05}) on the estimator $\hat p_e^{\tau}$:
    \[
        \P{\overline{\F}_e^{\tau}} = \P{|\hat p_e^{\tau} - p_e| > \frac 12 p_e} \le 2 \mye^{-\frac{\tau}{12} p_e} \le 2 \mye^{-\frac{\tau}{24}\e_{\tau}} \le \frac{1}{2KT^2} \enspace.
    \]
    Note that we used the definition of $\e_{\tau}$ and the facts that $2 p_e \ge \e_{\tau}$ and $K,T\ge 2$. Now, if we call $\F$ the event corresponding to part~2 of \Cref{def:eps-good-approx},
    we can proceed via union bounding as in \Cref{eq:calE} and get 
    \begin{equation}
    \label{eq:calF}
        \P{\F} = \P{\bigcap_{\tau \le \ftau} \bigcap_{e \notin \sEtau} \F_e^{\tau}} \ge 1 - \frac{1}{2T} \enspace.
    \end{equation}
    As a third step, we get back to the $\calE_{e}^{\tau}$ events, but we consider $e \in \sEtau$. For $\tau \le \ftau$ and $e \in \sEtau$ we have:
    \[
        \P{\calE_{e}^{\tau}} = \P{\hat p_e^{\tau} \ge \e_{\tau}} \le \P{\hat p_e^{\tau} - p_e \ge \frac{1}2 \e_{\tau}} =
        \P{\hat p_e^{\tau} - p_e \ge x  p_e},
    \]
    where we used $p_e < \e_{\tau}/2$ and named $x = \e_{\tau}/(2p_e) > 1$. At this point we can use the Chernoff upper bound (multiplicative version, see \citet[part~$1$ of Theorem~4.4]{Eli05} with $\delta = x$) and obtain:
    \begin{align*}
        \P{\calE_{e}^{\tau}} \le \P{\hat p_e^{\tau} - p_e \ge x  p_e} \le \left(\frac{\mye^{x}}{(1+x)^{1+x}}\right)^{\tau p_e} \le \mye^{-\frac{\tau}{3} x p_e} = \mye^{-\frac{\tau}{6} \e_{\tau}} \le \frac{1}{4KT^2} \enspace.
    \end{align*}
    The third inequality follows from $2x/(2+x) \le \ln(1+x)$ which holds for all positive $x$:
    \[
        \frac{\mye^{x}}{(1+x)^{1+x}} = \mye^{x - (1+x)\ln(1+x)} \le \mye^{-x^2/(2+x)} \le \mye^{-x/3}, \quad \forall x \ge 1 \enspace.
    \]
    If we now call $\scC$ the event described in part~3 of \Cref{def:eps-good-approx},
    we get, using the bound on $ \P{\calE_{e}^{\tau}}$ and a union bound as in \Cref{eq:calE,eq:calF}:
    \begin{equation}
    \label{eq:calC}
        \P{\scC} = \P{\bigcap_{\tau \le \ftau} \bigcap_{e \in \sEtau} \overline{\calE}_e^{\tau}} \ge 1 - \frac{1}{4T} \enspace.
    \end{equation}
    The theorem then follows by a union bound on the complementary events of $\calE, \F$ and $\scC$.
    \end{proof}

\subsection{Proof of Theorem~\ref{thm:blocks-reduction}}

    In order to prove the regret bound achieved by \block, we need to show that it is able to compute unbiased estimators for the average loss of observed actions within each time block.
    This property is guaranteed  as long as the learner plays consistently a same action within each time block, and conditioned on the event that each action in the support out-neighborhood of the chosen action is observed at least once in the respective time block (depending on the realizations of the feedback graph).

    \begin{lemma} \label{lem:unbiased-estimator-blocks}
    Let $G = \supp{\scG}$ and $c_{\tau}$ and $\hat c_{\tau}$ defined as in \Cref{eq:c_tau,eq:blocks-estimator}. For each block $B_{\tau}$, if the learner plays consistently action $a$, then for each $a' \in \outneigh{G}{a}$ the estimators $\hat c_{\tau}(a')$ are unbiased under $\calE^\tau_{(a,a')}$:
    \[
        \E{\hat c_\tau(a') \,\middle|\, \calE^\tau_{(a,a')}} = c_{\tau}(a') \;, \quad \forall a' \in \outneigh{G}{a} \enspace.
    \]
    \end{lemma}
    \begin{proof}
    Recall that $\calE^\tau_{(a,a')}$ is the event that the edge $(a,a')$ in $G$ is observed at least once in block $B_\tau$. Substituting the definition~\eqref{eq:blocks-estimator} of the estimator, we can write
    \[
        \E{\hat c_\tau(a') \,\middle|\,  \calE^\tau_{(a,a')}} = \sum_{t \in B_\tau} \ell_t(a') \E{\frac{\ind{\{(a,a') \in E_t\}}}{\Delta^\tau_{(a,a')}} \,\middle|\,  \calE^\tau_{(a,a')}} \enspace.
    \]
    Now we just need to prove that the expectation in the right-hand side is equal to $1/\Delta$:
    \begin{align*}
    \E{\frac{\ind{\{(a,a') \in E_t\}}}{\Delta^\tau_{(a,a')}} \,\middle|\, \calE^\tau_{(a,a')}}
    &= \sum_{r=1}^{\Delta} \E{\frac{\ind{\{(a,a') \in E_t\}}}{r} \,\middle|\, \Delta^\tau_{(a,a')}=r} \P{\Delta^\tau_{(a,a')}=r \,\middle|\, \calE^\tau_{(a,a')}} \\
    &= \sum_{r=1}^{\Delta} \frac1r \P{(a,a') \in E_t \,\middle|\, \Delta^\tau_{(a,a')}=r}\P{\Delta^\tau_{(a,a')}=r \,\middle|\, \calE^\tau_{(a,a')}} \\
    &= \frac1\Delta \sum_{r=1}^{\Delta} \P{\Delta^\tau_{(a,a')}=r \,\middle|\, \calE^\tau_{(a,a')}} = \frac 1{\Delta} \enspace.
    \end{align*}
    Note that in the third equality we used the fact that, conditioned on $\Delta^\tau_{(a,a')}=r>0$, the $r$ time steps when $(a,a') \in E_t$ are distributed uniformly at random in the $\Delta$ time steps.
    \end{proof}
    
    We can now prove the regret bound of \block in \Cref{thm:blocks-reduction}, which we restate below.
    Its regret depends on the performance of the algorithm $\A$ used on the meta-instance derived from the blocks reduction.
    
\retheoremblocksreduction*
    \begin{proof}[Proof of \Cref{thm:blocks-reduction}]
        Consider the partition of the $T$ time steps into $N$ blocks $B_1, \dots, B_N$ of equal size $\Delta$ and let $\calE$ be the clean event, corresponding to all edges $e$ in the graph $\supp{\hat \scG} = \hat G = (V,E)$ being realized at least once in each block. Formally, $\calE = \bigcap_{\tau=1}^N \bigcap_{e \in E} \calE_{e}^{\tau}$, where $\calE_{e}^{\tau}$ are defined as in the proof of \Cref{lem:unbiased-estimator-blocks}.
        By \Cref{def:eps-good-approx} (part~3), all the edges $e \in E$ have a probability $p_e$ in $\scG$ that is at least $\e/2$. Thus, it is immediate to verify that
         \[
            \P{\calE^{\tau}_e} = 1 - (1-p_e)^{\Delta} \ge 1 - \left(1-\frac\e 2\right)^{\Delta} \ge 1-\mye^{-\e\Delta/2}\ge 1 - \frac{1}{K^2T^2}
         \]
         holds for any edge $e \in E$ using our choice of $\Delta$.
         We show by union bound that the probability any of these edges never realizes in some block is
         \[
         \P{\bigcup_{\tau \le N}\bigcup_{e\in E} \overline{\calE}^{\tau}_e} \le \sum_{\tau \le N}\sum_{e\in E} \P{\overline{\calE}^{\tau}_e} \le \frac{1}{T} \enspace,
         \]
         where we used that there are at most $K^2$ directed edges (including self-loops) in $\hat G$ and we substituted the chosen values of $N$ and $\Delta.$

        We can then bound the overall regret $R_T$ as follows:
        \begin{equation} \label{eq:blocks-regret-1}
            R_T \le \E{\sum_{t=1}^{T} \ell_t(I_t)\,\middle|\, \calE} - \min_k \sum_{t=1}^{T} \ell_t(k)  + T\cdot\P{\overline{\calE}} + (T - \Delta N) \enspace.
        \end{equation}
        Note that the final term is an upper bound to the regret in the final time steps of the algorithm. We just showed that $\P{\overline \calE}$ is smaller than $1/T$. This, together with the fact that $T - \Delta N$ is at most $\Delta - 1$, gives the additive $\Delta$ we have in the final statement. 
        
        We now focus on the remaining term, which corresponds to the regret conditioned on $\calE$. It is equal to
        \begin{align}\nonumber
            \E{\sum_{t=1}^{T} \ell_t(I_t) \,\middle|\, \calE} - \min_k \sum_{t=1}^{T} \ell_t(k)
            &= \Delta \cdot \left( \E{\sum_{\tau=1}^{N} \sum_{t\in B_\tau} \frac{\ell_t(I_\tau)}{\Delta} \,\middle|\, \calE} - \min_k \sum_{\tau=1}^{N} \sum_{t\in B_\tau} \frac{\ell_t(k)}{\Delta} \right) \\ \label{eq:part2}
            &= \Delta \cdot \left( \E{\sum_{\tau=1}^{N} c_\tau(I_\tau) \,\middle|\, \calE} - \min_k \sum_{\tau=1}^{N} c_\tau(k) \right) \enspace,
        \end{align}
        where, we recall it, $c_\tau(i)$ is the average loss of action $i$ in block $B_\tau$.
        Indeed, our algorithm chooses the same action $I_t = I_\tau$ for all time steps $t \in B_\tau$, and the decision is based on algorithm $\mathcal{A}$.
        
        Consider now the loss estimates $\hat c_1, \dots, \hat c_N$ that we provide to algorithm $\mathcal{A}$.
        These estimates are such that $\E{\hat c_\tau(i) \,\middle|\, \calE} = c_\tau(i)$ by \Cref{lem:unbiased-estimator-blocks}. Note that conditioning on $\calE$ instead that on the single $\calE_e^{\tau}$ does not affect the fact that the estimators are unbiased: this is due to the fact that the edge realizations are independent from the losses and the strategy of the learner.

        Therefore, letting $k^*$ be the action minimizing $c_1(k)+\cdots+c_T(k)$ over $k=1,\ldots,K$,
        \begin{align}\label{eq:part3}
            \E{\sum_{\tau=1}^{N} c_\tau(I_\tau) \,\middle|\, \calE} - \min_k \sum_{\tau=1}^{N} c_\tau(k)
            = \E{\sum_{\tau=1}^{N} \hat c_\tau(I_\tau) - \sum_{\tau=1}^{N} \hat c_\tau(k^*) \,\middle|\, \calE}
            \le R^\mathcal{A}_N(\hat G) \enspace,
        \end{align}
        where $R^\mathcal{A}_N(\hat G)$ is the regret bound of algorithm $\mathcal{A}$ given losses $\hat c_1, \dots, \hat c_N$ and feedback graph $\hat G = \supp{\hat \scG}$. Finally, substituting \Cref{eq:part2,eq:part3} into \Cref{eq:blocks-regret-1} yields the desired bound.
    \end{proof}
    
\subsection{Proof of Corollary~\ref{cor:blocks-reduction}}
\recorblock*
    \begin{proof}[Proof of \Cref{cor:blocks-reduction}]
        The statement follows from \Cref{thm:blocks-reduction}, the assumption on $\e$ (which lets us safely handle the additive $\Delta$ term), and the fact that \expg achieves regret 
        $
            R^\A_N \le C_s\sqrt{\alpha N}\ln(KN)
        $
        on strongly observable graphs, and regret 
        $
            R^\A_N \le C_w (\delta \ln K)^{1/3}N^{2/3}
        $
        on (weakly) observable graphs.
    \end{proof}

\subsection{Proof of Theorem~\ref{thm:block-result}}
    To prove \Cref{thm:block-result} we first need two preliminary lemmata. In \Cref{lem:prop_Phi} we present some generic properties of the stopping function $\Phi(\scG,T)$, while in \Cref{lem:hatG} we prove that $\Phi(\scG,T-\ftau K)$ is indeed the regret obtained in \block after the stopping condition in \round is triggered.
    \begin{lemma}\label{lem:prop_Phi}
        Let $\scG$ be a stochastic feedback graph such that $\Phi(\scG,T) \neq \infty$, and let $\e^*$ be the threshold where the $\argmin$ in the definition of $\Phi(\scG,T)$ is attained. Consider a run of the algorithm \explore where \round does not fail while using the stopping function $\Phi$ defined in \Cref{eq:def_Phi}. We have the following:
        \begin{itemize}[topsep=0pt,parsep=0pt,itemsep=0pt]
            \item[$(i)$] $\Phi(\hat\scG_{\tau'}, T) \le 2 \Phi(\hat\scG_{\tau}, T) $, for all $\tau,\tau'$ such that $\tau \le \tau' \le \ftau$,
            \item[$(ii)$] $\Phi(\hat \scG_{\tau},T)\le \sqrt{2} \Phi(\scG,T)$ for all $\tau$ such that  $120 \ln(KT)/\e^* \le \tau \le \ftau$ (if such $\tau$ exists),
        \end{itemize}
        where $\ftau \le \floor{T/K}$ is the index of the last iteration of the outer for loop in \Cref{alg:round-robin}.
    \end{lemma}
   \begin{proof}
        We consider a run of \explore where \round does not fail. This means that all the $\hat \scG_{\tau}$ are $\e_{\tau}$-good approximation of $\scG$, for all $\tau \le \ftau$. Focus on the first part of the statement. All edges in $\supp{\hat\scG_{\tau}}$ are contained in $\supp{\hat\scG_{\tau'}}$ since \round does not fail. This implies that the observability regime only improves as $\tau$ increases. We have two cases: if the best threshold for $\hat \scG_{\tau}$ (say it corresponds to some edge probability in $\hat \scG_{\tau}$ without loss of generality) induces a thresholded stochastic feedback graph with strongly observable support $G = (V, E)$ and independence number $\alpha$, we have that $\hat \scG_{\tau'}$ is strongly observable too; moreover, all the edges $e \in E$ are such that $|p_e - \hat p^{\tau}_e| \le p_e/2$ by \Cref{def:eps-good-approx} (part~2); the same holds for $\tau'$: $|p_e - \hat p^{\tau'}_e| \le p_e/2$. Consider graph $G$ with edge probabilities $\hat p^{\tau'}_e$, respectively $p_e$ and $\hat p^{\tau}_e$ and let $\e_1$, respectively $\e_2$ and $\e_3$, be their smallest probability (restricting on the edges of $G$). We have that:
        \begin{align*}
            \min_{\e\in (0,1]} &\left\{\frac{\alpha((\hat \scG_{\tau'})_{\e})}{\e} \,:\, \textup{$ \supp{( \hat\scG_{\tau'})_{\e}}$ strongly observable}\right\} \le \frac{\alpha}{\e_1} \le  2 \frac{\alpha}{\e_2} \le 4 \frac{\alpha}{\e_3} \\
            &=  4 \min_{\e\in (0,1]} \left\{\frac{\alpha((\hat \scG_{\tau})_{\e})}{\e} \,:\, \textup{$ \supp{( \hat\scG_{\tau})_{\e}}$ strongly observable}\right\} \;,
        \end{align*}
        where the first inequality follows from suboptimality of graph $G$ with threshold $\e_1$ for $\hat\scG_{\tau'}$, the second and the third inequality by the conditions on $p_e$, $\hat p^{\tau'}_e$ and $p_e^{\tau}$, and the last equality by definition of $G$ and $\alpha.$ If we now substitute this inequality in the definition of $\Phi$, we obtain that 
        $
            2 \Phi(\hat\scG_{\tau}, T) \ge \Phi(\hat\scG_{\tau'},T)
        $.
        We can reason in the same exact way considering the (weakly) observable case and obtain
        $
            \sqrt[3]{4}\Phi(\hat\scG_{\tau}, T) \ge \Phi(\hat\scG_{\tau'}, T)
        $.
        Putting the two results together we conclude the proof of point $(i)$.
        
        We move our attention to the second part of the lemma. Because of \Cref{thm:round-robin-estimates} together with the lower bound on $\tau$, it holds that $\hat \scG_{\tau}$ is an $\e^*/2$-good approximation of $\scG$. This implies that all the edges in $\supp{\scG_{\e^*}}$ are contained in the support of $\hat \scG_{\tau}$ and that they are well approximated, as in parts~1 and~2 of \Cref{def:eps-good-approx}. We have two cases, according to the topology of the support corresponding to the threshold $\e^*$ which guarantees the optimal regret for $\scG$. First, consider the case that $\e^*$ corresponds to a strongly observable structure in $\supp{\scG_{\e^*}}$ with independence number $\alpha^*$; we have that
        \begin{align*}
            \min_{\e\in (0,1]} &\left\{\frac{\alpha((\hat \scG_{\tau})_{\e})}{\e} \,:\, \textup{$ \supp{( \hat\scG_{\tau})_{\e}}$ strongly observable}\right\} \le 2 \frac{\alpha^*}{\e^*}\\
            &= 2\min_{\e\in (0,1]}\left\{\frac{\alpha(\scG_{\e})}{\e} \,:\, \textup{$ \supp{\scG_{\e}}$ strongly observable}\right\} \;,
        \end{align*}
        where in the first inequality we used the suboptimality of threshold $\e^*/2$ for $\hat \scG_{\tau}$ and the fact that the independence number of $\alpha((\hat \scG_{\tau})_{\e^*})$ is at most $\alpha^*$ (and the strong observability is maintained). Then, we have that
        \[
            \Phi(\hat\scG_{\tau}, T) \le  4 C_s \sqrt{2\frac{\alpha^*}{\e^*} T} \big(\ln(KT)\big)^{3/2} = \sqrt{2}\Phi(\scG,T) \;,
        \]
        where the inequality follows naturally from the (possible) suboptimality of the choice of the strongly observable regime and the threshold $\e^*/2$ for $\hat\scG_{\tau}.$
        We can argue similarly for the case in which the optimal $\e^*$ corresponds to the weakly observable regime in $\scG$. In this case, for the same arguments as per the strongly observable regime, we have that
        \begin{align*}
            \min_{\e\in (0,1]} &\left\{\frac{\delta((\hat \scG_{\tau})_{\e})}{\e} \,:\, \textup{$ \supp{( \hat\scG_{\tau})_{\e}}$ observable}\right\} \le 2 \frac{\delta^*}{\e^*} \\
            &= 2\min_{\e\in (0,1]}\left\{\frac{\delta(\scG_{\e})}{\e} \,:\, \textup{$ \supp{\scG_{\e}}$ observable}\right\} \;.
        \end{align*}
        Finally, similarly to the strongly observable case, it holds that
        \[
            \Phi(\hat\scG_{\tau}, T) \le  4 C_w \left(2\frac{\delta^*}{\e^*}\big(\ln(KT)\big)^2\right)^{1/3}T^{2/3} = \sqrt[3]{2}\Phi(\scG,T)\le \sqrt{2}\Phi(\scG,T) \;.
        \]
        This concludes the proof.
    \end{proof}

    \begin{lemma}\label{lem:hatG}
        Consider a run of \explore (\Cref{alg:explore-commit}).
        Assume that the invocation of \round returns a stochastic feedback graph $\hat \scG$ that is an $\hat \e$-good approximation of $\scG$ satisfying $\Phi(\hat \scG, T-\ftau K) \le \ftau K$, where $\ftau$ is the index of the last iteration of the outer for loop in \Cref{alg:round-robin}.
        Then, the regret experienced by the invocation of \block is at most $\Phi(\hat \scG,T - \ftau K)$.
    \end{lemma}
    \begin{proof}
        Denote with $R^{\br}_{T'}$ the worst-case regret experienced by \block in the final $T' = T-\ftau K$ time steps, under the assumption on $\hat\scG$ in the statement, and let $\hat \e^*$ be the best threshold as in \Cref{alg:explore-commit}. We have two cases, according to $\hat\e^*$ referring to strongly or (weakly) observable graphs. If $\hat\e^* =\hat\e^*_s$, then, by the part of \Cref{cor:blocks-reduction} relative to strongly observable graphs, we have that
        \begin{align*}
            R^{\br}_{T'} \le  4 C_s \sqrt{\frac{\hat \alpha^*}{\hat \e^*_s}T'} \big(\ln(KT')\big)^{3/2} = \Phi(\hat \scG_{\hat \e^*},T') \enspace.
        \end{align*}
        If $\hat\e^* =\hat\e^*_w$, then we can apply the part of \Cref{cor:blocks-reduction} relative to (weakly) observable graphs and obtain that
        \begin{align*}
            R^{\br}_{T'} \le 4 C_w \biggl(\frac{\hat\delta^*}{\hat\e^*_w}\bigl(\ln(KT')\bigr)^2\biggr)^{1/3}(T')^{2/3} =\Phi(\hat \scG_{\hat \e^*},T') \enspace.
        \end{align*}
    \end{proof}

    At this point, we have all the essential ingredients to prove the regret bound of \etc as stated in \Cref{thm:block-result}.
    We rewrite the statement of \Cref{thm:block-result} for convenience.

\retheoremblockresult*
    \begin{proof}[Proof of \Cref{thm:block-result}]
    We condition the analysis on the clean event $\calE$ that \round does not fail. Let $\te$ be the largest $\e$ such that $\supp{\scG_{\e}}$ is observable, and $\ttau$ be the smallest (random) integer such that $\supp{\hat \scG_{\ttau}}$ is observable for $\hat \scG_{\ttau}$ in \round. We have some immediate bound on these quantities. First, $\te \ge \e(K,T)$, by the assumption on $\supp{\scG_{\e(K,T)}}$ being observable. 
    Second, $\ttau \le \frac{120}{\te}\ln(KT)$; this is due to the fact that, after $\tau = \lceil \frac{120}{\te}\ln(KT)\rceil $ time steps, the estimated graph $\hat \scG_\tau$ is an $\te/2$-good approximation of $\scG$ and thus contains all the edges in $\supp{\scG_{\te}}$ by \Cref{def:eps-good-approx} (part~1) with $\e = \te/2$, and because of the conditioning on $\calE$. All in all, we can summarize these observations by noticing that
     \[
      \frac T{2K} \ge 120\frac{\ln(KT)}{\e(K,T)} \ge 120\frac{\ln(KT)}{\te}\ge \ttau \enspace,
     \]
     where the first inequality is true as long as $\e(K,T) \ge 240 K \ln(KT)/T$.
     Using point $(i)$ of \Cref{lem:prop_Phi} and the inequality we just showed, we observe that 
     \[
        \Phi(\hat \scG_{\lfloor \frac T{2K}\rfloor }, T) \le 2\Phi(\hat \scG_{\ttau}, T) \le 8 C_w\Bigl(2\frac{KT^2}{\te}\ln(KT)^2\Bigr)^{\!1/3} \!\le 8 C_w\Bigl(2\frac{KT^2}{\e(K,T)}\ln(KT)^2\Bigr)^{\!1/3} \!\le \frac{T}{2} \;,
     \]
     as long as $\e(K,T) \ge 2 \cdot 16^3 C_w^3 K (\ln(KT))^2 / T$. Note that in the previous chain of inequalities we considered the (possibly suboptimal) choice of the (weakly) observable structure of the graph with threshold $\te$ and upper bound on $\delta$ given by $K$. The inequality we just showed implies that the stopping criterion in \round is triggered and thus we can apply \Cref{lem:hatG}.
     
     Now, let $\tau^*$ be the smallest $\tau$ such that $\Phi(\scG,T)=\Phi(\scG_{\e^*},T) \le \tau K$, being $\e^*$ the optimal threshold for $\scG.$ In this second step, we want to show that $\ftau$ is not too far away from $\tau^*$ for the interesting values of $\tau^*$; namely, that $\ftau \le 4 \tau^*$ as long as $\Phi(\scG,T)$ is not $\tilde \Omega(T)$. 
     
     First, consider the case that $\Phi(\scG,T)$ refers to the strongly observable regime in $\Phi(\scG_{\e^*},T)$. By minimality of $\tau^*$, we have the following:
     \begin{equation} \label{eq:tau-star}
        \tau^* K \ge \Phi(\scG,T) = 4 C_s \sqrt{\frac{\alpha^*}{\e^*} T} \big(\ln(KT)\big)^{3/2} \ge \frac{1}{2} \tau^* K \enspace.
     \end{equation}
     We now set the constant appearing in the definition of $\varepsilon(K,T)$ from the statement to be $C = 2 \cdot 16^3 C_w^3$. With this choice, the previously stated requirements for $\e(K,T)$ are satisfied, while at the same time it holds that $ \Phi(\scG,T) \le C_s^2 T(\ln(KT))^2/(15K)$; this is immediate to verify by arguing that $\Phi(\G,T)$ is at most the regret incurred by using the (possibly suboptimal, weakly) observable structure of $\G$ truncated at $\varepsilon(K,T)$.
     Then, from the second inequality of \eqref{eq:tau-star}, it follows that $\tau^* \le 2 C_s^2 T (\ln(KT))^2/(15K^2)$. We can rewrite the first inequality of \eqref{eq:tau-star} as follows:
     \[
        \e^* \ge 16 C_s^2 \frac{\alpha^*}{(K \tau^*)^2} T \big(\ln(KT)\big)^{3} \ge 120\frac{\ln(KT)}{\tau^*} \enspace.
     \]
     Consider now to what happens at the $\overline{\tau} = \lceil 120 \ln(KT)/\e^* \rceil \le 4 \tau^*$ iteration of \round. The estimated graph $\hat\scG_{\overline{\tau}}$ in that iteration is an $\e^*/2$-good approximation of $\scG$, thus it contains all the edges of $\scG$, with the probabilities correctly estimated up to a constant multiplicative factor, as detailed in \Cref{def:eps-good-approx} (part~2). Thus,
     \[
     \Phi(\hat \scG_{4 \tau^*},T) \le 2 \Phi(\hat \scG_{\overline{\tau}},T) \le 2\sqrt 2 \Phi(\scG,T)\le 4 \tau^* K,
     \] 
     which implies that the stopping time $\ftau$ is attained before $4 \tau^*.$ Note that the first inequality is due to point $(i)$ of \Cref{lem:prop_Phi}, whereas the second inequality follows from point $(ii)$ of \Cref{lem:prop_Phi}.
     
     Similarly, we consider the case that $\Phi(\scG,T)$ refers to the weakly observable regime in $\Phi(\scG_{\e^*},T)$. By minimality of $\tau^*$, we have the following:
     \begin{equation} \label{eq:tau-star-2}
        \tau^* K \ge \Phi(\scG,T) = 4 C_w \left(\frac{\delta^*}{\e^*}\big(\ln(KT)\big)^2\right)^{1/3}T^{2/3} \ge \frac{1}{2} \tau^* K \enspace.
     \end{equation}
     By the choice of $\varepsilon(K,T)$, we have that $ \Phi(\scG,T) \le T \sqrt{2C_w^3\ln(KT)/(15 K)}$.
     Then, from the second inequality of \eqref{eq:tau-star-2}, it follows that $\tau^* \le T \sqrt{8C_w^3\ln(KT)/(15 K^3)}$. Consider now the first inequality, we can rewrite it to obtain:
     \[
        \e^* \ge 64 C_w^3 \frac{\delta^*}{(K \tau^*)^3} \bigl(T\ln(KT)\bigr)^{2} \ge 120 \frac{\ln(KT)}{\tau^*} \enspace.
     \]
     We can now use the same argument as in the strongly observable case and conclude that $\ftau \le 4 \tau^*$. 
    
    At this point, we are ready to show that our algorithm \explore exhibits the desired regret bounds. We are conditioning on the good event $\calE$; this happens with probability at least $1-\frac 1T$, so we just analyze this case, as the complementary of $\calE$ yields at most an extra additive $1$, in expectation, to the regret bound. 
    
    Recall that $R_T$ is the worst-case regret; thus,
    \[
        R_T \le \ftau K + \Phi(\hat \scG, T-\ftau K) \le 2 \ftau K \le 8 \tau^* K \le 16 \Phi(\scG,T) \enspace,
    \]
    where in the first inequality we used the decomposition in regret before and after the commitment and the bound on \Cref{lem:hatG} (which is applicable given the conditioning on $\calE$ and thus all the $\scG_{\tau}$ are $\e_{\tau}$-good approximations of $\scG$), in the second one the definition of $\ftau$, in the third one the fact that $\ftau \le 4 \tau^*$, and in the last the definition of $\tau^*$ as minimal $\tau$ such that $\Phi(\scG,T) \le \tau K$. 
    \end{proof}

\section{Proofs of Lower Bounds} \label{sec:lower-bounds-proofs}

The main idea in the lower bounds is that the adversary sets all edge probabilities equal to $\e \in (0, 1]$ in order to define a stochastic feedback graph $\scG$ with a specific support $G$ that satisfies adequate properties.
This requires the attribution of additional power to the adversary because we allow it to choose the edge probabilities; nevertheless, this is fine from a worst-case perspective because it corresponds to choosing a particularly difficult instance among those that have certain characteristics.
Doing so makes the edge between each (ordered) pair of nodes either realize independently at each round $t$ with probability equal to $\e$, or never realize.
Moreover, there exists a vertex that is at least marginally better than the other ones with respect to the expected loss.
The learner only obtains information about the loss of the optimal node whenever it plays a node that is adjacent to it in $G = \supp{\scG}$ and the edge between the played node and the optimal node is realized.
Since that edge is realized only with probability $\e$, it is significantly harder for the the learner to detect the optimal node, which allows the adversary to increase the size of the gaps between the optimal node and the suboptimal ones.
More specifically, while in the deterministic setting playing once action $a$ is enough to observe the loss incurred by a neighbouring action $a'$, the learner will now need $1/\e$ time steps, in expectation, to observe the loss of $a'$ if the edge $(a,a')$ only realizes with probability $\e$.
Further notice that, in the setting considered within the proofs of our lower bounds, the learner may even know the true distribution $\scG$ and observe the realization of the entire feedback graph $G_t$ at the end of each round $t$.

We start with a lower bound for the strongly observable case considering stochastic feedback graphs $\scG$ with $\alpha(\scG)>1$.
The following result can be recovered by adapting the proof of \citet[Theorem 5]{AlonCGMMS17} that holds for any graph of interest (directed or undirected).

\begin{theorem} \label{thm:lower-bound-strongly-observable}
    Pick any directed or undirected graph $G = (V, E)$ with $\alpha(G) > 1$ and any $\e \in (0,1]$.
    There exists a stochastic feedback graph $\scG$ with $\supp{\scG} = G$ and such that, for all $T \geq 0.0064 \alpha(\scG_\e)^3/\e$ and for any possibly randomized algorithm $\A$, there exists a sequence $\ell_1,\ldots,\ell_T$ of loss functions on which the expected regret of $\A$ with respect to the stochastic generation of $G_1, \dots, G_T \sim \scG$ is at least $0.017\sqrt{{\alpha(\scG_\e)T}/{\e}}$.
\end{theorem}
\begin{proof}
    The structure of this proof follows the same rationale of the lower bound by \citet[Theorem~5]{AlonCGMMS17} with additional considerations due to the stochasticity of the feedback graph. 
    To prove the lower bound we will use Yao's minimax principle~\citep{yao1977minimaxprinciple}, which shows that it is sufficient to provide a probabilistic strategy for the adversary on which the expected regret of any deterministic algorithm is lower bounded.
    
    We can assume that $G$ has all self-loops.
    If $G$ is missing some self-loops, we may add them for the sake of the lower bound: this only makes the problem easier for the learner.
    Also note that the addition of self-loops does not change the independence number of $G$.
    Now let $\scG$ be such that $p(i,j) \in \{0,\e\}$ and $p(i,j) = \e$ if and only if $(i,j) \in E$, for all $i,j \in V$.
    Note that $\alpha(G) = \alpha(\scG)$ and $\scG = \scG_\e$.
    We also remark that the following lower bound for such a $\scG$ will be a lower bound for the instance having a stochastic feedback graph obtained from the starting graph, without the addition of self-loops, by setting the realization probability of all its edges to $\e$.
    Without loss of generality, we order the nodes depending on an (arbitrary) independent set of $G$ of size $\alpha(G)$ so that $1, 2,\ldots,\alpha(G)$ are the nodes belonging to said independent set, and $\alpha(G)+1,\ldots, |V|$ correspond to all the other nodes in $G$.
    
    We will use the following distribution of losses.
    We sample $Z$ from some (later defined) distribution $Q$ over the independent set chosen above.
    Conditioned on $Z = i$, the loss $\ell_t(j)$ is sampled from an independent Bernoulli distribution with mean $\half$ if $j \not = i$ and $j \leq \alpha(G)$, it is sampled from an independent Bernoulli with mean $\half - \beta$ if $j = i$ for some $\beta \in [0, \tfrac{1}{4}]$, and it is set to 1 otherwise.
    
    We denote by $T_i$ the number of times node $i$ was chosen by the algorithm after $T$ rounds and denote by $\Tbad = \sum_{i > \alpha(G)} T_i$ the number of times the algorithm chooses an action not in the independent set.
    We use $\Enb_i[\cdot] = \Enb[\cdot \,|\, Z=i]$ and $\Prob_i(\cdot) = \P{\cdot \,|\, Z=i}$ to denote the expectation and probability over $(G_1,\ell_1),\ldots,(G_T,\ell_T)$
    conditioned on $Z = i$, respectively.
    We denote by $\ell_t(I_t)$ the loss of algorithm $A$ playing $I_t$ in round $t$.
    We emphasize that the complete loss sequence and the (partial) loss sequence observed by the learner may differ depending not only on the actions of the learner but also on the realization of the edges in the feedback graph.
    This last observation will be used to lower bound the regret of the learner also in terms of $\e$, the probability of an edge realization.
    
    We set $Q(i) = \frac{1}{\alpha(G)}$ if $i$ is in the independent set and $Q(i) = 0$ otherwise. Following \citet[Equation~(8)]{AlonCGMMS17} we have, for any deterministic algorithm $A$, that
    
    {\setlength{\abovedisplayskip}{-5pt}\setlength{\belowdisplayskip}{-5pt}
    \begin{equation} \label{eq:strongLBregret}
        \max_{k \in V}\Enb\left[\sumT \big(\ell_t(I_t) - \ell_t(k)\big)\right]
        \geq \beta \biggl(T - \frac{1}{\alpha(G)}\sum_{i \leq \alpha(G)} \Enb_i[T_i]\biggr) \;.
    \end{equation}}
    We now consider an auxiliary distribution $\Prob_0$, also over $(G_1,\ell_1),\ldots,(G_T,\ell_T)$, which is equivalent to the distribution $\Prob_i$ that we specified above, but with $\beta = 0$ for all nodes.
    We denote by $\Enb_0$ the corresponding expectation.
    We also denote by $\lambda_t$ the feedback set at time $t$, composed by the realization $G_t$ of the feedback graph together with the set of losses observed by the learner in round $t$, and by $\lambda^t = (\lambda_1, \ldots, \lambda_t)$ the tuple of all feedback sets up to and including round $t$.
    Since the algorithm is deterministic, its action $I_t$ in round $t$ is fully determined by $\lambda^{t-1}$.
    Therefore, $\Enb_i[T_i \,|\, \lambda^T] = \Enb_0[T_i \,|\, \lambda^T]$.
    When $\lambda^{t-1}$ is understood from the context, let $\Prob_{j,t} = \Prob_j(\cdot \,|\, \lambda^{t-1})$ be the conditional probability measure of feedback sets $\lambda_t$ at time $t$.
    We have that
    \begin{equation*}
    \begin{split}
        \Enb_i[T_i] - \Enb_0[T_i] & = \sum_{\lambda^T} \Prob_i(\lambda^T) \Enb_i[T_i \,|\, \lambda^T] - \sum_{\lambda^T} \Prob_0(\lambda^T) \Enb_0[T_i \,|\, \lambda^T] \\
        & = \sum_{\lambda^T} \Prob_i(\lambda^T) \Enb_i[T_i \,|\, \lambda^T] - \sum_{\lambda^T} \Prob_0(\lambda^T) \Enb_i[T_i \,|\, \lambda^T] \\
        & \leq T \sum_{\lambda^T \,:\, \Prob_i(\lambda^T) > \Prob_0(\lambda^T)} \big(\Prob_i(\lambda^T) - \Prob_0(\lambda^T) \big) \enspace.
    \end{split}
    \end{equation*}
    By using Pinsker's inequality and the chain rule for the relative entropy, we can further observe that
    \begin{equation*}
    \begin{split}
        \sum_{\lambda^T \,:\, \Prob_i(\lambda^T) > \Prob_0(\lambda^T)} \big(\Prob_i(\lambda^T) - \Prob_0(\lambda^T) \big)
        & \leq \sqrt{\half \KL(\Prob_0 \,\|\, \Prob_i)} \\
        & = \sqrt{\frac{1}{2} \sumT \sum_{\lambda^{t-1}} \Prob_0(\lambda^{t-1}) \KL(\Prob_{0,t} \,\|\, \Prob_{i,t})} \enspace,
    \end{split}
    \end{equation*}
    which, combined with the previous inequality, allows us to affirm that
    \begin{equation}\label{eq:strongLBpinkser}
        \Enb_i[T_i] - \Enb_0[T_i] \leq \sqrt{\frac{1}{2} \sumT \sum_{\lambda^{t-1}} \Prob_0(\lambda^{t-1}) \KL(\Prob_{0,t} \,\|\, \Prob_{i,t})} \enspace.
    \end{equation}
    
    At this point, observe that $\supp{\scG} = G = (V, E)$.
    Fix any $\lambda^{t-1}$ and consider $\KL(\Prob_{0,t} \,\|\, \Prob_{i,t})$ where, we recall, $\Prob_{0,t}(\lambda_t) = \Prob_0(\lambda_t \,|\, \lambda^{t-1})$ and $\Prob_{i,t}(\lambda_t) = \Prob_i(\lambda_t \,|\, \lambda^{t-1})$.
    Recall that $\lambda^{t-1}$ fully determines the node $I_t$ picked by the algorithm in round $t$.
    If $(I_t,i) \not\in E$, then $\Prob_{0,t}$ and $\Prob_{i,t}$ have the same distribution
    and the relative entropy term is $0$.
    If $(I_t,i) \in E$, then the loss of node $i$ in $\lambda_t$ follows a Bernoulli distribution with mean $\half$ under $\Prob_0$ and follows a Bernoulli distribution with mean $\half - \beta$ under $\Prob_i$.
    Denote by $\calE_t$ the event that edge $(I_t,i)$ is realized in $G_t$.
    Note that $\Prob_0(\calE_t) = \Prob_i(\calE_t) = \e$.
    Using the log-sum inequality and the fact that the relative entropy between the two aforementioned Bernoulli distributions is given by $\half \ln\bigl(\frac{1}{1 - 4\beta^2}\bigr)$, we can see that 
    %
    \begin{equation}\label{eq:kl-divergence-bernoulli}
    \begin{split}
        \KL(\Prob_{0,t} \,\|\, \Prob_{i,t})
        & = \KL\Bigl(\e \Prob_{0,t}(\cdot \,|\, \calE_t) + (1-\e)\Prob_{0,t}(\cdot \,|\, \overline{\calE_t}) \,\big\|\, \e \Prob_{i,t}(\cdot \,|\, \calE_t) + (1-\e)\Prob_{i,t}(\cdot \,|\, \overline{\calE_t})\Bigr) \\
        & = \KL\Bigl(\e \Prob_{0,t}(\cdot \,|\, \calE_t) + (1-\e)\Prob_{0,t}(\cdot \,|\, \overline{\calE_t}) \,\big\|\, \e \Prob_{i,t}(\cdot \,|\, \calE_t) + (1-\e)\Prob_{0,t}(\cdot \,|\, \overline{\calE_t})\Bigr) \\ 
        & \leq \e \KL\bigl(\Prob_{0,t}(\cdot \,|\, \calE_t) \,\|\, \Prob_{i,t}(\cdot \,|\, \calE_t)\bigr) + (1 - \e)\KL\bigl(\Prob_{0,t}(\cdot \,|\, \overline{\calE_t}) \,\|\, \Prob_{0,t}(\cdot \,|\, \overline{\calE_t})\bigr) \\
        & = \e \KL\bigl(\Prob_{0,t}(\cdot \,|\, \calE_t) \,\|\, \Prob_{i,t}(\cdot \,|\, \calE_t)\bigr) \\
        & = -\frac{\e}2 \ln\left(1 - 4\beta^2\right)  \leq 8\ln(4/3)\beta^2\e \enspace.
    \end{split}
    \end{equation}
With this inequality, we may upper bound the sum in the right-hand side of~\eqref{eq:strongLBpinkser} by considering, for each $t$, only the tuples $\lambda^{t-1}$ for which $i \in N^\text{out}_G(I_t)$ holds.
Indeed, the KL divergence for any other possible $\lambda^{t-1}$ is equal to $0$ because the edge $(I_t, i)$ never realizes (it is not in the support of $\scG$, hence $p(I_t, i) = 0$).
As a consequence,
    \begin{equation}\label{eq:strongLBlogsumexp}
    \begin{split}
        \sumT \sum_{\lambda^{t-1}} \Prob_0(\lambda^{t-1}) \KL(\Prob_{0,t} \,\|\, \Prob_{i,t}) 
        & \leq \sumT \Prob_0(i \in N^\text{out}_{G}(I_t)) 8\ln(4/3)\beta^2\e \\
        & = 8\ln(4/3)\beta^2\e \Enb_0[\abs{\{t: i \in N^\text{out}_{G}(I_t)\}}]  \\
        & \leq 8\ln(4/3)\beta^2\e \Enb_0[T_i + \Tbad] \enspace.
    \end{split}
    \end{equation}
    We may claim that $\Enb_0[\Tbad] \leq 0.04\sqrt{\frac{\alpha(G)}{\e}T}$, because otherwise the expected regret under $\Prob_0$ would have been at least 
    \begin{align*}
        \max_{k \in V}\Enb_0\left[\sumT (\ell_t(I_t) - \ell_t(k))\right]
        &= \Enb_0\left[\Tbad + \frac12 \sum_{j \leq \alpha(G)} T_j\right]  - \frac12 T\\
        &= \Enb_0\left[\frac12 \Tbad + \frac12 \biggl( \Tbad + \sum_{j \leq \alpha(G)} T_j\biggr)\right]  - \frac12 T \\
        &= \Enb_0\left[\frac12 \Tbad \right] \\
        &> 0.02\sqrt{\frac{\alpha(G)}{\e}T} \enspace.
    \end{align*}
    Combining \Cref{eq:strongLBpinkser,eq:strongLBlogsumexp}, and using that $\Enb_0[\Tbad] \leq 0.04\sqrt{\frac{\alpha(G)}{\e}T}$, we find that 
    \begin{align*}
        \Enb_i[T_i] - \Enb_0[T_i] & \leq 2T \beta \sqrt{\e\ln(4/3)\Enb_0\left[T_i + 0.04\sqrt{\frac{\alpha(G)}{\e}T}\right]} \enspace.
    \end{align*}
    This implies that the regret can be further lower bounded, continuing from~\eqref{eq:strongLBregret}, by 
    \begin{align*}
        & \beta \left(T - \frac{1}{\alpha(G)}\sum_{i = 1}^{\alpha(G)} \Enb_0[T_i] - \frac{1}{\alpha(G)}\sum_{i = 1}^{\alpha(G)} 2T \beta \sqrt{\e\ln(4/3)\Enb_0\left[T_i + 0.04\sqrt{\frac{\alpha(G)}{\e}T}\right]} \right) \\
        & \geq \beta \left(T - \frac{1}{\alpha(G)}\sum_{i = 1}^{\alpha(G)} \Enb_0[T_i] - 2T \beta \sqrt{\e\ln(4/3) \Enb_0\left[\frac{1}{\alpha(G)}\sum_{i = 1}^{\alpha(G)}T_i + 0.04\sqrt{\frac{\alpha(G)}{\e}T}\right]} \right) \\
        & \geq \beta T \left(1 - \frac{1}{\alpha(G)} - 2 \beta \sqrt{\e\ln(4/3) \left(\frac{T}{\alpha(G)} + 0.04\sqrt{\frac{\alpha(G)}{\e} T}\right)} \right) \;,
    \end{align*}
    where the first inequality is Jensen's inequality for concave functions and the second inequality is due to the fact that $\sum_{i = 1}^{\alpha(G)} \Enb_0[T_i] \leq T$ by definition of $T_i$. Since we assumed that $T \geq {0.0064 \alpha(G)^3}/{\e}$, 
    we have that $0.04 \sqrt{\frac{\alpha(G)}{\e} T} \leq \frac{T}{2\alpha(G)}$ and thus 
    \begin{align*}
        \max_{k \in V}\Enb\left[\sumT (\ell_t(I_t) - \ell_t(k))\right] & \geq \beta T \left(1 - \frac{1}{\alpha(G)} - 2 \beta \sqrt{\frac32 \ln(4/3) \frac{\e T}{\alpha(G)}}\right) \\
        & \geq \beta T \left(\frac12 - 2 \beta \sqrt{\frac32 \ln(4/3) \frac{\e T}{\alpha(G)}}\right) \;,
    \end{align*}
    where in the second inequality we used the assumption that $\alpha(G) \geq 2$. By setting $\beta = \frac{1}{33}\sqrt{\frac{\alpha(G)}{2\ln(4/3) \e T}} \in (0, \frac14]$, we may complete the proof as
    \begin{align*}
        \max_{k \in V}\Enb\left[\sumT (\ell_t(I_t) - \ell_t(k))\right] & \geq \frac{1}{33} \left(\frac12 - \frac{\sqrt{3}}{33}\right) \sqrt{\frac{\alpha(G)T}{2\ln(4/3) \e}} \geq 0.017 \sqrt{\frac{\alpha(G)}{\e}T} \enspace.
    \end{align*}
\end{proof}

Given that this lower bound leaves the case $\alpha(\scG)=1$ uncovered, we provide an additional lower bound that considers any feedback graph.
This new bound is tight up to logarithmic factors, for instance, in all cases where $\alpha(\scG)$ is constant.

\begin{theorem} \label{thm:lower-bound-strongly-observable-2}
    Pick any directed or undirected graph $G = (V, E)$ with $\abs{V} = K \ge 2$ and any $\e \in (0, 1]$.
    There exists a stochastic feedback graph $\scG$ with $\supp{\scG} = G$ and such that, for all $T \ge 1/(2\e)$ and for any possibly randomized algorithm $\A$, there exists a sequence $\ell_1,\ldots,\ell_T$ of loss functions on which the expected regret of $\A$ with respect to the stochastic generation of $G_1, \dots, G_T \sim \scG$ is at least $\frac1{32}\sqrt{2T/\e}$.
\end{theorem}
\begin{proof}
    Following a similar rationale as in the proof of \Cref{thm:lower-bound-strongly-observable}, we can consider $G$ to be the complete graph (with all self-loops) because the problem for it is easier than that with any other graph.
    In fact, adding edges never makes the problem harder to solve.
    Moreover, we can define $\scG$ by setting all edge probabilities to $\e$ so that $\scG_\e = \scG$ and $\supp{\scG} = G$.
    We remark that the lower bound with such a $\scG$ is also a lower bound for the instance obtained by considering the initial (possibly non-complete) graph and assigning realization probability $\e$ to all its edges.
    Applying Yao's minimax principle allows us to reduce our current aim to proving a lower bound for the expected regret of any deterministic algorithm against a randomized adversary.
    
    We can then construct the sequence of loss functions by defining their distribution.
    Let $v \in V$ be an arbitrary vertex, say, $v = 1$.
    Pick $Z \in \{-1, +1\}$ uniformly at random and define $\beta = \frac14(2\e T)^{-1/2} \in [0, \frac14]$.
    Then, let the loss at any time $t$ be independently $\ell_t(i) \sim \Bern{\frac12}$ for $i \neq 1$ while $\ell_t(1) \sim \Bern{\frac12 - \beta Z}$.
    Define $\Prob_1(\cdot) = \P{\cdot \,\middle|\, Z=+1}$ and $\Prob_2(\cdot) = \P{\cdot \,\middle|\, Z=-1}$, as well as $\Enb_1[\cdot] = \E{\cdot \,\middle|\, Z=+1}$ and $\Enb_2[\cdot] = \E{\cdot \,\middle|\, Z=-1}$.
    We also define $\Prob_0(\cdot)$ and $\Enb_0[\cdot]$, obtained in an analogous manner as the previous ones by setting $\beta = 0$.
    
    
    At this point, let $T_1$ be the number of times $t$ that the algorithm selects vertex $I_t = 1$ after $T$ rounds.
    Following a similar computation as in \Cref{eq:strongLBpinkser,eq:strongLBlogsumexp}, we first denote by $\Prob_{j,t} = \Prob_j(\cdot \,|\, \lambda^{t-1})$ the conditional probability over feedback sets $\lambda_t$, and notice that
    \begin{align}
        \Enb_1[T_1] - \Enb_2[T_1]
        &\le T\sqrt{\frac12 \sumT \sum_{\lambda^{t-1}} \Prob_2(\lambda^{t-1}) \KL(\Prob_{2,t} \,\|\, \Prob_{1,t})} \nonumber\\
        &\le T\sqrt{\e\beta T \ln\Bigl(1+\frac{4\beta}{1-2\beta}\Bigr)} \nonumber\\
        &\le 2\beta T\sqrt{2\e T} \enspace. \label{eq:strongLB2pinsker}
    \end{align}
    Conditioning on $Z=+1$, the algorithm incurs an expected instantaneous regret equal to $\beta$ whenever it picks any vertex $i \neq 1$.
    Otherwise, conditioning on $Z=-1$, the algorithm incurs the same expected instantaneous regret each time it selects vertex $1$.
    The expected regret thus becomes
    \begin{align*}
        \max_{k\in V} \E{\sumT (\ell_t(I_t) - \ell_t(k))}
        &\ge \frac12 \Enb_1[\beta(T-T_1)] + \frac12 \Enb_2[\beta T_1] \\
        &\ge \frac{\beta}2 T - \frac{\beta}2 (\Enb_1[T_1] - \Enb_2[T_1]) \\
        &\ge \beta T \biggl(\frac12 - \beta\sqrt{2\e T}\biggr)
        = \frac14 \beta T = \frac1{32} \sqrt{\frac{2T}{\e}} \enspace,
    \end{align*}
    where the third inequality follows by \Cref{eq:strongLB2pinsker}, and we also use our choice of $\beta$.
\end{proof}
We can additionally prove further lower bounds for the weakly observable case.
Here we also adapt the proof for the lower bound in the case of a deterministic feedback graph by having each edge realize only with probability $\e \in (0, 1]$ at each time step.
We make the same considerations as in the previous lower bound for strongly observable graphs.
In this case, however, we refer to \citet[Theorem~7]{AlonCDK15}.
As in the case of deterministic feedback graph, we need the following combinatorial lemma.
\begin{lemma}[{\citet[Lemma~8]{AlonCDK15}}] \label{lem:independent-set-log}
    Let $G=(V,E)$ be a directed graph over $|V|=n$ vertices, and let $W \subseteq V$ be a set of vertices whose minimal dominating set is of size $k$.
    Then, there exists an independent set $U \subseteq W$ of size $|U| \ge \frac1{50} k/\ln n$, such that any vertex of $G$ dominates at most $\ln n$ vertices of $U$.
\end{lemma}
We can then prove the desired lower bound which states what follows.
\begin{theorem} \label{thm:lower-bound-weakly-observable}
    Pick any directed or undirected, weakly observable graph $G = (V, E)$ with $\abs{V} = K$ and $\delta(G) \ge 100\ln K$, and any $\e \in (0,1]$.
    There exists a stochastic feedback graph $\scG$  with $\supp{\scG} = G$ and such that, for all $T \ge 2K/(\e\ln K)$ and for any possibly randomized algorithm $\A$, there exists a sequence $\ell_1,\ldots,\ell_T$ of loss functions on which the expected regret of $\A$ with respect to the stochastic generation of $G_1, \dots, G_T \sim \scG$ is at least $\frac1{150} \bigl(\frac{\delta(\scG_\e)}{\e \ln^2 K}\bigr)^{1/3} T^{2/3}$.
\end{theorem}
\begin{proof}
    The proof follows the steps of the lower bound from \citet[Theorem~7]{AlonCDK15}.
    As in the previous lower bounds, we use Yao's minimax principle to infer that it suffices to design a probabilistic adversarial strategy that leads to a sufficiently large lower bound for the expected regret of any deterministic algorithm.
    
    We consider any weakly observable $G = (V, E)$ having $\abs{V} = K$ vertices and $\delta(G) \ge 100\ln K$.
    Since the adversary may choose edge probabilities, it can pick them all equal to $\e$ so that $\scG = \scG_\e$ and $\supp{\scG} = G$.
    By \Cref{lem:independent-set-log} we know that $G$ contains an independent set $U$ of size $\abs{U} = m \ge \delta(G)/(50\ln K)$ such that any $v \in V$ dominates no more than $\ln K$ vertices of $U$.
    We will denote actions in $U$ as ``good'' actions, whereas all the others will be denoted as ``bad'' actions.
    Given our assumption on $\delta(G)$, we observe that $m \ge 2$.
    A further observation we can make is that $N^\text{in}_G(i) \subseteq V\setminus U$ for all $i \in U$ because $U$ is independent, meaning that we need to pick a bad action in order to be able to observe the loss of any good action.

    As similarly done in the proof of \Cref{thm:lower-bound-strongly-observable}, we sample $Z$ from our ``target'' set $U$ uniformly at random.
    This choice induces a distribution of the losses $\ell_t(i)$ for all $t$ and all $i$ independently.
    To be precise, given $\beta = m^{1/3} (32\e T\ln K)^{-1/3} \in [0, \frac14]$, the loss is $\ell_t(i) \sim \textsf{Bern}(\frac12 - \beta)$ if $i = Z$, while it is $\ell_t(i) \sim \textsf{Bern}(\frac12)$ if $i \in U$, $i \neq Z$.
    The loss is deterministically set to $\ell_t(i) = 1$ for any other vertex $i \in V\setminus U$.
    
    Taking up the same notation introduced in the proof of \Cref{thm:lower-bound-strongly-observable}, we denote by $T_i$ the number of times action $i$ is played by the deterministic algorithm after $T$ rounds, while $\Tbad = \sum_{i\in V\setminus U} T_i$.
    In particular, $I_t$ is the action chosen by the algorithm at time $t$.
    We also use $\Prob_i(\cdot) = \P{\cdot \,\middle|\, Z=i}$ and $\Enb_i[\cdot] = \E{\cdot \,\middle|\, Z=i}$ with a similar definition, including the auxiliary distribution $\Prob_0$ and the corresponding expectation $\Enb_0$ obtained by setting $\beta=0$.
    Moreover, for each good action $i$ we introduce $X_i = \sum_{t=1}^T \ind{\{I_t \in N^\text{in}_{G}(i)\}}$ to denote the number of times the algorithm picks a bad action from $N^\text{in}_{G}(i)$.
    
    Notice that we can restrict our reasoning to algorithms that have $\Tbad \le \beta T$ (otherwise reducing to this case by only introducing a factor $3$ in the regret bound), as similarly argued in the proof of \citet[Theorem~7]{AlonCDK15}.
    This implies that
    \begin{equation} \label{eq:weakLB-bad-actions-bound}
        \sum_{i \in U} X_i \le \Tbad \ln K \le \beta T\ln K
    \end{equation}
    since each $j \in V\setminus U$ dominates at most $\ln K$ vertices of $U$.
    
    Recalling \Cref{eq:strongLBpinkser}, we are interested in bounding
    \begin{equation} \label{eq:weakLBpinsker}
        \Enb_i[T_i] - \Enb_0[T_i] \le T\sqrt{\frac12 \sumT \sum_{\lambda^{t-1}} \Prob_0(\lambda^{t-1}) \KL(\Prob_{0,t} \,\|\, \Prob_{i,t})} \enspace,
    \end{equation}
    where $\Prob_{j,t} = \Prob_j(\cdot \,|\, \lambda^{t-1})$ is the conditional probability over feedback sets $\lambda_t$.
    The KL divergence in the above sum is $\KL(\Prob_{0,t} \,\|\, \Prob_{i,t}) \le 8\ln(4/3)\beta^2\e$, where we use a similar reasoning as in \Cref{eq:kl-divergence-bernoulli}.
    As a consequence,
    \begin{align*}
        \sumT \sum_{\lambda^{t-1}} \Prob_0(\lambda^{t-1}) \KL(\Prob_{0,t} \,\|\, \Prob_{i,t})
        &\le \sum_{t=1}^T \Prob_0(I_t\in N^\text{in}_{G}(i)) 8\ln(4/3)\beta^2\e \\
        &\le 4\beta^2\e \Enb_0[\abs{\{t : I_t\in N^\text{in}_{G}(i)\}}] \\
        &= 4\beta^2\e \Enb_0[X_i] \enspace,
    \end{align*}
    which together with \Cref{eq:weakLBpinsker} allows us to show that
    \begin{equation} \label{eq:weakLBpinsker-2}
        \Enb_i[T_i] - \Enb_0[T_i] \le \beta T\sqrt{2\e \Enb_0[X_i]} \enspace.
    \end{equation}
    
    Let us now consider the expected regret for the deterministic algorithm at hand.
    We know that it must be at least
    \[
        \max_{k\in V} \E{\sumT(\ell_t(I_t) - \ell_t(k))} \ge \frac1m \sum_{i \in U} \Enb_i[\beta(T - T_i)] = \beta T - \frac{\beta}{m} \sum_{i \in U} \Enb_i[T_i]
    \]
    because the algorithm incurs at least $\beta$ regret each time it picks an action different from $Z$.
    By \Cref{eq:weakLBpinsker-2,eq:weakLB-bad-actions-bound}, and using the concavity of the square root, the summation on the right-hand side is such that
    \begin{align}
        \frac1m \sum_{i \in U} \Enb_i[T_i]
        &\le \beta T\sqrt{\frac{2\e}{m} \sum_{i \in U} \Enb_0[X_i]} + \frac1m \Enb_0\left[\sum_{i \in U} T_i\right] \nonumber\\
        &\le T\sqrt{\frac{2\beta^3\e}{m} T\ln K} + \frac{T}{m} \nonumber\\
        &= \frac14 T + \frac1m T \le \frac34 T \enspace, \label{eq:weakLB-good-actions-bound}
    \end{align}
    where the equality follows by our choice of $\beta$, whereas the last inequality holds because $m \ge 2$.
    Hence, the expected regret is
    \[
        \max_{k\in V} \E{\sumT(\ell_t(I_t) - \ell_t(k))} \ge \frac{\beta}4 T
        = \frac14\! \left(\frac{m}{32\e\ln K}\right)^{1/3} \!T^{2/3}
        \ge \frac1{50}\! \left(\frac{\delta(G)}{\e\ln^2 K}\right)^{\!1/3} \!T^{2/3}\,.
    \]
\end{proof}

An additional theorem is required in order to cover the case $\delta(\scG) < 100\ln K$.
In the same way as in \citet{AlonCDK15}, we follow a simple reasoning with generic weakly observable graphs.
The following lower bound holds for weakly observable graphs of any size and is tight up to logarithmic factors for instances having $\delta(\scG) < 100\ln K$.

\begin{theorem} \label{thm:lower-bound-weakly-observable-2}
    Pick any directed or undirected, weakly observable graph $G = (V, E)$ with $\abs{V} \ge 2$ and any $\e \in (0, 1]$.
    There exists a stochastic feedback graph $\scG$ with $\supp{\scG} = G$ and such that, for all $T \ge 2\sqrt{2}/\e$ and for any possibly randomized algorithm $\A$, there exists a sequence $\ell_1, \ldots, \ell_T$ of loss functions on which the expected regret of $\A$ with respect to the stochastic generation of $G_1, \dots, G_T \sim \scG$ is at least $\frac{\sqrt{2}}{16} \e^{-1/3} T^{2/3}$.
\end{theorem}
\begin{proof}
    The proof follows a similar structure as that of \citet[Theorem~11]{AlonCDK15}.
    We consider the same instance constituted by a graph $G = (V, E)$ having $\abs{V} \ge 3$ vertices, since it is the minimum number of vertices in order for $G$ to be weakly observable.
    In fact, any graph with exactly $2$ vertices is either unobservable or strongly observable.
    By definition, there exists a vertex in this graph with no self-loop and with at least one incoming edge missing from any of the remaining vertices.
    Without loss of generality, let $v=1$ be such a vertex and let $2 \notin N^\text{in}_G(v)$ be one of the vertices without an edge towards $v$.
    We may consider the case where all edge probabilities are set to $\e$ (implying that $\scG = \scG_\e$ and $\supp{\scG} = G$), given that we essentially assume the adversary can select them.
    
    We can apply Yao's minimax principle, as usual, to reduce this problem to that of lower bounding the expected regret for any deterministic algorithm against a randomized adversary.
    Hence, we need to design a distribution for the loss functions $\ell_1, \dots, \ell_T$ provided to the algorithm.
    Let $\beta = \frac1{2\sqrt{2}} (\e T)^{-1/3} \in [0, \frac14]$ and pick $Z \in \{-1, +1\}$ uniformly at random.
    For all $t$, we choose the losses such that $\ell_t(1) \sim \Bern{1/2 - \beta Z}$, $\ell_t(2) \sim \Bern{1/2}$, and $\ell_t(j) = 1$ for all $j \neq 1$ independently.
    Similarly to the construction in the proof of \Cref{thm:lower-bound-weakly-observable}, we have ``good'' actions $\{1, 2\}$ incurring at most $\beta$ expected instantaneous regret, while all remaining actions are ``bad'' since they incur at least $1/2$ instantaneous regret in expectation.
    
    We reuse the same definitions for $T_i$ and $X_i$ as in the proof of \Cref{thm:lower-bound-weakly-observable} for any fixed deterministic algorithm.
    On the other hand, we let $\Prob_1(\cdot) = \P{\cdot \,\middle|\, Z=+1}$ and $\Prob_2(\cdot) = \P{\cdot \,\middle|\, Z=-1}$.
    We analogously define $\Enb_1[\cdot] = \E{\cdot \,\middle|\, Z=+1}$ and $\Enb_2[\cdot] = \E{\cdot \,\middle|\, Z=-1}$.
    Finally, we introduce $\Prob_0(\cdot)$ and $\Enb_0[\cdot]$ obtained as the previous ones by setting $Z=0$.
    
    Following the same rationale that led to \Cref{eq:weakLBpinsker-2}, we can show that
    \[
        \Enb_i[T_i] - \Enb_0[T_i] \le \beta T \sqrt{2\e \Enb_i[X_1]}
    \]
    for $i \in \{1, 2\}$.
    This implies, via similar steps as in \Cref{eq:weakLB-good-actions-bound}, that
    \begin{equation} \label{eq:weakLB-good-actions-bound-2}
        \frac12 \Enb_1[T_1] + \frac12 \Enb_2[T_2]
        \le \beta T\sqrt{2\e \E{X_1}} + \frac{T}2 \enspace.
    \end{equation}
    
    Finally, if $\E{X_1} > \frac1{32} \beta^{-2}\e^{-1}$, the algorithm's expected regret becomes
    \[
        \max_{k \in V} \E{\sumT (\ell_t(I_t) - \ell_t(k))}
        \ge \frac12 \E{X_1}
        > \frac1{64}\beta^{-2}\e^{-1}
        = \frac18 \e^{-1/3} T^{2/3} \enspace,
    \]
    where the last equality holds by our choice of $\beta$.
    Otherwise, when $\E{X_1} \le \frac1{32} \beta^{-2}\e^{-1}$, the right-hand side of \Cref{eq:weakLB-good-actions-bound-2} is bounded by $\frac34 T$ and thus the regret must be
    \[
        \max_{k \in V} \E{\sumT (\ell_t(I_t) - \ell_t(k))}
        \ge \frac12 \Enb_1[\beta(T-T_1)] + \frac12 \Enb_2[\beta(T-T_2)]
        \ge \frac{\beta}4 T
        = \frac{\sqrt{2}}{16} \e^{-1/3} T^{2/3} \enspace.
    \]
\end{proof}




\section{Be Optimistic If You Can, Commit If You Must}\label{app:OTC}

In this section, we describe \Cref{alg:beyondsanityalg} and the analysis we use to obtain the results of \Cref{sec:topo}.
First of all, we briefly state the rationale for the design of this new algorithm.
The main idea is similar in spirit to that of \explore: \Cref{alg:beyondsanityalg} constantly updates the estimates for the edge probabilities of the underlying $\scG$ and computes the best regret regime it can achieve.
\begin{algorithm}[h]
\textbf{Environment}: stochastic feedback graph $\scG$, sequence of losses $\ell_1, \ell_2, \dots, \ell_T$\;
\textbf{Input}: time horizon $T$ and actions $V = \{1,2,\dots,K\}$\;
\textbf{Initialize}: sample $I_1$ uniformly at random, receive $G_1$\; 
\For{$t = 2, \ldots, T$}{
\If(\tcp*[f]{optimistic phase}){\Cref{eq:stopwhen} has never been true}{
Set $\ptilt(j,i) = \tfrac{1}{t-1}\sum_{s = 1}^{t-1}\ind{\{(j,i) \in E_s\}}$ \;
Set $\phattji = \ptilt(j,i) + \UCBt $ \;
Set $\hat{\scG}_t^{\textnormal{UCB}} = \{\phattji : i,j \in V\}$\; 
Compute $\theta_t$ and $\e^{\theta}_t$ as in \Cref{eq:UCBthreshold} \;
Set $\hat{\scG}_t = \{\phattji \ind{\{\phattji \geq \e^{\theta}_t\}} : i,j \in V\}$ and $\hat{G}_t = \supp{\hat{\scG}_t}$, \; 
Compute $\pmint = \min_{i} \min_{j \in N_{\hat{G}_t}^\text{in}(i)} \phattji$\;
Set $\gamma_t = \min\left\{\bigl(\min_{s \in [2,t]} t\pmins\bigr)^{-1/2}, \half \right\}$\;
Set $\eta_{t-1} = \Bigl(16/(\min_{s \in [2,t]}(\pmins)^2) + 4 t/(\min_{s \in [2,t]}\pmins)  + \sum_{s = 2}^{t-1} \theta_s(\hat{\scG}_s)\Bigr)^{-1/2}$\;
Set $\psi_t$ to be the uniform distribution over $V$\;
Set $\qti \propto \exp\left(\eta_{t-1} \sum_{s = 2}^{t-1} \ltilt(i)\right)$\;
Set $\pi_t(i) = (1 - \gamma_t) \qti + \gamma_t \psi_t(i)$\;
}
\If(\tcp*[f]{commit phase}){\Cref{eq:stopwhen} is true for any $t'-1 < t$}{
Set $\tstar$ to the first round $t'-1$ in which \Cref{eq:stopwhen} is true\;
Set $\tilde{\scG} = \{\tilde{p}(j, i) : i,j \in V\}$ as the stochastic graph with $\tilde{p}(j, i) = \frac{1}{\tstar} \sum_{s = 1}^{\tstar}\ind{\{(j,i) \in E_s\}}$\;
Set $\hat{\scG} = \{\tilde{p}(j, i)\ind{\{\tilde{p}(j, i)\geq \e_\tstar\}} : i,j \in V\}$ with $\e_\tstar$ as in \Cref{eq:weakthreshold}\;
Set $\hat{\scG}_{\e^{\star}_{\delta, \sigma}}= \{\tilde{p}(j, i)\ind{\{\tilde{p}(j, i)\geq \e^{\star}_{\delta, \sigma}\}} : i,j \in V\}$ with $\e^{\star}_{\delta, \sigma}$ as in \Cref{eq:optimizede}\;
Set $\ptilt(j,i) = \tfrac{1}{t-1}\sum_{s = 1}^{t-1}\ind{\{(j,i) \in E_s\}}$ \;
Set $\phattji = \ptilt(j,i) + \UCBt $ \;
Set $\hat{\scG}_t^{\textnormal{UCB}} = \{\phattji : i,j \in V\}$\; 
Set $\hat{\scG}_t = \hat{\scG}_t^{\textnormal{UCB}}$ and $\hat{G}_t = \supp{\hat{\scG}_t}$\;
Set $\gamma = \min\Bigl\{\bigl(\wdelta(\hat{\scG}_{\e^{\star}_{\delta, \sigma}})\ln(KT)\bigr)^{1/3}T^{-1/3}, \half\Bigr\}$\;
Set $\eta = \sqrt{{\ln(K)}\left({2T \left(\wdelta(\hat{\scG}_{\e^{\star}_{\delta, \sigma}})/\gamma + \sigma(\hat{\scG}_{\e^{\star}_{\delta, \sigma}})\right)}\right)^{-1}}$\;
Set  $\psi_t$ according to~\eqref{eq:weaklyexploring}\;
Set $\qti \propto \exp\left(\eta \sum_{s = \tstar + 1}^{t-1} \ltilt(i)\right)$\;
Set $\pi_t(i) = (1 - \gamma) \qti + \gamma \psi_t(i)$\;
}
Sample $I_t \sim \pi_t$\;
Receive $G_t$ and $\{(i, \ell_t(i)) : i \in \outneigh{G_t}{I_t}\}$\;
Compute $\ltilt(i)$ as in \eqref{eq:IWestimator}\;
}
\caption{\optimist (\otc)}
\label{alg:beyondsanityalg}
\end{algorithm}
However, \explore has to wait until it can determine the best regret regime before actually tackling the learning task.
On the contrary, \Cref{alg:beyondsanityalg} begins by optimistically assuming that the best thresholded graph has a strongly observable support while simultaneously updating the edge probability estimates; this is made possible given the additional assumption on receiving the realized graph $G_t = (V, E_t) \sim \scG$ together with the observed losses at the end of each round $t$.
At any point in time, as soon as \Cref{alg:beyondsanityalg} finds that it can achieve a better regret regime by switching to the weakly observable one (by computing the optimal threshold on the current estimate for $\scG$), it commits to weak observability.
We can prove that this strategy is able to achieve the best possible regret over all thresholded feedback graphs, analogously to \explore, but with a dependency on the improved graph-theoretic parameters introduced in \Cref{sec:topo}.

Consequently, there are two regimes of \Cref{alg:beyondsanityalg}. In the first regime, the algorithm works under the assumption that $\supp{\scG}$ is strongly observable; in the second regime, the algorithm works under the assumption that $\supp{\scG}$ is observable. The switch happens in round $\tstar + 1$, where $\tstar$ is the first round $t-1$ in which 
\begin{align}\label{eq:stopwhen}
    \Psi_{t-1} & \geq \Lambda_{t-1},
\end{align}
is true. The term $\Psi_t$ is an upper bound on the regret after the first $t$ rounds, and is given by
\begin{align}
    \Psi_t = \min\biggl\{t,\; & 2 + {11(\ln(3K^2T^2))^2 \max_{s \in [2, t]} \theta_s(\hat{\scG}_s)} \nonumber\\
    & + \bigl(12\ln(K) + 4\sqrt{2\ln(3K^2T^2)}\bigr) {\sqrt{t\max_{s \in [2, t]} \theta_s(\hat{\scG}_s)}}\biggr\} \enspace, \label{eq:Psit}
\end{align}
where $\hat{\scG}_t$ minimizes $\theta_t$, which is defined in \Cref{eq:UCBthreshold}. The term $\theta_t(\hat{\scG}_t)$ is an upper bound the second-order term in the regret bound of Exponential Weights.
Crucially, the same term $\theta_t(\hat{\scG}_t)$ does not require us to compute a weighted independence number at each round: we can explicitly compute it in $O(K^4)$ time.
Furthermore, in \Cref{lem:strongwalpha} we show that, conditioning on the event $\kappa$, the term $\theta_t(\hat{\scG}_t)$ is upper bounded by the minimum thresholded weighted independence number of $\scG$, which in turn is useful when bounding the regret.
We recall that the event $\kappa$, introduced in \Cref{sec:topo}, corresponds to the event that
\[
\abs*{\ptilt(j,i) - \pji} \leq \UCBt, ~~ \forall (j,i) \in V \times V
\]
for all $t \geq 2$ simultaneously.

Similarly, $\Lambda_t$ is an upper bound on the regret of \Cref{alg:beyondsanityalg} \emph{if} it were to switch regime in round $t$ and is given by 
\begin{align}
    \Lambda_t = \min_\e \biggl\{ 41T^{2/3}\left(\ln(3K^2T^2)\wdelta((\hat{\scG}_{t})_{\e})\right)^{1/3} + 41\sqrt{\ln(3K^2T^2)\sigma((\hat{\scG}_{t})_{\e})T}\biggr\} \enspace, \label{eq:Lambdat}
\end{align}
where $\hat{\scG}_t = \{\tilde{p}_t(j, i)\ind{\{\tilde{p}_{t}(j, i)\geq {\logfactor}/{t}\}} : i,j \in V\}$. 
In other words, Algorithm~\ref{alg:beyondsanityalg} changes regime whenever it thinks that the regret of a (weakly) observable graph is smaller than the regret of a strongly observable graph. In the following, we prove that $\Psi_t$ and $\Lambda_t$ are indeed upper bounds on the regret, but first we state Lemma~\ref{lem:IWproperty}, which is a central result in this section. More precisely, it provides an upper bound for the cost of not using the exact edge probabilities $\pji$ but instead using upper confidence bound estimates $\phattji$. Note that the bound scales with $\bar{\pi}_t(i) = \sum_{j \in N_{\hat{G}_t}^\textup{in}(i)} \pi_t(j)$. For $\scG_\e$ having a strongly observable support, this is an important property of the bound since we require that $\bar{\pi}_t(i) \leq 1-\pi_t(i)$ for vertices $i$ without a self-loop in $\supp{\scG_\e}$ to ensure that we can bound the regret in terms of the weighted independence number. 
\begin{lemma}\label{lem:IWproperty}
    Define $\bar{\pi}_t(i) = \sum_{j \in N_{\hat{G}_t}^\textup{in}(i)} \pi_t(j)$. 
    For any distribution $u$ over $[K]$ and $\tstar \leq T$, with estimator~\eqref{eq:IWestimator} we have that 
    \begin{align*}
        & \E{\sumTstop \sumK (\pi_t(i) - u(i))\ell_t(i)} \leq 2 + \sumTsto\E{\frac{6\ln(3K^2T^2)}{t-1}\sumK  \frac{\pi_t(i)\bar{\pi}_t(i)}{\hat{\Pp}_t(i)}~\middle|~\kappa}\\
        & \quad  + \E{\sumTsto 2\sqrt{2\frac{\ln(3K^2T^2)}{t-1}} \sqrt{\sumK  \frac{\pi_t(i)\bar{\pi}_t(i)}{\hat{\Pp}_t(i)}}~\middle|~\kappa} + \E{\sumTsto \sumK (\pi_t(i) - u(i))\ltilt(i)~\middle|~\kappa}.
    \end{align*}
\end{lemma}
\begin{proof}
    For $t > 1$, by the empirical Bernstein bound \citep[Theorem 1]{audibert2007tuning}, with probability at least $1 - \frac{1}{K^2T^2}$ we have that
    \begin{align}\label{eq:ucbbernstein}
        \abs*{\ptiltji - \pji} & \leq \sqrt{2\frac{\overline{\sigma}_t^2\ln(3K^2T^2)}{t-1}} + \frac{3}{t-1} \ln(3K^2T^2) \nonumber \\
        & \leq \UCBthat  \enspace,
    \end{align}
    where we used the fact that
    \[
    \overline{\sigma}_t^2 = \frac{1}{t-1}\sum_{s' = 1}^{t-1} \biggl(\ind{\{(j, i) \in E_{s'}\}} - \ptiltji\biggr)^{\!2} \leq \ptiltji \leq \phattji \enspace.
    \]
    Thus, by the union bound over $K^2$ edges and $\tstar$ rounds, we have that equation \eqref{eq:ucbbernstein} holds for all edges and time steps $t\ge 2$ with probability at least $1-\frac{1}{T}$. This means that $\P{\kappa} \ge 1-\frac1T$ by definition of $\kappa$.
    
    By using the tower rule and the fact that $\ell_t(i) \in [0, 1]$, we can see that
    \begin{align} \label{eq:IWprop-conditioned}
        & \E{\sumTstop \sumK (\pi_t(i) - u(i))\ell_t(i)} \nonumber\\
        & = \P{\BADEVENT} \E{\sumTstop \sumK (\pi_t(i) - u(i))\ell_t(i) \,\middle|\, \BADEVENT} + (1 - \P{\kappa}) \E{\sumTstop \sumK (\pi_t(i) - u(i))\ell_t(i) \,\middle|\, \kappa} \nonumber\\
        & \leq \P{\BADEVENT} T + \E{\sumTstop \sumK (\pi_t(i) - u(i))\ell_t(i)~\middle|~\kappa} \nonumber\\
        & \leq 2 + \E{\sumTsto \sumK (\pi_t(i) - u(i))\ell_t(i)~\middle|~\kappa} \;.
    \end{align}
    Let $\idseeitx = \idseeit$ be the indicator of the event that $i$ belongs to both $\outneigh{G_t}{I_t}$ and $\outneigh{\hat{G}_t}{I_t}$, and let $\xi_{t}(i) = \hat{\Pp}_t(i) - \Pp_t(i) = \sum_{j \in N_{\hat G_t}^\text{in}(i)} \pi_t(j)(\phattji - \pji)$.
    We continue by applying \Cref{lem:additivebias} on the expectation in the right-hand side of~\eqref{eq:IWprop-conditioned}, obtaining that
    \begin{align*}
        & \E{\sumTsto \sumK (\pi_t(i) - u(i))\ell_t(i)~\middle|~\kappa} \\
        & = \E{\sumTsto \sumK (\pi_t(i) - u(i))\ltil_t(i)~\middle|~\kappa} + \E{\sumTsto \sumK (\pi_t(i) - u(i))\xi_t(i) \frac{\idseeitx \ell_t(i)}{\Pp_t(i)\hat{\Pp}_t(i)}~\middle|~\kappa} \\
        & \leq \E{\sumTsto \sumK (\pi_t(i) - u(i))\ltil_t(i)~\middle|~\kappa} + \E{\sumTsto \sumK \pi_t(i) \xi_t(i) \frac{\idseeitx \ell_t(i)}{\Pp_t(i)\hat{\Pp}_t(i)}~\middle|~\kappa} \;,
    \end{align*}
    where the inequality is due to the fact that the loss is nonnegative and the fact that $\xi_t(i) > 0$ because $\phattji - \pji > 0$ is true, given $\kappa$.
    We already know that $\phattji \ge \ptilt(j,i)$ by definition of $\phattji$.
    As long as $\kappa$ holds, we also know that $\ptilt(j,i) - \pji \leq \UCBt$ is true.
    Then, we can use all the above observations to demonstrate that the term $\xi_t(i)$ satisfies
    \begin{align}
        \xi_{t}(i) & = \sum_{j \in N_{\hat G_t}^\text{in}(i)} \pi_t(j)(\phattji - \pji) \nonumber\\
        & \leq \sum_{j \in N_{\hat G_t}^\text{in}(i)} \pi_t(j)\left(\UCBthat\right) \nonumber\\
        & \quad + \sum_{j \in N_{\hat G_t}^\text{in}(i)} \pi_t(j)\left(\ptilt(j,i) - \pji\right) \nonumber\\
        & \leq 2\sum_{j \in N_{\hat G_t}^\text{in}(i)} \pi_t(j) \left(\UCBthat\right) \label{eq:partialxibound}
    \end{align}
    %
    %
    By the Cauchy-Schwarz inequality, it holds that
    \begin{align*}
        \sum_{j \in N_{\hat G_t}^\text{in}(i)} \pi_t(j) \sqrt{a_j}
        = \sum_{j \in N_{\hat G_t}^\text{in}(i)} \sqrt{\pi_t(j)} \sqrt{\pi_t(j)a_j}
        \le \sqrt{\bar{\pi}_t(i) \sum_{j \in N_{\hat G_t}^\text{in}(i)} \pi_t(j)a_j}
    \end{align*}
    with $a_j \ge 0$ for all $j \in \inneigh{\hat G_t}{i}$, where we recall that $\bar{\pi}_t(i) = \sum_{j \in N_{\hat G_t}^\text{in}(i)} \pi_t(j)$.
    We can use this property to further bound $\xi_{t}(i)$ in~\eqref{eq:partialxibound} as
    \begin{align*}
        \xi_{t}(i) & \leq 2\sum_{j \in N_{\hat G_t}^\text{in}(i)} \pi_t(j) \left(\UCBthat\right) \nonumber\\
        & \leq 2 \sqrt{2\bar{\pi}_t(i)\frac{\hat{\Pp}_t(i)\ln(3K^2T^2)}{t-1}} + \bar{\pi}_t(i)\frac{6\ln(3K^2T^2)}{t-1} \enspace.
    \end{align*}
    At this point, we can use the inequality for $\xi_t(i)$ to show that
    \begin{align}\label{eq:firstxibound}
        & \E{\sumTsto \sumK \pi_t(i) \xi_t(i) \frac{\idseeitx \ell_t(i)}{\Pp_t(i)\hat{\Pp}_t(i)}~\middle|~\kappa} \nonumber \\
        & \leq \E{\sumTsto 2\sqrt{2\frac{\ln(3K^2T^2)}{t-1}} \sumK \pi_t(i) \frac{\idseeitx \ell_t(i)\sqrt{\bar{\pi}_t(i)}}{\Pp_t(i)\sqrt{\hat{\Pp}_t(i)}}~\middle|~\kappa} \nonumber \\
        & \quad + \sumTsto\E{\frac{6\ln(3K^2T^2)}{t-1}\sumK \pi_t(i)\bar{\pi}_t(i) \frac{\idseeitx \ell_t(i)}{\Pp_t(i)\hat{\Pp}_t(i)}~\middle|~\kappa}\nonumber \\
        & \leq \E{\sumTsto 2\sqrt{2\frac{\ln(3K^2T^2)}{t-1}} \sumK  \pi_t(i)\sqrt{\frac{\bar{\pi}_t(i)}{\hat{\Pp}_t(i)}}~\middle|~\kappa} \\
        & \quad + \sumTsto\E{\frac{6\ln(3K^2T^2)}{t-1}\sumK  \frac{\bar{\pi}_t(i)\pi_t(i)}{\hat{\Pp}_t(i)}~\middle|~\kappa} \nonumber  \\
        & \leq \E{\sumTsto 2\sqrt{2\frac{\ln(3K^2T^2)}{t-1}} \sqrt{\sumK  \frac{\pi_t(i)\bar{\pi}_t(i)}{\hat{\Pp}_t(i)}}~\middle|~\kappa} \nonumber\\
        & \quad +\sumTsto\E{\frac{6\ln(3K^2T^2)}{t-1}\sumK  \frac{\pi_t(i)\bar{\pi}_t(i)}{\hat{\Pp}_t(i)}~\middle|~\kappa} \nonumber,
    \end{align}
    where in the second inequality we used the fact that $\ell_t(i) \leq 1$ and that $\Enb_{t-1}[\idseeitx] = \Pp_t(i)$, while the final inequality is Jensen's inequality for concave functions.
    
    By combining the above, we may complete the proof:
    \begin{align*}
        & \E{\sumTstop \sumK (\pi_t(i) - u(i))\ell_t(i)} \leq 2 + \sumTsto\E{\frac{6\ln(3K^2T^2)}{t-1}\sumK  \frac{\pi_t(i)\bar{\pi}_t(i)}{\hat{\Pp}_t(i)}~\middle|~\kappa}\\
        & \quad  + \E{\sumTsto 2\sqrt{2\frac{\ln(3K^2T^2)}{t-1}} \sqrt{\sumK  \frac{\pi_t(i)\bar{\pi}_t(i)}{\hat{\Pp}_t(i)}}~\middle|~\kappa} + \E{\sumTsto \sumK (\pi_t(i) - u(i))\ltilt(i)~\middle|~\kappa}.
    \end{align*}
    %
\end{proof}

\subsection{Initial Regime of \textnormal{\otc}}

To understand the initial regime of \otc (\Cref{alg:beyondsanityalg}), consider the following. Since the support of $\hat{\scG}_t^{\text{UCB}}$ is the complete graph, there always exists a threshold $\e$ for which $\supp{(\hat{\scG}_{t}^{\text{UCB}})_{\e}}$ is strongly observable. For ease of notation, given any stochastic feedback graph $\scG$ with edge probabilities $p(j,i)$, we introduce
\begin{align*}
    P_t(i, \scG) = \sum_{j \in N_{\supp{\scG}}^\text{in}(i)} \pi_t(j) \pji \enspace.
\end{align*}
Denote by $\scS$ the family of strongly observable graphs over vertices $V = [K]$; 
we can then define $\e^{\theta}_t$ as
\begin{align}\label{eq:UCBthreshold}
    \e^\theta_t
    &= \argmin_{\e \,:\, \supp{(\hat{\scG}^{\text{UCB}}_{t})_{\e}} \in \scS} \theta_t((\hat{\scG}^{\text{UCB}}_{t})_{\e}) \nonumber\\
    &= \argmin_{\e \,:\, \supp{(\hat{\scG}^{\text{UCB}}_{t})_{\e}} \in \scS} \biggl(\frac{2}{\min_{i} \min_{j \in \inneigh{\supp{(\hat{\scG}^{\text{UCB}}_{t})_{\e}}}{i}} \phattji}  + \sum_{i \in \inneigh{\supp{(\hat{\scG}^{\text{UCB}}_{t})_{\e}}}{i}} \frac{2\pi_t(i)}{P_t(i, (\hat{\scG}^{\text{UCB}}_{t})_{\e})}\biggr) \,.
\end{align}
A crucial property of $\hat\scG_t$ (that is, $\hat{\scG}_{t}^{\text{UCB}}$ thresholded at $\e^\theta_t$) is that, if $\phattji \geq \pji$ for all edges $(j, i)$, by \Cref{lem:strongwalpha} we have that
\begin{align*}
    \min_{\e \,:\, \supp{(\hat{\scG}^{\text{UCB}}_{t})_{\e}} \in \scS} \theta_t((\hat{\scG}^{\text{UCB}}_{t})_{\e}) = \tilde{O}\left( \min_{\e \,:\, \supp{\scG_\e} \in \scS} \walpha(\scG_\e) \right) \enspace,
\end{align*}
which is a property we will use when computing the final regret bound of \Cref{alg:beyondsanityalg}.
It also ensures that we can bound the cost of not knowing $\pji$ in \Cref{lem:IWproperty} by $\min_{\e \,:\, \supp{(\hat{\scG}^{\text{UCB}}_{t})_{\e}} \in \scS} \theta_t((\hat{\scG}^{\text{UCB}}_{t})_{\e})$, which is also important in computing the final regret bound of Algorithm~\ref{alg:beyondsanityalg}.
We thus upper bound the regret of the initial regime of \otc in terms of $\theta_t$ in what follows.

%
\begin{lemma}\label{lem:stronglytuned}
    For any distribution $u$ over $[K]$, after $\tstar \leq T$ rounds \Cref{alg:beyondsanityalg} guarantees
    \begin{align*}
        & \E{\sumTstop \sumK (\pi_t(i) - u(i))\ltilt(i) \;\middle|\; \kappa} \leq 2 + 11(\ln(3K^2T^2))^2\,\E{\max_{t \in [2, \tstar]} \theta_t(\hat{\scG}_t) \;\middle|\; \kappa} \\
        & \quad  + \Bigl(12\ln(K) + 4\sqrt{2\ln(3K^2T^2)}\Bigr) \E{\sqrt{\tstar\max_{t \in [2, \tstar]} \theta_t(\hat{\scG}_t)} \;\middle|\; \kappa} \;.
    \end{align*}
\end{lemma}
\begin{proof}
    We start with an application of \Cref{lem:IWproperty}:
    \begin{align*}
        & \E{\sumTstop \sumK (\pi_t(i) - u(i))\ell_t(i)} \leq 2 + \sumTsto\E{\frac{6\ln(3K^2T^2)}{t-1}\sumK  \frac{\pi_t(i)\bar{\pi}_t(i)}{\hat{\Pp}_t(i)} \;\middle|\; \kappa} \\
        & \quad  + \E{\sumTsto 2\sqrt{2\frac{\ln(3K^2T^2)}{t-1}} \sqrt{\sumK  \frac{\pi_t(i)\bar{\pi}_t(i)}{\hat{\Pp}_t(i)}} \;\middle|\; \kappa} + \E{\sumTsto \sumK (\pi_t(i) - u(i))\ltilt(i) \;\middle|\; \kappa} \;,
    \end{align*}
    where, we recall it, $\bar{\pi}_t(i) = \sum_{j \in N_{\hat{G}_t}^\textup{in}(i)} \pi_t(j)$. Now, for $i$ without a self-loop in $\hat{G}_t$ we have that $\bar{\pi}_t(i) \leq 1 - \pi_t(i)$. Now, conditioning on $\kappa$, we may follow the reasoning surrounding \Cref{eq:strongsecondorderbound} to find that 
    \begin{align*}
        \sumK \frac{\pi_t(i)\bar{\pi}_t(i)}{\hat{\Pp}_t(i)} 
        \leq \theta_t(\hat{\scG}_t) \enspace.
    \end{align*}
    We now use $\sumT \frac{1}{t} \leq \ln(T) + 1$, $\sumT \frac{1}{\sqrt{t}} \leq 2\sqrt{T}$, and the above inequality to obtain that
    \begin{align*}
        & \E{\sumTstop \sumK (\pi_t(i) - u(i))\ell_t(i)} \leq 2 + \sumTsto\E{\frac{6\ln(3K^2T^2)}{t-1}\theta_t(\hat{\scG}_t) \;\middle|\; \kappa}\\
        & \quad  + \E{\sumTsto 2\sqrt{2\frac{\ln(3K^2T^2)}{t-1}} \sqrt{\theta_t(\hat{\scG}_t)} \;\middle|\; \kappa} + \E{\sumTsto \sumK (\pi_t(i) - u(i))\ltilt(i) \;\middle|\; \kappa} \\
        & \leq 2 + 6(\ln(3K^2T^2))^2\,\E{\max_{t \in [2, \tstar]} \theta_t(\hat{\scG}_t) \;\middle|\; \kappa} + \E{4\sqrt{2\ln(3K^2T^2)\tstar\max_{t \in [2, \tstar]} \theta_t(\hat{\scG}_t)} \;\middle|\; \kappa} \\
        & \quad + \E{\sumTsto \sumK (\pi_t(i) - u(i))\ltilt(i) \;\middle|\; \kappa} \;.
    \end{align*}
    By applying \Cref{lem:stronglyuntuned}, we can complete the proof:
    \begin{align*}
        & \E{\sumTstop \sumK (\pi_t(i) - u(i))\ell_t(i)} \leq 2 + 6(\ln(3K^2T^2))^2\,\E{\max_{t \in [2, \tstar]} \theta_t(\hat{\scG}_t) \;\middle|\; \kappa} \\
        & \quad + \E{4\sqrt{2\ln(3K^2T^2)\tstar\max_{t \in [2, \tstar]} \theta_t(\hat{\scG}_t)} \;\middle|\; \kappa} \\
        & \quad + \E{7\ln(K)\sqrt{\sumTsto \theta_t(\hat{\scG}_t)} + \max_{t \in [2, \tstar]} \theta_t(\hat{\scG}_t)   \;\middle|\;  \kappa}\\
        & \quad + \E{\max_{t \in [2, \tstar]} \frac{4\ln(K)}{\pmint} + 5\ln(K)\sqrt{\max_{t \in [2, \tstar]} \frac{\tstar}{\pmint}}  \;\middle|\;  \kappa}\\
        & \leq 2 + 11(\ln(3K^2T^2))^2\,\E{\max_{t \in [2, \tstar]} \theta_t(\hat{\scG}_t) \;\middle|\; \kappa}\\
        & \quad  + \Bigl(12\ln(K) + 4\sqrt{2\ln(3K^2T^2)}\Bigr) \E{\sqrt{\tstar\max_{t \in [2, \tstar]} \theta_t(\hat{\scG}_t)} \;\middle|\; \kappa},
    \end{align*}
    where we used that $\frac{1}{\pmint} \leq \theta_t(\hat{\scG}_t)$ for all $t \in [2, \tstar]$.
\end{proof}

In the proof of \Cref{lem:stronglytuned} we make use of the following auxiliary result, which bounds the regret of $\pi_t$ given $\kappa$. 

\begin{lemma}\label{lem:stronglyuntuned}
    For any distribution $u$ over $[K]$, after $\tstar \leq T$ rounds Algorithm~\ref{alg:beyondsanityalg} guarantees
    %
    \begin{align*}
         \E{\sumTsto \sumK (\pi_t(i) - u(i))\ltilt(i)  \;\middle|\;  \kappa} 
        & \leq \E{7\ln(K)\sqrt{\sumTsto \theta_t(\hat{\scG}_t)} + \max_{t \in [2, \tstar]} \theta_t(\hat{\scG}_t)   \;\middle|\;  \kappa}\\
        & \quad + \E{\max_{t \in [2, \tstar]} \frac{4\ln(K)}{\pmint} + 5\ln(K)\sqrt{\max_{t \in [2, \tstar]} \frac{\tstar}{\pmint}}  \;\middle|\;  \kappa}.
    \end{align*}
\end{lemma}
\begin{proof}
    We want to apply Lemma \ref{lem:EW}, which bounds the regret of Exponential Weights.
    Recall that \Cref{alg:beyondsanityalg} defines
    \begin{align*}
        \pmint = \min_{i \in V} \min_{j \in \inneigh{\supp{\hat{\scG}_t}}{i}} \phattji
    \end{align*}
    as the minimum (positive) edge probability in ${\hat{\scG}_t}$. 
    Observe that for any node $i$ without a self-loop in $\supp{\hat{\scG}_t}$ we have that
    \begin{align}\label{eq:LBpnoself}
        \hat{\Pp}_t(i) & = \sum_{j \not = i} \phattji \left((1 - \gamma_t) \qti + \frac{\gamma_t}{K}\right) \nonumber \\
        & \geq \pmint \sum_{j \not = i} \left((1 - \gamma_t) \qti + \frac{\gamma_t}{K}\right) \nonumber \\
        & = (1 - \pi_t(i)) \pmint \\
        & = \left(1 - (1 - \gamma_t) \qti - \frac{\gamma_t}{K}\right) \pmint \nonumber \\
        & \geq \frac{\gamma_t}{2}\pmint \enspace. \nonumber
    \end{align}
    Using \eqref{eq:LBpnoself}  and the definitions of $\eta_{t-1}$ and $\gamma_t$, together with the fact that $\ell_t(i) \in [0,1]$, we can see that 
    \begin{align*}
        \eta_{t-1} \ltilt(i) & \leq \eta_{t-1} \frac{1}{\hat{\Pp}_t(i)}
        \leq \eta_{t-1} \frac{2}{\gamma_t \pmint} \leq 1 \enspace,
    \end{align*}
    where the last inequality is due to the fact that $\eta_{t-1} \leq \half \gamma_t \pmint $.
    Given event $\kappa$, since for any node $i$ without a self-loop in $\supp{\hat{\scG}_t}$ we have that $\eta_{t-1} \ltilt(i) \leq 1$, we may apply Lemma \ref{lem:EW} with $S_t = S = \{i : i \not \in \inneigh{\supp{\hat{\scG}_t}}{i}\}$ to obtain that
    \begin{align*}
        & \E{\sumTsto \sumK (\qti - u(i))\ltilt(i)  \;\middle|\;  \kappa} \\
        & \leq \E{\frac{\ln K}{\eta_{\tstar}} + \sumTsto \eta_{t-1}\left(\sum_{i \in S_t}q_t(i)(1 - q_t(i))\ltilt(i)^2 + \sum_{i \not \in S_t} q_t(i)\ltilt(i)^2\right)  \;\middle|\;  \kappa} \;.
    \end{align*}
    We now bound
    \begin{align*}
        \E{\sum_{i \in S_t}q_t(i)(1 - q_t(i))\ltilt(i)^2 ~ \middle| ~ \kappa} &  = \E{\sum_{i \in S_t}q_t(i)(1 - q_t(i))\frac{\Pp_t(i)\ell_t(i)^2}{\hat{\Pp}_t(i)(\Pp_t(i) + \xi_t(i))} ~ \middle| ~ \kappa}\\
        & \leq \E{\sum_{i \in S_t}q_t(i)\frac{(1 - q_t(i))}{\hat{\Pp}_t(i)} ~ \middle| ~ \kappa} \\
        & = \E{\sum_{i \in S_t} \frac{q_t(i){(1 - q_t(i))}}{P_t(i, \hat{\scG}_t)} ~ \middle| ~ \kappa} \\
        & \leq \E{\sum_{i \in S_t} \frac{2 q_t(i)}{\pmint} ~ \middle| ~ \kappa} \leq \E{\frac{2}{\pmint} ~ \middle| ~ \kappa} \;.
    \end{align*}
    For $i \not \in S_t$, since $\pi_t(i) \geq \half \qti$ and $\hat{\Pp}_t(i) - \Pp_t(i) \geq 0$ given $\kappa$, we have that 
    \begin{align*}
        \E{\sum_{i \not \in S_t} q_t(i)\ltilt(i)^2 \;\middle|\; \kappa} 
        \leq \E{\sum_{i \not \in S_t} \frac{q_t(i)}{\hat{\Pp}_t(i)} \;\middle|\; \kappa}
        \leq \E{\sum_{i \not \in S_t} \frac{2\pi_t(i)}{P_t(i, \hat{\scG}_t)} \;\middle|\; \kappa} \;,
     \end{align*}
    which combined with the preceding inequality means that, given $\kappa$, we have that
    \begin{align}\label{eq:strongsecondorderbound}
        \sum_{i \in S_t}\frac{q_t(i)(1 - q_t(i))}{\hat{\Pp}_t(i)} + \sum_{i \not \in S_t} \frac{q_t(i)}{\hat{\Pp}_t(i)} \leq \frac{2}{\pmint}  + \sum_{i \not \in S_t} \frac{2\pi_t(i)}{P_t(i, \hat{\scG}_t)} = \theta_t(\hat{\scG}_t) \enspace.
    \end{align}
    Therefore, we have that
    \begin{align*}
        & \E{\sumTsto \sumK (\qti - u(i))\ltilt(i)  \;\middle|\;  \kappa} \\
        &\quad \leq \E{\frac{\ln K}{\eta_{\tstar}} + \sumTsto \eta_{t-1}\left(\sum_{i \in S_t}\frac{q_t(i){(1 - q_t(i))}}{P_t(i, \hat{\scG}_t)} + \sum_{i \not \in S_t} \frac{q_t(i)}{P_t(i, \hat{\scG}_t)}\right)  \;\middle|\;  \kappa} \\
        &\quad \le \E{\frac{\ln K}{\eta_{\tstar}} + \sumTsto \eta_{t-1} \theta_t(\hat{\scG}_t)  \;\middle|\;  \kappa} \enspace.
    \end{align*}
    Now, using a slightly modified version of \citep[Lemma~14]{gaillard2014second} (replacing $|a_i| \leq 1$ by $|a_i| \leq \max_i |a_i|$) we can see that 
    \begin{align*}
        \sumTsto \eta_{t-1} \theta_t(\hat{\scG}_t) & \leq \sumTsto \theta_t(\hat{\scG}_t) \sqrt{1 + \sum_{s = 2}^{t-1} \theta_s(\hat{\scG}_s)} \\
        & \leq 3\sqrt{\sumTsto \theta_t(\hat{\scG}_t)} + \max_{t \in [2, \tstar]} \theta_t(\hat{\scG}_t) \enspace.
    \end{align*}
    As a final step in this proof, we want to consider the distribution $\pi_t$ the algorithm actually samples actions from instead of $q_t$.
    We can bound $\sumTsto \gamma_t \leq 2\sqrt{\max_{t \in [2, \tstar]} \frac{\tstar}{\pmint}}$ and
    \begin{align*}
        \frac{1}{\eta_\tstar} \leq \frac{4}{\min_{t\in [2,\tstar]}\pmint} + \sqrt{ \frac{\tstar}{\min_{t \in [2, \tstar]} \pmint}} + \sqrt{\sumTsto \theta_t(\hat{\scG}_t)} \enspace.
    \end{align*}
    Thus, combining the above we find that 
    \begin{align*}
        \E{\sumTsto \sumK (\pi_t(i) - u(i))\ltilt(i)  \;\middle|\;  \kappa}
        &\leq \E{\sumTsto \sumK (\qti - u(i))\ltilt(i) + \sumTsto \gamma_t  \;\middle|\;  \kappa} \\
        & \leq \E{7\ln(K)\sqrt{\sumTsto \theta_t(\hat{\scG}_t)} + \max_{t \in [2, \tstar]} \theta_t(\hat{\scG}_t)   \;\middle|\;  \kappa}\\
        & \quad + \E{\max_{t \in [2, \tstar]} \frac{4\ln(K)}{\pmint} + 5\ln(K)\sqrt{\max_{t \in [2, \tstar]} \frac{\tstar}{\pmint}}  \;\middle|\;  \kappa} \;.
    \end{align*}
\end{proof}

\subsection{Regret After Round \texorpdfstring{$\tstar$}{t*}}

With \Cref{lem:stronglytuned} at hand, we can control the regret in the first $\tstar$ rounds. However, we also need to control the regret in the remaining rounds, which we show how to do here. Recall that $\tilde{\scG}$ is the graph with edge probabilities $\tilde{p}(j, i) = \frac{1}{\tstar} \sum_{s = 1}^{\tstar}\ind{\{(j,i) \in E_s\}}$. At the end of round~$\tstar$ we have that $\hat{\scG} = \tilde{\scG}_{\e_{\tstar}}$ is an $\e_\tstar$-good approximation of $\scG$ with high probability, where
\begin{align}\label{eq:weakthreshold}
    \e_\tstar = \frac{\logfactor}{\tstar} \enspace.
\end{align}
We set 
\begin{align} \label{eq:optimizede}
    \e^{\star}_{\delta, \sigma} = \argmin_{\e \,:\, \supp{\hat{\scG}_\e} \textnormal{ observable}} (\wdelta(\hat{\scG}_\e)\ln(3K^2T^2))^{1/3}T^{2/3} + \sqrt{\sigma(\hat{\scG}_\e)T \ln(3K^2T^2)}
\end{align}
and define the corresponding stochastic graph by $\hat{\scG}_{\e^{\star}_{\delta, \sigma}} = \{\tilde{p}(j, i)\ind{\{\tilde{p}(j, i)\geq \e^{\star}_{\delta, \sigma}\}} : i,j \in V\}$. We denote its support by $\hat{G}^{\star} = \supp{\hat{\scG}_{\e^{\star}_{\delta, \sigma}}}$. We also require any estimated minimum weight weakly dominating set in round $t$, given by
\begin{align*}
    D^{\star}_t = \argmin_{D \in \mathcal{D}(\hat{G}^{\star})} \sum_{i \in D} \frac{1}{\min_{j \in N^{\text{out}}_{\hat{G}^{\star}}(i)} \phattij} \enspace,
\end{align*}
where $\mathcal{D}(\hat{G}^{\star})$ corresponds to the family of weakly dominating sets in $\hat{G}^{\star}$.
We define %
\begin{align}\label{eq:weaklyexploring}
    \psi_t(i) \propto 
    \begin{cases}
    \bigl(\min_{j\in N^\text{out}_{\hat{G}^{\star}}(i)} \phattij \bigr)^{-1} & \textnormal{ for $i \in D^{\star}_t$ } \\
    0 & \textnormal{ for $i \not \in D^{\star}_t$}
    \end{cases}        
\end{align}
to be the exploration distribution in round $t$. Note that this distribution is non-uniform over the weakly dominating set $D^{\star}_t$. This is because we want to ensure that the loss of each node is observed roughly equally often. If we were to sample uniformly at random, then this would not be possible because the probability that an edge realizes is not necessarily identical for all edges; however, note that the distribution is in fact uniform if the estimated edge probabilities are uniform.

\begin{lemma}\label{lem:weaklyuntuned}
    Suppose that $\hat{\scG}$ is an $\e_\tstar$-good approximation of $\scG$. For any distribution $u$ over~$[K]$, \Cref{alg:beyondsanityalg} guarantees
    %
    \begin{align*}
        & \E{\sumTwea \sumK (\pi_t(i) - u(i))\ltilt(i)  \;\middle|\;  \kappa}\\
        & \leq 16 \wdelta(\hat{\scG}_{\e^{\star}_{\delta, \sigma}}) \ln(3K^2T^2) + 5(\wdelta(\hat{\scG}_{\e^{\star}_{\delta, \sigma}})\ln(3K^2T^2))^{1/3} T^{2/3} + 4\sqrt{\sigma(\hat{\scG}_{\e^{\star}_{\delta, \sigma}}) T\ln(K)} \enspace.
    \end{align*}
\end{lemma}
\begin{proof}
    Consider the set $S = \{i \,:\, i \not \in N_{\hat{G}^{\star}}^\text{in}(i)\}$ of nodes without a self-loop in $\hat{G}^{\star}$.
    Observe that for any node $i \in S$, given $\kappa$, we have that for some node $k \in D_t^{\star}$ with $t>\tstar$,
    \begin{align*}\label{eq:LBpnoselfweakly}
        \hat{\Pp}_t(i) & = \sum_{j \not = i} \phattji \left((1 - \gamma) \qti + {\gamma}\psi_t(i)\right) \nonumber \\
        & \geq \gamma \hat{p}_t(k, i) \psi_{t}(k) \nonumber\\
        & \geq \frac{\gamma }{\sum_{k \in D^{\star}_t} \bigl(\min_{j\in N^\text{out}_{\hat{G}^{\star}}(k)} \hat{p}_t(k,j) \bigr)^{-1}} \enspace.
    \end{align*}
    Observe that $\E{ \phattji \;\middle|\;  \kappa} \geq \pji \geq \half \tilde{p}(j,i)$ for all edges $(j,i)$ in $\hat{G}^\star$ by definition of $\e_{\tstar}$-good approximation. This implies that
    \begin{align}
        \hat{\Pp}_t(i) \geq \frac{\gamma}{2\wdelta(\hat{\scG}_{\e^{\star}_{\delta, \sigma}})}
    \end{align}
    holds for any node $i \in S$, conditioning on $\kappa$.
    We apply \Cref{lem:EW} with $S_t = \emptyset$ to obtain
    \begin{align*}
        \E{\sumTwea \sumK (\qti - u(i))\ltilt(i)  \;\middle|\;  \kappa}
        & \leq \E{\frac{\ln K}{\eta} + \sumTwea \eta \sumK q_t(i)\ltilt(i)^2  \;\middle|\;  \kappa} \\
        & \leq \E{\frac{\ln K}{\eta} + \sumTwea \eta \sumK \frac{q_t(i)}{\hat{\Pp}_t(i)}  \;\middle|\;  \kappa} \;,
    \end{align*}
    where we used the fact that $\hat{\Pp}_t(i) - \Pp_t(i) \geq 0$, given $\kappa$.
   Recalling \Cref{eq:LBpnoselfweakly} and using the fact that $\pi_t(i) \geq \half \qti$, we can see that 
    \begin{align*}
        \E{\sum_{i \in S}\frac{q_t(i)}{\hat{\Pp}_t(i)}  \;\middle|\;  \kappa} \leq \E{\sum_{i \in S}\frac{2\pi_t(i)}{\hat{\Pp}_t(i)}  \;\middle|\;  \kappa} 
        & \leq \E{\frac{4\wdelta(\hat{\scG}_{\e^{\star}_{\delta, \sigma}})}{\gamma}  \;\middle|\;  \kappa} \;.
    \end{align*}
    %
    Considering the sum over $i \not \in S$, we have 
    \begin{align*}
        \E{\sum_{i \not \in S}\frac{q_t(i)}{\hat{\Pp}_t(i)}  \;\middle|\;  \kappa}
        & \leq \E{\sum_{i \not \in S}\frac{2\pi_t(i)}{\hat{\Pp}_t(i)}  \;\middle|\;  \kappa} \\
        & \leq \E{\sum_{i \not \in S}\frac{2}{\hat{p}_t(i, i)}  \;\middle|\;  \kappa}
        \leq \E{\sum_{i \not \in S}\frac{4}{\tilde{p}(i, i)}  \;\middle|\;  \kappa}
        \leq 4 \sigma(\hat{\scG}_{\e^{\star}_{\delta, \sigma}}) \enspace.
    \end{align*}
    Thus, we have that 
    \begin{align}\label{eq:weaklysecondorderbound}
       \E{\sumK \frac{q_t(i)}{\hat{\Pp}_t(i)} \,\middle|\, \kappa}
        & \leq 4 \E{\frac{\wdelta(\hat{\scG}_{\e^{\star}_{\delta, \sigma}})}{\gamma} + \sigma(\hat{\scG}_{\e^{\star}_{\delta, \sigma}})  \;\middle|\;  \kappa} \;,
    \end{align}
    which means that we can use $\eta = \sqrt{\frac{\ln(K)}{4T} \bigl(\wdelta(\hat{\scG}_{\e^{\star}_{\delta, \sigma}})/\gamma + \sigma(\hat{\scG}_{\e^{\star}_{\delta, \sigma}})\bigr)^{-1}}$ to obtain
    \begin{align*}
        \E{\sumTwea \sumK (\pi_t(i) - u(i))\ltilt(i)  \;\middle|\;  \kappa}
        & \le \E{\sumTwea \sumK (\qti - u(i))\ltilt(i)  \;\middle|\;  \kappa}  + \gamma T \\
        & \leq \frac{\ln K}{\eta} + 4\eta T \left(\frac{\wdelta(\hat{\scG}_{\e^{\star}_{\delta, \sigma}})}{\gamma} + \sigma(\hat{\scG}_{\e^{\star}_{\delta, \sigma}})\right) + \gamma T \\
        & = 4\sqrt{T\ln(K) \left( \frac{\wdelta(\hat{\scG}_{\e^{\star}_{\delta, \sigma}})}{\gamma} +  \sigma(\hat{\scG}_{\e^{\star}_{\delta, \sigma}})\right)} + \gamma T \;.
    \end{align*}
    Now, observe that $T \leq 8\wdelta(\hat{\scG}_{\e^{\star}_{\delta, \sigma}}) \ln(3K^2T^2)$ whenever the algorithm's parameter $\gamma = \min\bigl\{\bigl(\wdelta(\hat{\scG}_{\e^{\star}_{\delta, \sigma}})\ln(3K^2T^2)\bigr)^{1/3}T^{-1/3}, \half\bigr\} = \half$.
    As a consequence,
    \begin{align*}
        & \E{\sumTwea \sumK (\pi_t(i) - u(i))\ltilt(i)  \;\middle|\;  \kappa} \\
        &\enspace \leq 4 \sqrt{T\ln(K) \left( \frac{\wdelta(\hat{\scG}_{\e^{\star}_{\delta, \sigma}})}{\gamma} + \sigma(\hat{\scG}_{\e^{\star}_{\delta, \sigma}})\right)} + \gamma T \\
        &\enspace \leq 16\wdelta(\hat{\scG}_{\e^{\star}_{\delta, \sigma}}) \ln(3K^2T^2) + 5 (\wdelta(\hat{\scG}_{\e^{\star}_{\delta, \sigma}}) \ln(3K^2T^2))^{1/3} T^{2/3} + 4\sqrt{\sigma(\hat{\scG}_{\e^{\star}_{\delta, \sigma}}) T\ln(K)} \;,
    \end{align*}
    which completes the proof.
\end{proof}

For the following lemma, we will use a simplifying assumption on $T$: we will assume that $T$ is such that 
\begin{align}\label{eq:simplifyTweakly}
    & 2 + \bigl(37\wdelta(\hat{\scG}_{\e^{\star}_{\delta, \sigma}}) + 12\sigma(\hat{\scG}_{\e^{\star}_{\delta, \sigma}})\bigr)\ln(3K^2T^2)^2 + 12\wdelta(\hat{\scG}_{\e^{\star}_{\delta, \sigma}})^{2/3}\left(\ln(3K^2T^2)\right)^{5/3}T^{1/3} \nonumber\\
    & \leq 28\left(\wdelta(\hat{\scG}_{\e^{\star}_{\delta, \sigma}})\ln(3K^2T^2)\right)^{1/3}T^{2/3} + 29\sqrt{\ln(3K^2T^2)\sigma(\hat{\scG}_{\e^{\star}_{\delta, \sigma}}) T} \enspace.
\end{align}

\begin{lemma}\label{lem:weaklysemifinal}
    Suppose that \eqref{eq:simplifyTweakly} holds and that $\hat{\scG}$ is an $\e_{\tstar}$-good approximation of $\scG$. For any distribution $u$ over $[K]$, \Cref{alg:beyondsanityalg} guarantees
    \begin{align*}
        & \E{\sumTweak \sumK (\pi_t(i) - u(i))\ell_t(i)} \\
        & \quad \leq 41\left(\ln(3K^2T^2)\wdelta(\hat{\scG}_{\e^{\star}_{\delta, \sigma}})\right)^{1/3}T^{2/3} + 41\sqrt{\ln(3K^2T^2)\sigma(\hat{\scG}_{\e^{\star}_{\delta, \sigma}})T} \enspace.
    \end{align*}
    We also have that
    \begin{align*}
        & \E{\sumTweak \sumK (\pi_t(i) - u(i))\ell_t(i)} \leq \\
        & \min_{\e \ge 2\e_{\tstar}}\left\{  82\left(\ln(3K^2T^2)\wdelta(\scG_\e)\right)^{1/3}T^{2/3} + 82\sqrt{\ln(3K^2T^2)\sigma(\scG_\e)T} \,:\, \supp{\scG_\e} \textnormal{ observable}\right\}.
    \end{align*}
    \end{lemma}
\begin{proof}
    Following the proof of \Cref{lem:IWproperty}, we can see that
    \begin{align*}
        & \E{\sumTweak \sumK (\pi_t(i) - u(i))\ell_t(i)} \leq 2 + \sumTwea\E{\frac{6\ln(3K^2T^2)}{t-1}\sumK  \frac{\pi_t(i)\bar{\pi}_t(i)}{\hat{\Pp}_t(i)} \;\middle|\; \kappa}\\
        & \quad  + \E{\sumTwea \!2\sqrt{2\frac{\ln(3K^2T^2)}{t-1}} \sqrt{\sumK  \frac{\pi_t(i)\bar{\pi}_t(i)}{\hat{\Pp}_t(i)}} \,\middle|\, \kappa} + \E{\sumTwea \sumK (\pi_t(i) - u(i))\ltilt(i) \,\middle|\, \kappa}.
    \end{align*}
    Now, using the same reasoning that led to \Cref{eq:weaklysecondorderbound}, we have that
    \begin{align*}
        & \E{\sumTweak \sumK (\pi_t(i) - u(i))\ell_t(i)} \leq 2 + \sumTwea\E{\frac{12\ln(3K^2T^2)}{t-1}\biggl(\frac{\wdelta(\hat{\scG}_{\e^{\star}_{\delta, \sigma}})}{\gamma} + \sigma(\hat{\scG}_{\e^{\star}_{\delta, \sigma}})\biggr) \,\middle|\, \kappa}\\
        & \quad  + \E{\sumTwea 4\sqrt{\frac{\ln(3K^2T^2)}{t-1}} \sqrt{\frac{\wdelta(\hat{\scG}_{\e^{\star}_{\delta, \sigma}})}{\gamma} + \sigma(\hat{\scG}_{\e^{\star}_{\delta, \sigma}})} \;\middle|\; \kappa} \\
        & \quad + \E{\sumTwea \sumK (\pi_t(i) - u(i))\ltilt(i) \;\middle|\; \kappa} \\
        & \leq 2 + \E{{12\ln(3K^2T^2)^2}\left(\frac{\wdelta(\hat{\scG}_{\e^{\star}_{\delta, \sigma}})}{\gamma} + \sigma(\hat{\scG}_{\e^{\star}_{\delta, \sigma}})\right) \;\middle|\; \kappa}\\
        & \quad  + \E{ 8\sqrt{T\ln(3K^2T^2)\left(\frac{\wdelta(\hat{\scG}_{\e^{\star}_{\delta, \sigma}})}{\gamma} + \sigma(\hat{\scG}_{\e^{\star}_{\delta, \sigma}})\right)} \;\middle|\; \kappa} \\
        & \quad + \E{\sumTwea \sumK (\pi_t(i) - u(i))\ltilt(i) \;\middle|\; \kappa} \;,
    \end{align*}
    where we used that $\sumT \frac{1}{\sqrt{t}} \leq 2\sqrt{T}$ and $\sumTt \frac{1}{t-1} \leq 1 + \ln(T) \le \ln(3K^2T^2)$ for $K,T \ge 2$.
    Following the final steps in the proof of \Cref{lem:weaklyuntuned}, we can show that 
    \begin{align*}
        & \E{8\sqrt{T\ln(3K^2T^2)\left(\frac{\wdelta(\hat{\scG}_{\e^{\star}_{\delta, \sigma}})}{\gamma} + \sigma(\hat{\scG}_{\e^{\star}_{\delta, \sigma}})\right)} \;\middle|\; \kappa}  \\
        & \quad \leq  32\wdelta(\hat{\scG}_{\e^{\star}_{\delta, \sigma}}) \ln(3K^2T^2) + 8T^{2/3}\left(\wdelta(\hat{\scG}_{\e^{\star}_{\delta, \sigma}})\ln(3K^2T^2)\right)^{1/3} + 8\sqrt{T\sigma(\hat{\scG}_{\e^{\star}_{\delta, \sigma}})\ln(3K^2T^2)} \enspace.
    \end{align*}
    Hence, by applying \Cref{lem:weaklyuntuned}, we obtain that
    \begin{align*}
        & \E{\sumTwea \sumK (\pi_t(i) - u(i))\ltilt(i) \;\middle|\; \kappa} \\
        &\quad \leq 16 \wdelta(\hat{\scG}_{\e^{\star}_{\delta, \sigma}}) \ln(3K^2T^2) + 5T^{2/3}(\wdelta(\hat{\scG}_{\e^{\star}_{\delta, \sigma}})\ln(3K^2T^2))^{1/3} + 4\sqrt{T\sigma(\hat{\scG}_{\e^{\star}_{\delta, \sigma}})\ln(K)} \enspace.
    \end{align*}
    Finally, by definition of $\gamma$ we notice that
    \begin{align*}
        {12\ln(3K^2T^2)^2}\frac{\wdelta(\hat{\scG}_{\e^{\star}_{\delta, \sigma}})}{\gamma} \leq 24\wdelta(\hat{\scG}_{\e^{\star}_{\delta, \sigma}})\ln(3K^2T^2)^2 + 12T^{1/3}\wdelta(\hat{\scG}_{\e^{\star}_{\delta, \sigma}})^{2/3}\left(\ln(3K^2T^2)\right)^{5/3} \;.
    \end{align*}
    Thus, combining the above we obtain
    \begin{align*}
        &\E{\sumTweak \sumK (\pi_t(i) - u(i))\ell_t(i)} \\
        & \leq 2 + 37\wdelta(\hat{\scG}_{\e^{\star}_{\delta, \sigma}})\ln(3K^2T^2)^2 + 12T^{1/3}\wdelta(\hat{\scG}_{\e^{\star}_{\delta, \sigma}})^{2/3}\left(\ln(3K^2T^2)\right)^{5/3}\\
        &\quad + 13T^{2/3}\left(\wdelta(\hat{\scG}_{\e^{\star}_{\delta, \sigma}})\ln(3K^2T^2)\right)^{1/3} + 12\sqrt{T\sigma(\hat{\scG}_{\e^{\star}_{\delta, \sigma}})\ln(3K^2T^2)} \\
        &\quad + 12\sigma(\hat{\scG}_{\e^{\star}_{\delta, \sigma}})\ln(3K^2T^2)^2 \enspace.
    \end{align*}
    Since we assumed that \eqref{eq:simplifyTweakly} holds, we can show that 
    \begin{align*}
        & \E{\sumTweak \sumK (\pi_t(i) - u(i))\ell_t(i)} \\
        &\quad \leq  41T^{2/3}\left(\wdelta(\hat{\scG}_{\e^{\star}_{\delta, \sigma}})\ln(3K^2T^2)\right)^{1/3} + 41\sqrt{T\sigma(\hat{\scG}_{\e^{\star}_{\delta, \sigma}})\ln(3K^2T^2)} \enspace,
    \end{align*}
    which is the first result in the statement. For the second result,
    recall that $\e^{\star}_{\delta, \sigma}$ is the minimizer of the above bound by its definition in~\eqref{eq:optimizede}. Since $\hat{\scG}$ is an $\e_{\tstar}$-good approximation of $\scG$, we conclude that 
    \begin{align*}
        & \E{\sumTweak \sumK (\pi_t(i) - u(i))\ell_t(i)} \leq \\
        & \min_{\e \ge 2\e_{\tstar}}\left\{  82T^{2/3}\left(\ln(3K^2T^2)\wdelta(\scG_\e)\right)^{1/3} + 82\sqrt{\ln(3K^2T^2)\sigma(\scG_\e)T} \,:\, \supp{\scG_\e} \textnormal{ observable}\right\} \;.
    \end{align*}
\end{proof}

\subsection{Regret After \texorpdfstring{$T$}{T} Rounds}

We now have all the intermediate results we need to prove the overall regret bound of \Cref{alg:beyondsanityalg}.

\begin{theorem}\label{th:finalbeyondsanity}
    Suppose that \eqref{eq:simplifyTweakly} holds. Then, for any distribution $u$ over $[K]$, \Cref{alg:beyondsanityalg} satisfies 
    \begin{align*}
        & \E{\sumT \sumK (\pi_t(i) - u(i))\ell_t(i)} \leq \min \Biggl\{ T \;,  \\
        & \quad 6 + 2\min_{\e \,:\, \supp{\scG_\e} \textnormal{ strongly observable}} \Biggl\{ {198\walpha(\scG_\e)} (\ln(2K^3T^2))^3 \\
        & \quad\quad  + \Bigl(12\ln(K) + 4\sqrt{2\ln(3K^2T^2)}\Bigr){\sqrt{18\tstar \walpha(\scG_\e) \ln(2 K^3 T^2)}}\Biggr\} \enspace, \\
        & \quad 4 + 164\ln(3K^2T^2) \min_{\e \,:\, \supp{\scG_\e} \textnormal{ observable}} \Bigl((\wdelta(\scG_\e))^{1/3}T^{2/3} + \sqrt{\sigma(\scG_\e)T}\Bigr) \Biggr\} \;.
    \end{align*}
\end{theorem}
\begin{proof}
    Let us recall that in \Cref{eq:Psit,eq:Lambdat} we define
    \begin{align*}
        \Psi_\tstar = \min\Biggl\{\tstar,\; & 2 + {11(\ln(3K^2T^2))^2 \max_{t \in [2, \tstar]} \theta_t(\hat{\scG}_t)}  \\
        &  + \Bigl(12\ln(K) + 4\sqrt{2\ln(3K^2T^2)}\Bigr){\sqrt{\tstar\max_{t \in [2, \tstar]} \theta_t(\hat{\scG}_t)}}\Biggr\}
    \end{align*}
    and
    \begin{align*}
        \Lambda_\tstar = 41\left(\ln(3K^2T^2)\wdelta(\hat{\scG}_{\e^{\star}_{\delta, \sigma}})\right)^{1/3}T^{2/3} + 41\sqrt{\ln(3K^2T^2)\sigma(\hat{\scG}_{\e^{\star}_{\delta, \sigma}})T} \enspace.
    \end{align*}
    Denote by $\calE$ the event that $\tilde{\scG}_{\e_t} = \{\tilde{p}_t(j, i)\ind{\{\tilde{p}_t(j, i)\geq {\logfactor}/{t}\}} : i,j \in V\}$ is a $\e_t$-good approximation of $\scG$ with $\e_t = \logfactor/t$ for all $t \leq T$. By \Cref{lem:eps-t-good-approx}, we have that $\calE$ occurs with probability at least $1 - \frac{1}{T}$ and thus,
    for any $\tstar \in [1, T]$, we have that
    \begin{align*}
        & \E{\sumT \sumK (\pi_t(i) - u(i))\ell_t(i)} \\
        &\quad \leq 1 + \E{\sumT \sumK (\pi_t(i) - u(i))\ell_t(i)~ \middle | ~ \calE} \\
        &\quad = 1 + \E{\sum_{t = 1}^{\tstar} \sumK (\pi_t(i) - u(i))\ell_t(i)~ \middle | ~ \calE} + \E{\sum_{t = \tstar + 1}^T \sumK (\pi_t(i) - u(i))\ell_t(i) ~ \middle | ~ \calE} \\
        &\quad \leq 1 + \E{\Psi_\tstar + \Lambda_\tstar ~ \middle | ~ \kappa, \calE} \enspace,
    \end{align*}
    where the last inequality is due to \Cref{lem:stronglytuned,lem:weaklysemifinal}.
    We now consider two cases depending on whether \Cref{alg:beyondsanityalg} commits to the weakly observable regret regime at any time step or it never does so.
    In the first case, say \Cref{eq:stopwhen} never holds for any $t \in [2, T]$.
    We consequently have that 
    \begin{align*}
        & \E{\sumT \sumK (\pi_t(i) - u(i))\ell_t(i)} \leq 1 + \E{\min\left\{\Psi_\tstar, \Lambda_\tstar\right\} ~ \middle | ~ \kappa, \calE} \;.
    \end{align*}
    We first try to upper bound the conditional expectation of $\Lambda_\tstar$. By definition of $\e$-good approximation of $\scG$, we have
    \begin{align*}
        \E{\Lambda_\tstar ~ \middle | ~ \kappa, \calE}
        & = \Enb\Biggl[ \min_{\e \in [0, 1]} \Biggl\{41\left(\ln(3K^2T^2)\wdelta((\hat{\scG}_{\tstar})_{\e})\right)^{1/3}T^{2/3} \\
        &\quad  + 41\sqrt{\ln(3K^2T^2)\sigma((\hat{\scG}_{\tstar})_{\e})T} \;:\, \supp{(\hat{\scG}_{\tstar})_{\e}} \textnormal{ observable} \Biggr\} ~ \Bigg| ~ \kappa, \calE\Biggr] \\
        & \leq 2\Enb\Biggl[\min_{\e \in [2\e_\tstar, 1]} \Biggl\{ 41\left(\ln(3K^2T^2)\wdelta((\hat{\scG}_{\tstar})_{\e})\right)^{1/3}T^{2/3} \nonumber \\
        & \quad + 41\sqrt{\ln(3K^2T^2)\sigma((\hat{\scG}_{\tstar})_{\e})T}: \supp{{\scG}_{\e}} \textnormal{ observable}\Biggr\} \;\Bigg|\, \kappa, \calE\Biggr] \;.
    \end{align*}
    To cover the remaining thresholds in $[0, 2\e_\tstar)$, we define $\beste = \max\mathcal{Q}$ as the largest threshold~$\e$ that minimizes 
    \begin{align*}
        \mathcal{Q} = \argmin_{\e \in [0, 1]} \Biggl\{& 41\left(\ln(3K^2T^2)\wdelta((\hat{\scG}_{\tstar})_{\e})\right)^{1/3}T^{2/3} \nonumber \\
        & + 41\sqrt{\ln(3K^2T^2)\sigma((\hat{\scG}_{\tstar})_{\e})T} \;:\, \supp{{\scG}_{\e}} \textnormal{ observable}\Biggr\} \;.
    \end{align*}
    If $\beste < 2\e_\tstar$, meaning that $\beste$ as well as the other thresholds in $\mathcal{Q}$ do not belong to the already covered interval $[2\e_\tstar,1]$, then $\tstar < \frac{\twicelogfactor}{\beste} = \twicelogfactor t_{\beste}$ with $t_{\beste} = 1/\beste$.
    Thus, we must have that  
    \begin{align*}
        \tstar & \leq \twicelogfactor\left(\bigl(\wdelta(\scG_\beste)\bigr)^{1/3}t_{\beste}^{2/3} + \sqrt{\sigma(\scG_\beste)t_{\beste}}\right) \\
        & \leq \min_{\e \in [0, 1]}\Biggl\{ \twicelogfactor\left((\wdelta(\scG_\e))^{1/3}T^{2/3} + \sqrt{\sigma(\scG_\e)T}\right) \;:\, \supp{\scG_\e} \textnormal{ observable}\Biggr\} \;,
    \end{align*}
    where the first inequality is due to the fact that $\wdelta(\scG_\beste) \geq t_{\beste}$ or $\sigma(\scG_\beste) \geq t_\beste$ or both are true because either $p(i,i) = \beste$ for some $i$ such that $i\in \inneigh{\supp{\scG_{\beste}}}{i}$ or one of the minimum outgoing edge probabilities for a vertex in some minimum weight weakly dominating set is equal to $\beste$.
    
    On the other hand, we also need to upper bound the conditional expectation of $\Psi_\tstar$.
    By \Cref{lem:strongwalpha} and recalling the definition of $\walpha$ from \Cref{sec:topo}, we have that 
    $$
    \E{{\max_{t \in [2, \tstar]} \theta_t(\hat{\scG}_t)}  \;\middle|\;  \kappa} \leq \E{\min_{\e \,:\, \supp{{\scG}_{\e}} \textnormal{ strongly observable}} 18\walpha(\scG_\e) \ln(2 K^3 T^2)  \;\middle|\;  \kappa} \;.
    $$
    and thus 
    \begin{align*}
        & 2 + \E{{11(\ln(3K^2T^2))^2 \max_{t \in [2, \tstar]} \theta_t(\hat{\scG}_t)}  \;\middle|\;  \kappa, \calE} \\
        & \quad  + \E{\Bigl(12\ln(K) + 4\sqrt{2\ln(3K^2T^2)}\Bigr){\sqrt{\tstar\max_{t \in [2, \tstar]} \theta_t(\hat{\scG}_t)}}  \;\middle|\;  \kappa, \calE} \\
        & \leq 2 + \min_{\e \,:\, \supp{\scG_\e} \textnormal{ strongly observable}} \Biggl\{ {198\walpha(\scG_\e)} (\ln(2K^3T^2))^3 \\
        & \quad  + \Bigl(12\ln(K) + 4\sqrt{2\ln(3K^2T^2)}\Bigr){\sqrt{18\tstar \walpha(\scG_\e) \ln(2 K^3 T^2)}} \Biggr\} \;.
    \end{align*}
    Since $\twicelogfactor \leq 82 \ln(3K^2T^2)$, we can combine the above to obtain 
    \begin{align*}
        & \E{\sumT \sumK (\pi_t(i) - u(i))\ell_t(i)} \leq \min \Biggl\{ T \,,  \\
        &\quad  3 + \min_{\e} \Biggl\{ {198\walpha(\scG_\e)} (\ln(2K^3T^2))^3 + \\
        &\quad\;   \Bigl(12\ln(K) + 4\sqrt{2\ln(3K^2T^2)}\Bigr){\sqrt{18\tstar \walpha(\scG_\e) \ln(2 K^3 T^2)}} \;:\, \supp{\scG_\e} \textnormal{ strongly observable}\Biggr\} \,, \\
        &\quad  1 + \min_{\e} \Biggl\{ 82 \ln(3K^2T^2) \left(\left(\wdelta(\scG_\e)\right)^{1/3}T^{2/3} + \sqrt{\sigma(\scG_\e)T}\right) \;:\, \supp{\scG_\e} \textnormal{ observable}\Biggr\} \Biggr\} \;.
    \end{align*}
    In the second case, $\tstar$ is the first round in which~\eqref{eq:stopwhen} holds.
    Therefore, we must have 
    \begin{align*}
        \E{\sumT \sumK (\pi_t(i) - u(i))\ell_t(i)} \leq 1 + 2 \E{\Psi_\tstar \,\middle|\, \kappa, \calE}
    \end{align*}
    and 
    \begin{align*}
        & \E{\sumT \sumK (\pi_t(i) - u(i))\ell_t(i)} \\
        &\quad \leq 1 + \E{\sum_{t = 1}^{\tstar-1} \sumK (\pi_t(i) - u(i))\ell_t(i) ~ \middle| ~ \calE } + 1 + \E{\sum_{t = \tstar + 1}^T \sumK (\pi_t(i) - u(i))\ell_t(i) ~ \middle | \calE} \\
        &\quad \leq \E{\Psi_{\tstar - 1} + \Lambda_{\tstar} \,\middle|\, \kappa, \calE}  + 2 \\
        &\quad \leq \E{\Lambda_{\tstar - 1} + \Lambda_{\tstar} \,\middle|\, \kappa, \calE}  + 2 \enspace,
    \end{align*}
    which combined give us 
    \begin{align*}
        \E{\sumT \sumK (\pi_t(i) - u(i))\ell_t(i)} \leq 1
         + \E{\min\left\{\Lambda_{\tstar - 1} + \Lambda_{\tstar} + 1, 2 \Psi_\tstar \right\}  \;\middle|\;  \kappa, \calE} \;.
    \end{align*}
    Following the proof of the bound in the case where \eqref{eq:stopwhen} never holds for any $t \in [2, T]$, we can see that 
    \begin{align*}
        & \E{\sumT \sumK (\pi_t(i) - u(i))\ell_t(i)} \leq \min \Biggl\{ T \;,  \\
        & \quad 6 + 2\min_{\e \,:\, \supp{\scG_\e} \textnormal{ strongly observable}} \Bigg\{ {198\walpha(\scG_\e)} (\ln(2K^3T^2))^3 \\
        & \quad\quad  + \Bigl(12\ln(K) + 4\sqrt{2\ln(3K^2T^2)}\Bigr){\sqrt{18\tstar \walpha(\scG_\e) \ln(2 K^3 T^2)}}\Bigg\} \enspace, \\
        & \quad 4 + 164\ln(3K^2T^2) \min_{\e \,:\, \supp{\scG_\e} \textnormal{ observable}} \left(\left(\wdelta(\scG_\e)\right)^{1/3}T^{2/3} + \sqrt{\sigma(\scG_\e)T}\right) \Biggr\} \;,
    \end{align*}
    which completes the proof. 
\end{proof}

\subsection{Auxiliary Lemmas for \textnormal{\otc}}

In this section, we prove some results that are useful in the above regret analysis of \otc (\Cref{alg:beyondsanityalg}).
Recall that $\scS$ is the family of strongly observable graphs over vertices $V = [K]$.

\begin{lemma}\label{lem:strongwalpha}
    Suppose that there exists a threshold $\e$ such that $\supp{\scG_\e} \in \scS$. Then, we have that 
    \begin{align*}
        \E{\max_{t \in [2, \tstar]} \min_{\e \,:\, \supp{(\hat{\scG}_{t}^{\textnormal{UCB}})_{\e}} \in \scS} \theta_t((\hat{\scG}_{t}^{\textnormal{UCB}})_{\e})  \;\middle|\;  \kappa} & \leq \E{\min_{\e \,:\, \supp{{\scG}_{\e}} \in \scS} 18\walpha(\scG_\e) \ln(2 K^3 T^2)  \;\middle|\;  \kappa} 
    \end{align*}
\end{lemma}
\begin{proof}
Let us recall the definition of $\theta_t$:
\begin{align*}
    \theta_t((\hat{\scG}^{\text{UCB}}_{t})_{\e}) = \frac{2}{\min_{i} \min_{j \in \inneigh{\supp{(\hat{\scG}^{\text{UCB}}_{t})_{\e}}}{i}} \pji}  + \sum_{i \in \inneigh{\supp{(\hat{\scG}^{\text{UCB}}_{t})_{\e}}}{i}} \frac{2\pi_t(i)}{P_t(i, (\hat{\scG}^{\text{UCB}}_{t})_{\e})} \enspace.
\end{align*}
By definition of the weighted independence number (see \Cref{sec:weighted-alpha} for further details), we have that 
\begin{align*}
    \frac{2}{\min_{i} \min_{j \in \inneigh{\supp{(\hat{\scG}^{\text{UCB}}_{t})_{\e}}}{i}} \pji} \leq 2\walpha((\hat{\scG}^{\text{UCB}}_{t})_{\e}) \enspace.
\end{align*}
By \Cref{lem:weightedindependence}, we have that 
\begin{align*}
    2\sum_{i \in \inneigh{\supp{(\hat{\scG}^{\text{UCB}}_{t})_{\e}}}{i}} \frac{\pi_t(i)}{P_t(i, (\hat{\scG}^{\text{UCB}}_{t})_{\e})} \leq 16\walpha((\hat{\scG}^{\text{UCB}}_{t})_{\e}) \ln(2K^3T^2) \enspace,
\end{align*}
where we used that $\gamma_t \psi_t(i) \geq \frac{1}{KT}$ and $\phattji \geq \frac{1}{T}$.

Given $\kappa$, we have that $\phattji \geq \pji$ and thus it holds that
\begin{align*}
    & \Enb\biggl[\max_{t \in [2, \tstar]}\min_{\e \,:\, \supp{(\hat{\scG}_{t}^{\textnormal{UCB}})_{\e}} \in \scS} \theta_t((\hat{\scG}_{t}^{\textnormal{UCB}})_{\e}) \,\bigg|\, \kappa\biggr]\\
    &\quad \leq \Enb\bigg[\max_{t \in [2, \tstar]} \min_{\e \,:\, \supp{(\hat{\scG}_{t}^{\textnormal{UCB}})_{\e}} \in \scS} 18\walpha((\hat{\scG}_{t}^{\textnormal{UCB}})_{\e}) \ln(2 K^3 T^2) \,\bigg|\, \kappa\biggr] \\
    &\quad \leq \Enb\biggl[\min_{\e \,:\, \supp{\scG_\e} \in \scS} 18\walpha(\scG_\e) \ln(2 K^3 T^2) \;\bigg|\; \kappa\biggr] \enspace.
\end{align*}
%
\end{proof}

The following result is a variant of the bound in \citet[Lemma 4]{AlonCDK15} with a decreasing learning rate. 
\begin{lemma}\label{lem:EW}
    Let $q_1, \ldots, q_T$ be the probability vectors defined by $q_t(i) \propto \exp(-\eta_{t-1}\sum_{s = 1}^{t-1} \ell_s(i))$ for a sequence of loss functions $\ell_1, \ldots, \ell_T$ such that $\ell_t(i) \geq 0$ for all $t$ and $i$. Let $\eta_0 = \eta_1 \geq \ldots \geq \eta_T$. For each $t$, let $S_t$ be a subset of $[K]$ such that $\eta_{t-1} \ell_t(i) \leq 1$ for all $i \in S_t$. Then, for any distribution $u$ it holds that 
    \begin{align*}
        \sumT \sumK (q_t(i) - u(i))\ell_t(i) \leq \frac{\ln(K)}{\eta_T} + \sumT \eta_{t-1}\left(\sum_{i \in S_t}q_t(i)(1 - q_t(i))\ell_t(i)^2 + \sum_{i \not \in S_t} q_t(i)\ell_t(i)^2\right).
    \end{align*}
\end{lemma}
\begin{proof}
    The proof follows from a minor adaptation of the proof of \citet[Lemma 4]{AlonCDK15}. We start from \citet[Lemma 1]{vanderhoeven2018many}:
    \begin{equation}\label{eq:startproofEW}
    \begin{split}
        & \sumT \sumK (\qti - \ui)\ell_t(i) \\
        &\quad \leq \frac{\ln(K)}{\eta_T} + \sumT\left(\sumK \qti \ell_t(i)  + \frac{1}{\eta_{t-1}}\ln\left(\sumK \qti \exp(-\eta_{t-1} \ell_t(i)) \right)\right) \;. 
    \end{split}
    \end{equation}
    Now, since $\ell_t(i) \geq 0$ we may use $\exp(x) \leq 1 + x + x^2$ for $x \leq 1$ and $\ln(1 - x) \leq -x$ for all $x < 1$ to show that
    \begin{align*}
        \frac{1}{\eta_{t-1}}\ln\left(\sumK \qti \exp(-\eta_{t-1} \ell_t(i)) \right) & \leq \frac{1}{\eta_{t-1}}\ln\left(1 -\sumK \qti (\eta_{t-1} \ell_t(i) - \eta_{t-1}^2 \ell_t(i)^2) \right) \\
        & \leq -\sumK \qti (\ell_t(i) - \eta_{t-1} \ell_t(i)^2) \enspace.
    \end{align*}
    Combined with equation \eqref{eq:startproofEW}, this gives us
    \begin{align*}
        \sumT \sumK (\qti - \ui)\ell_t(i)
        & \leq \frac{\ln(K)}{\eta_T} + \sumT\sumK \eta_{t-1} \qti  \ell_t(i)^2 \enspace.
    \end{align*}
    We define $\lbart = \sum_{i \in S_t} \qti \ell_t(i)$. Since $\ell_t(i) \geq 0$ we have that $\eta_{t-1}(\ell_t(i) - \lbart) \geq -1$ by construction. Since adding the same $\lbart$ to each $\ell_t(i)$ on the r.h.s. of equation \eqref{eq:startproofEW} does not influence the regret we have 
    \begin{align*}
        \sumT \sumK (\qti - \ui)\ell_t(i)
        & \leq \frac{\ln(K)}{\eta_T} + \sumT\sumK \eta_{t-1} \qti  (\ell_t(i) - \lbart)^2 \enspace. 
    \end{align*}
    To complete the proof we follow the proof of \citet[Lemma 4]{AlonCDK15}, which gives us
    \begin{align*}
        \sumT \sumK (q_t(i) - u(i))\ell_t(i) \leq \frac{\ln(K)}{\eta_T} + \sumT \eta_{t-1}\left(\sum_{i \in S_t}q_t(i)(1 - q_t(i))\ell_t(i)^2 + \sum_{i \not \in S_t} q_t(i)\ell_t(i)^2\right).
    \end{align*}
\end{proof}

\begin{lemma}\label{lem:additivebias}
    Let $\xi_{t}(i) = \sum_{j \in N_{\hat{G}_t}^\textup{in}(i)} \pi_t(i)(\phattji - \pji)$. In any round $t$, we have that
    \begin{align*}
        & \sumK (\pi_t(i) - u(i))\lhatt(i) \\
        &\quad = \sumK (\pi_t(i) - u(i))\ltil_t(i) + \sumK (\pi_t(i) - u(i))\xi_t(i) \frac{\idseeit \ell_t(i)}{\Pp_t(i)\hat{\Pp}_t(i)} \enspace.
    \end{align*}
\end{lemma}
\begin{proof}
    Let $\idseeitx = \idseeit$ and denote by
    \begin{align*}
        \xi_{t}(i) = \hat{\Pp}_t(i) - \Pp_t(i) = \sum_{j \in N_{\hat{G}_t}^\text{in}(i)} \pi_t(i)(\phattji - \pji) \enspace.
    \end{align*}
    We have that 
    \begin{align*}
        \ltilt(i) = \frac{\idseeitx \ell_t(i)}{\hat{\Pp}_t(i)}
        & = \frac{\idseeitx \ell_t(i)}{\Pp_t(i) + \xi_t(i)} \\
        &= \frac{\idseeitx \ell_t(i)(\Pp_t(i) + \xi_t(i))}{\Pp_t(i)(\Pp_t(i) + \xi_t(i))} - \xi_t(i) \frac{\idseeitx \ell_t(i)}{\Pp_t(i)(\Pp_t(i) + \xi_t(i))} \\
        & = \frac{\idseeitx \ell_t(i)}{\Pp_t(i)} - \xi_t(i) \frac{\idseeitx \ell_t(i)}{\Pp_t(i)(\Pp_t(i) + \xi_t(i))} \\
        & = \lhat_t(i) - \xi_t(i) \frac{\idseeitx \ell_t(i)}{\Pp_t(i)\hat{\Pp}_t(i)} \enspace.
    \end{align*}
    Therefore, for any distribution $u$ we have that
    \begin{align*}
        \sumK (\pi_t(i) - u(i))\lhatt(i) = \sumK (\pi_t(i) - u(i))\ltil_t(i) + \sumK (\pi_t(i) - u(i))\xi_t(i) \frac{\idseeitx \ell_t(i)}{\Pp_t(i)\hat{\Pp}_t(i)} \enspace,
    \end{align*}
    which completes the proof.
\end{proof}

\begin{lemma} \label{lem:eps-t-good-approx}
    Let $\tilde{\scG}_{\e_t} = \{\tilde{p}_t(j, i)\ind{\{\tilde{p}_t(j, i)\geq \e_t\}} : i,j \in V\}$ and $\e_t = 60\ln(KT)/t$ for all $t \in [2, T]$.
    Then, with probability at least $1-1/T$, $\tilde \scG_{\e_t}$ is an $\e_t$-good approximation of $\scG$ for all $t \in [2, T]$.
\end{lemma}
\begin{proof}
    Let  $\epslEt = \{(i,j) \in V^2 \,:\, p(i,j) \ge 2\e_{t}\}$ and $\epssEt = \{(i,j) \in V^2 \,:\, p(i,j) < \e_{t}/2\}$ be the two sets of edges as defined in the proof of \Cref{thm:round-robin-estimates}.
    We let $\calE^t_{(i,j)} = \{\tilde p_t(i,j) \ge \e_t\}$ and $\F^t_{(i,j)} = \{\abs{\tilde p_t(i,j) - p(i,j)} \le p(i,j)/2\}$, for all $(i,j) \in V^2$ and all $t \in [2, T]$, be the events as similarly denoted in that same proof.
    We consequently define the events $\calE$, $\F$, and $\scC$ as
    \begin{align*}
        \calE = \bigcap_{t=1}^T \bigcap_{(i,j)\in \epslEt} \calE^t_{(i,j)} \enspace, \qquad
        \F = \bigcap_{t=1}^T \bigcap_{(i,j)\notin \epssEt} \F^t_{(i,j)} \enspace, \qquad
        \scC = \bigcap_{t=1}^T \bigcap_{(i,j)\in \epssEt} \overline{\calE}^t_{(i,j)} \enspace.
    \end{align*}
    The following steps hold for all $K \ge 2$ and all $T \ge 2$.

    We begin by observing that $\P{\tilde p_t(i,j) < \e_t} \le \exp(-t\e_t/4) \le 1/(4K^2T^2)$ for all $t \in [2, T]$ and all $(i,j) \in \epslEt$, by a simple adaptation of the same argument in the proof of \Cref{thm:round-robin-estimates}.
    Then,
    \[
        \P{\calE} \ge 1-\sumT\frac{\abs{\epslEt}}{4K^2T^2} \ge 1-\frac1{4T} \enspace,
    \]
    which follows from the fact that $\abs{\epslEt} \le K^2$ for all $t \in [2, T]$.
    We can similarly argue that $\P{\abs{\tilde p_t(i,j)-p(i,j)} > p(i,j)/2} \le 2\exp(-t\e_t/24) \le 1/(2K^2T^2)$ for all $t \in [2, T]$ and all $(i,j) \notin \epssEt$; this implies that $\P{\F} \ge 1-1/(2T)$.
    Finally, we observe that $\P{\tilde p_t(i,j) \ge \e_t} \le \exp(-t\e_t/6) \le 1/(4K^2T^2)$ for all $t \in [2, T]$ and all $(i,j) \in \epssEt$, hence $\P{\scC} \ge 1-1/(4T)$.
    The statement follows by union bound over the complements of $\calE$, $\F$, and $\scC$.
\end{proof}

\section{Weighted Independence Number} \label{sec:weighted-alpha}

To improve the regret bounds in the case of strongly observable support, we need to introduce another graph-theoretic quantity: the \emph{weighted independence number} $\walpha(G, w)$, where $w \in \R^K_+$ is a vector of positive weights assigned to the vertices of our strongly observable graph $G = (V, E)$ with $V = [K]$.
Let $w(U) = \sum_{i \in U} w_i$ denote the weight of a subset of vertices $U \subseteq V$.
The weighted independence number is defined as
\[
    \walpha(G, w) = \max_{S \in \I(G)} w(S) \enspace,
\]
that is, the weight of a maximum weight independent set.
This set is chosen among all sets in the family $\I(G)$ of independent sets of $G$.
It can be equivalently defined by the following integer linear program:
\[
    \begin{array}{rrclcl}
    \walpha(G, w) = \displaystyle\max_{x} & \multicolumn{3}{l}{\displaystyle\sum_{i=1}^K w_i x_i} \\
    \textrm{s.t.} & x_i + x_j & \le & 1 && \forall (i, j) \in E, i\neq j\\
    & x_i & \in & \{0, 1\} && \forall i \in V \\
    \end{array}
\]
We plan to define $w$ according to our needs in what follows.
    
\subsection{Undirected Graph}

Let $\scG$ be a stochastic feedback graph with edge probabilities $p(i,j)$ and such that its support $\supp{\scG} = G = (V, E)$ is undirected and strongly observable.
Moreover, let $N(i)$ be the neighborhood in $G$ of any vertex $i \in V$ (excluding $i$) and let $C(i) = N(i) \cup \{i\}$ be the extended neighborhood of $i$ including vertex $i$ itself.

We can use the edge probabilities from $\scG$ to define a weight for each vertex $i$ as
\[
w_{\scG}(i) = w_i = \Biggl(\frac1{\abs{C(i)}}\sum_{j \in C(i)} p(j,i)\Biggr)^{\!-1} .
\]
This vertex weight is equal to the inverse of the arithmetic mean of the incident edge probabilities (including its self-loop).
Note that the two probabilities $p(i,j)$ and $p(j,i)$ in the two directions of any undirected edge $(i,j) \in E$ need not be equal.

This definition allows us to upper bound the second-order term in the regret for vertices with self-loop (as similarly done in the analysis of \expg \citep{AlonCDK15}) in terms of the weighted independence number since we can reduce it to bounding
\[
    \sum_{i\in V} \frac1{\sum_{j \in C(i)} p(j,i)} = \sum_{i\in V} \frac{w_i}{\abs{C(i)}} \enspace.
\]
We thus require a weighted version of Turán's theorem, which is formulated in the lemma below.
This result has already been proved \citep{sakai2003independent}, but we nevertheless provide a proof for completeness.

\begin{lemma} \label{lem:turan-weights-undirected}
    Let $G = (V, E)$ be an undirected graph with positive vertex weights $w_i$. Then,
    \[
        \sum_{i \in V} \frac{w_i}{\abs{C(i)}} \le \walpha(G, w) \enspace.
    \]
\end{lemma}
\begin{proof}
    Consider the following algorithm: as long as the graph is not empty, repeatedly choose a vertex $j$ that minimizes $\abs{C(j)}/w_j$ among all remaining vertices and remove it from the graph along with its neighborhood.
    Let $i_1, \dots, i_s$ be the sequence of $s$ vertices picked by this algorithm, which form an independent set by construction.
    Additionally, let $G_1, \dots, G_{s+1}$ be the sequence of graphs generated by this iterative procedure, where $G_1 = G$ is the starting graph and $G_{s+1}$ is the empty graph.
    We also let $C_r(i)$ denote the extended neighborhood over $G_r$ of any $i \in V(G_r)$.
    Define
    \[
        Q(H) = \sum_{i \in V(H)} \frac{w_i}{\abs{C(i)}} \qquad \forall H \subseteq G\enspace,
    \]
    as the quantity we are trying to bound for $G$ and consider it over the graphs in the sequence generated by the procedure. It is strictly decreasing until reaching $Q(G_{s+1}) = 0$.
    In particular, at any step of the procedure it decreases by
    \[
        Q(G_r) - Q(G_{r+1}) = \sum_{j \in C_r(i_r)} \frac{w_j}{\abs{C(j)}} \le \sum_{j \in C_r(i_r)} \frac{w_{i_r}}{\abs{C(i_r)}} = \frac{\abs{C_r(i_r)}}{\abs{C(i_r)}} w_{i_r} \le w_{i_r} \enspace,
    \]
    where the first inequality is due to the optimality of $\abs{C(i_r)}/w_{i_r}$ at step $r$.
    We can use this inequality to bound $Q(G)$ by
    \[
        Q(G) = \sum_{r=1}^s (Q(G_r) - Q(G_{r+1})) \le \sum_{r=1}^s w_{i_r} \le \max_{S \in \I(G)} w(S) = \walpha(G, w) \enspace.
    \]
\end{proof}

\subsection{Directed Graph}

Compared to the result in the previous section, we are more generally interested in directed graphs.
We consider the case of directed, strongly observable support $\supp{\scG} = G = (V, E)$ with $V = [K]$ and $(i,i) \in E$ for all $i \in V$.
In the directed case, we distinguish the in-neighborhood $N^\text{in}(i)$ over $G$ of a vertex $i \in V$ from its out-neighborhood $N^\text{out}(i)$.
We use the convention that vertices with self-loops are not included in their neighborhoods, while all vertices are always included in their extended in-neighborhood $C^\text{in}(i) = N^\text{in}(i) \cup \{i\}$ and out-neighborhood $C^\text{out}(i) = N^\text{out}(i) \cup \{i\}$, respectively.
We make this distinction to comply as much as possible with previous works providing analogous results \citep{AlonCGMMS17}, where the neighborhoods $\inneigh{}{i}$ and $\outneigh{}{i}$ did not include $i$ even in the presence of the self-loop $(i,i) \in E$.

The weighted independence number is defined in the same way as per undirected graphs, ignoring the direction of edges for the independence condition.
Here we define in two slightly different manners the vertex weights: let
\begin{equation} \label{eq:walpha-weigths-in}
w^\text{in}_{\scG}(i) = w^\text{in}_i = \Biggl(\frac1{\abs{C^\text{in}(i)}} \sum_{j \in C^\text{in}(i)} p(j,i)\Biggr)^{-1}
\end{equation}
be the inverse of the arithmetic mean of the incoming edge probabilities for $i$, and
\begin{equation} \label{eq:walpha-weigths-out}
w^\text{out}_{\scG}(i) = w^\text{out}_i = \Biggl(\frac1{\abs{C^\text{out}(i)}} \sum_{j \in C^\text{out}(i)} p(i,j)\Biggr)^{-1}
\end{equation}
the analogous over outgoing edges.
These two different assignments of vertex weights induce two weighted independence numbers $\walpha(G, w^\text{in})$ and $\walpha(G, w^\text{out})$, respectively.

Then, we prove a lemma similar to \citep[{Lemma~13}]{AlonCGMMS17} in the weighted case.
Note, however, that in this case the lemma is tightly related to the specific definitions of vertex weights we are adopting.
\begin{lemma} \label{lem:turan-weights-directed}
    Let $G = (V, E)$ be a directed graph with edge probabilities $p(i,j) \in [0,1]$, and positive vertex weight vectors $w^\textup{in}$ and $w^\textup{out}$ as in \Cref{eq:walpha-weigths-in,eq:walpha-weigths-out}, respectively. Then,
    \[
        \sum_{i \in V} \frac{w^\textup{in}_i}{\abs{C^\textup{in}(i)}} \le 3(\walpha(G, w^\textup{in}) + \walpha(G, w^\textup{out})) \ln(K+1) \enspace.
    \]
\end{lemma}
\begin{proof}
    We prove the statement by induction as in the proof of \citet[Lemma~13]{AlonCGMMS17}.
    Consider the following algorithm: as long as the graph is not empty, repeatedly choose the vertex $j$ that maximizes $\abs{C^\text{in}(j)}/w^\text{in}_j$ among all remaining vertices and remove it from the graph along with its incident edges.
    Let $i_1, \dots, i_K$ be the vertices in the order the algorithm picks them.
    Additionally, let $G_1, \dots, G_{K+1}$ be the sequence of graphs generated by this iterative procedure, where $G_1 = G$ is the original graph and $G_{K+1}$ is the empty graph.
    We also let $C_r^\text{in}(i)$ denote the extended in-neighborhood over $G_r$ of any $i \in V(G_r)$.
    Similarly to the proof of \Cref{lem:turan-weights-undirected}, define
    \[
        Q(H) = \sum_{i \in V(H)} \frac{w^\text{in}_i}{\abs{C^\text{in}(i)}} \qquad \forall H \subseteq G
    \]
    as the quantity we want to bound for $G$, where the size of the in-neighborhood is always computed with respect to the starting graph $G$.
    
    Define a new instance of the problem with graph $G' = (V, E')$ as the undirected version of $G$, where the edge probabilities are defined as $p'(i,j) = \frac12 p(i,j) + \frac12 p(j,i)$ for all $i,j\in V$ such that either $(i,j) \in E$ or $(j,i) \in E$.
    This new graph has $C(i) = C^\text{in}(i) \cup C^\text{out}(i)$.
    As a consequence, we can derive new vertex weights $w'_i = \bigl(\frac1{\abs{C(i)}} \sum_{j \in C(i)} p'(j,i)\bigr)^{-1}$.
    This instance is such that
    \begin{equation} \label{eq:turan-weights-directed-reduction-1}
       \sum_{i\in V} \frac{\abs{C(i)}}{w'_i}
       = \sum_{i \in V} \sum_{j \in C(i)} p'(j,i)
       = \sum_{i \in V} \sum_{j \in C^\text{in}(i)} p(j,i)
       = \sum_{i \in V} \frac{\abs{C^\text{in}(i)}}{w^\text{in}_i} \enspace.
    \end{equation}
    Furthermore, notice that the newly defined vertex weights satisfy
    \begin{align}
        w'_i = \frac{\abs{C(i)}}{\sum_{j \in C(i)} p'(j,i)}
        &\le \frac{\abs{C^\text{in}(i)}}{\sum_{j \in C(i)} p'(j,i)} + \frac{\abs{C^\text{out}(i)}}{\sum_{j \in C(i)} p'(j,i)} \nonumber\\
        &\le \frac{2\abs{C^\text{in}(i)}}{\sum_{j \in C^\text{in}(i)} p(j,i)} + \frac{2\abs{C^\text{out}(i)}}{\sum_{j \in C^\text{out}(i)} p(i,j)} \nonumber\\
        &= 2(w^\text{in}_i + w^\text{out}_i) \enspace. \label{eq:turan-weights-directed-reduction-2}
    \end{align}

    Consider now the first vertex $i_1$ chosen by the procedure we introduced before.
    The value it maximizes is lower bounded by
    \begin{align}
        \max_{i \in V} \frac{\abs{C^\text{in}(i)}}{w^\text{in}_i}
        &\ge \frac1K \sum_{i\in V} \frac{\abs{C^\text{in}(i)}}{w^\text{in}_i} \nonumber\\
        &= \frac1K \sum_{i\in V} \frac{\abs{C(i)}}{w'_i} && \text{by \Cref{eq:turan-weights-directed-reduction-1}} \nonumber\\
        &\ge \frac{K}{\sum_{i \in V} \frac{w'_i}{\abs{C(i)}}} && \text{by Jensen's inequality} \nonumber\\
        &\ge \frac{K/2}{\sum_{i \in V} \frac{w^\text{in}_i}{\abs{C(i)}} + \sum_{i \in V} \frac{w^\text{out}_i}{\abs{C(i)}}} && \text{by \Cref{eq:turan-weights-directed-reduction-2}} \nonumber\\
        &\ge \frac{K/2}{\walpha(G, w^\text{in}) + \walpha(G, w^\text{out})} \enspace. && \text{by \Cref{lem:turan-weights-undirected} over $G'$} \label{eq:turan-weights-directed-1}
    \end{align}
    We can use this fact to show an upper bound for the sum $Q(G)$ as
    \begin{align*}
        Q(G) = \sum_{i\in V} \frac{w^\text{in}_i}{\abs{C^\text{in}(i)}}
        &= \frac{w^\text{in}_{i_1}}{\abs{C^\text{in}(i_1)}} + \sum_{r=2}^K \frac{w^\text{in}_{i_r}}{\abs{C^\text{in}(i_r)}} \\
        &\le \frac{2(\walpha(G, w^\text{in}) + \walpha(G, w^\text{out}))}{K} + Q(G_2) \enspace. && \text{by \Cref{eq:turan-weights-directed-1}}
    \end{align*}
    As a last step, recursively repeat the same reasoning on $Q(G_2)$ and iterate it until reaching $G_K$ to conclude that
    \[
        Q(G) \le 2\sum_{r=1}^K \frac{\walpha(G_r, w^\text{in}) + \walpha(G_r, w^\text{out})}{K-r+1}
        \le 3(\walpha(G, w^\text{in}) + \walpha(G, w^\text{out})) \ln(K+1) \enspace.
    \]
\end{proof}

We finally have all the tools required for demonstrating the next lemma.
It essentially corresponds to \citet[Lemma~5]{AlonCDK15} with the addition of edge probabilities.
The main difference is that we show an upper bound in terms of two distinct independence numbers.
They are both computed over the graph $G$ with vertex weights defined in terms of the worst-case edge probabilities.
To be specific, we have a first weight assignment $w^-$ to vertices such that $w_{\scG}^-(i) = w^-_i = \left(\min_{j\in C^\text{in}(i)} p(j,i)\right)^{-1}$ is the reciprocal of the minimum incoming edge probability for vertex $i$.
The second assignment $w^+$, instead, assigns weight $w_{\scG}^+(i) = w^+_i = \left(\min_{j\in C^\text{out}(i)} p(i,j)\right)^{-1}$ equal to the inverse of the minimum outgoing edge probability for $i$.



\begin{lemma}\label{lem:weightedindependence}
    Let $G = (V, E)$ be a directed graph with $|V| = K \ge 2$ and edge probabilities $p(i,j)$, and such that $(i,i) \in E$ for all $i \in V$.
    Let $z_i \in \R_+$ be a positive weight assigned to each $i \in V$.
    Assume that $\sum_{i \in V} z_i \le 1$ and that $z_i \ge \beta$ for all $i \in V$, given some constant $\beta \in (0, \half]$.
    Then,
    \[
        \sum_{i \in V} \frac{z_i}{\sum_{j \in C^\textup{in}(i)} z_j p(j,i)} \le 6 (\walpha(G, w^-) + \walpha(G, w^+)) \ln\left(\frac{2K^2}{\beta \rho}\right) \enspace,
    \]
    where $\rho = \min_{i \in V} \sum_{j \in C^\textup{in}(i)} p(j,i) > 0$.
\end{lemma}
\begin{proof}
    The structure of this proof is similar to that of \citet[Lemma~5]{AlonCDK15}.
    Define a discretization of $z_1, \dots, z_K$ such that $(m_i - 1)/M \le z_i \le m_i/M$ for positive integers $m_1, \dots, m_K$ and $M = \ceil*{\frac{2K}{\beta\rho}}$.
    The discretized values are such that, for all $i \in V$,
    \begin{equation} \label{eq:second-order-term-random-graph-1}
        \sum_{j \in C^\text{in}(i)} m_j p(j,i) \ge M \sum_{j \in C^\text{in}(i)} z_j p(j,i) \ge \frac{2K}{\beta\rho} \beta \sum_{j \in C^\text{in}(i)} p(j,i) \ge 2K \ge 2 \abs{C^\text{in}(i)} \enspace,
    \end{equation}
    where the first inequality holds because $z_j \le m_j/M$, the second follows by definition of $M$ and by the assumption on $z_j$, whereas the third is due to the definition of $\rho$.
    Then, the sum of interest becomes
    \begin{align}
        \sum_{i \in V} \frac{z_i}{\sum_{j \in C^\text{in}(i)} z_j p(j,i)}
        &\le \sum_{i \in V} \frac{m_i}{M \sum_{j \in C^\text{in}(i)} z_j p(j,i)} && \text{since $z_i \le m_i/M$} \nonumber\\
        &\le \sum_{i \in V} \frac{m_i}{\sum_{j \in C^\text{in}(i)} m_j p(j,i) - \abs{C^\text{in}(i)}} && \text{since $M z_j \ge m_j-1$} \nonumber\\
        &\le 2 \sum_{i \in V} \frac{m_i}{\sum_{j \in C^\text{in}(i)} m_j p(j,i)} && \text{by \Cref{eq:second-order-term-random-graph-1}}. \label{eq:second-order-term-random-graph-2}
    \end{align}
    
    Now build a new directed graph $G' = (V', E')$ derived (as in the proof of \citet[Lemma~5]{AlonCDK15}) from graph $G$ by replacing each node $i\in V$ with a clique $K_i$ of size $m_i$ and all its edges having probability $p(i,i)$.
    Additionally add an edge from any $i' \in K_i$ to any $j' \in K_j$ having edge probability $p(i,j)$ if and only if $(i,j) \in E$.
    As a consequence, the right-hand side of \Cref{eq:second-order-term-random-graph-2} is equal to
    \[
        2\sum_{i \in V'} \frac{1}{\sum_{j \in C^\text{in}_{G'}(i)} p(j,i)} \enspace.
    \]
    Observe that the independent sets in $G$ are preserved in $G'$: any independent set $S = \{i : i \in V'\} \in \mathcal{I}(G')$ in $G'$ has a corresponding one $\{i : i'\in S, i' \in K_i\}$ in $G$ with same cardinality and weight, assuming that the weight of $i' \in K_i$ in $G'$ is equal to the weight of $i \in V$ according to the weight assignment in $G$.
    We can reduce this latter sum to the same form as in \Cref{lem:turan-weights-directed} by assigning vertex weights
    \begin{align*}
        w^\text{in}_{i'} = \Biggl(\sum_{j \in C^\text{in}(i)} \frac{m_j}{\sum_{k \in C^\text{in}(i)} m_k} p(j,i)\Biggr)^{-1} ,
        &&
        w^\text{out}_{i'} = \Biggl(\sum_{j \in C^\text{out}(i)} \frac{m_j}{\sum_{k \in C^\text{out}(i)} m_k} p(i,j)\Biggr)^{-1} ,
    \end{align*}
    to each vertex $i' \in K_i$, for all $i \in V$.
    Indeed, the previous sum becomes
    \begin{align}
        \sum_{i \in V'} \frac{1}{\sum_{j \in C^\text{in}_{G'}(i)} p(j,i)}
        &= \sum_{i \in V'} \frac{w^\text{in}_{i}}{\abs{C^\text{in}_{G'}(i)}} \nonumber\\
        &\le 3(\walpha(G', w^\text{in}) + \walpha(G', w^\text{out})) \ln(\abs{V'}+1) && \text{by \Cref{lem:turan-weights-directed}} \nonumber\\
        &\le 3(\walpha(G, w^-) + \walpha(G, w^+)) \ln(\abs{V'}+1) \enspace, \nonumber
    \end{align}
    where the last inequality follows from the fact that $w^\text{in}_{i'} \le w^-_{i}$ and $w^\text{out}_{i'} \le w^+_{i}$ for all $i \in V$ and all $i' \in K_i$.

    We conclude the proof by observing that this newly constructed graph also has 
    \[
        1 + \abs{V'} = 1 + \sum_{i \in V} m_i \le 1 + \sum_{i \in V} (M z_i + 1) \le K + M + 1 \le 2K\left(1 + \frac1{\beta\rho}\right) \le \frac{2K^2}{\beta\rho}
    \]
    vertices, where the final inequality holds because $\beta\rho \le K/2$ by definition, and we used the fact that $K \ge 2$.
\end{proof}


\end{document}

%% file: packages.tex
\usepackage[preprint]{neurips_2022}

\usepackage[utf8]{inputenc} 
\usepackage[T1]{fontenc}    
\usepackage{url}            
\usepackage{booktabs}       
\usepackage{amsfonts}       
\usepackage{nicefrac}       
\usepackage{microtype}      
\usepackage[dvipsnames]{xcolor}         

\usepackage[T1]{fontenc}
\usepackage[utf8]{inputenc}
\usepackage{lmodern}

\usepackage{graphicx,color,comment}
\usepackage{amssymb,tikz}
\usepackage{amsthm}
\usepackage{amsmath}
\usepackage{nicefrac}
\usepackage[hidelinks]{hyperref}
\usepackage[ruled,vlined,noend,noline]{algorithm2e}
\usepackage{cleveref}
\usepackage{pgf}
\usepackage{subcaption}
\usepackage{todonotes}
\usepackage[inline,shortlabels]{enumitem}
\usepackage{mathtools,xspace}
\usepackage{tikz}
\usepackage{bm}
\usetikzlibrary{shapes,arrows,positioning}
\usetikzlibrary{arrows,automata}

\usepackage{graphicx} 
\graphicspath{{./figures/}} 

\usepackage{enumitem}
\usepackage{thmtools, thm-restate}

%% file: macros.tex
\newcommand{\A}{\mathcal{A}} 
\newcommand{\R}{\mathbb{R}} 
\newcommand{\I}{\mathcal{I}} 
\newcommand{\E}[1]{\mathbb{E} \left[ #1 \right]} 
\renewcommand{\P}[1]{\mathbb{P} \left( #1 \right)} 
\newcommand{\ind}[1]{\mathbb{I}_{#1}} 
\newcommand{\F}{\mathcal{F}} 
\newcommand{\D}{\mathcal{D}} 
\newcommand{\calE}{\mathcal{E}} 

\newcommand{\G}{\mathcal{G}} 
\newcommand{\scS}{\mathcal{S}}

\newcommand{\scC}{\mathcal{C}}
\newcommand{\e}{\varepsilon}

\DeclareMathOperator*{\argmin}{arg\,min}

\newcommand{\polylog}{\operatorname{polylog}}
\DeclarePairedDelimiter{\abs}{|}{|}
\DeclarePairedDelimiter{\floor}{\lfloor}{\rfloor}
\DeclarePairedDelimiter{\ceil}{\lceil}{\rceil}
\newcommand{\outneigh}[2]{N^\textup{out}_{#1}(#2)}
\newcommand{\inneigh}[2]{N^\textup{in}_{#1}(#2)}
\newcommand{\Tbad}{T_\textup{bad}}
\newcommand{\half}{\tfrac{1}{2}}
\newcommand{\Pp}{P}
\newcommand{\Prob}{\mathbb{P}}
\newcommand{\Enb}{\mathbb{E}} 
\newcommand{\sumT}{\sum_{t=1}^T} 
\newcommand{\KL}{D_\textup{KL}}
\newcommand{\phattji}{\hat{p}_t(j,i)}
\newcommand{\phattij}{\hat{p}_t(i,j)}
\newcommand{\id}{\ind}

\newcommand{\idseeit}{\id{\{i \in \outneigh{G_t}{I_t}\} \cap \{i \in \outneigh{\hat{G}_t}{I_t}\}}} 

\newcommand{\ltil}{\tilde{\ell}}
\newcommand{\ltilt}{\tilde{\ell}_t}
\newcommand{\pji}{p(j, i)}
\newcommand{\sumK}{\sum_{i = 1}^K}
\newcommand{\lhat}{\hat{\ell}}
\newcommand{\lhatt}{\hat{\ell}_t}

\newcommand{\qti}{q_t(i)}
\newcommand{\ui}{u(i)}

\newcommand{\lbart}{\bar{\ell}_t}

\newcommand{\ptilt}{\tilde{p}_t}
\newcommand{\BADEVENT}{\overline{\kappa}}

\newcommand{\sumTt}{\sum_{t=2}^T}

\newcommand{\pmint}{p_{t}^{\min}}
\newcommand{\pmins}{p_{s}^{\min}}
\newcommand{\mye}{\mathrm{e}}

\newcommand{\sumTstop}{\sum_{t=1}^{\tstar}}
\newcommand{\sumTsto}{\sum_{t=2}^{\tstar}}

\newcommand{\tstar}{{t^{\star}}}
\newcommand{\logfactor}{60\ln(KT)}
\newcommand{\twicelogfactor}{120\ln(KT)}
\newcommand{\beste}{{\e^\star_\Lambda}}

\newcommand{\sumTweak}{\sum_{t=\tstar + 1}^{T}}
\newcommand{\sumTwea}{\sum_{t=\tstar + 1}^{T}}
\DeclareRobustCommand{\VAN}[3]{#2} 

\newcommand{\Bern}[1]{\operatorname{\textsf{Bern}}\left(#1\right)}
\newcommand{\UCBt}{\sqrt{\frac{2\tilde{p}_t(j,i)}{t-1}\ln(3K^2T^2)} + \frac{3}{t-1}\ln(3K^2T^2)}

\newcommand{\UCBthat}{\sqrt{\frac{2\phattji}{t-1}\ln(3K^2T^2)} + \frac{3}{t-1}\ln(3K^2T^2)}

\newcommand{\ptiltji}{\frac{1}{t-1}\sum_{s = 1}^{t-1}\ind{\{(j, i) \in E_s\}}}

\DeclareBoldMathCommand{\1}{1}

\newcommand{\ftau}{\hat \tau}
\newcommand{\opte}{\e^*}
\newcommand{\optG}{\scG^*}
\newcommand{\lEtau}{E^+_{\tau}}
\newcommand{\sEtau}{E^-_{\tau}}
\newcommand{\epslEt}{E^+_{t}}
\newcommand{\epssEt}{E^-_{t}}

\newcommand{\walpha}{\alpha_\mathsf{w}}
\newcommand{\wdelta}{\delta_\mathsf{w}}

\newcommand{\expg}{\textup{\textsc{Exp3.G}}\xspace}

\newcommand{\round}{\textup{\textsc{RoundRobin}}\xspace}
\newcommand{\block}{\textup{\textsc{BlockReduction}}\xspace}

\newcommand{\explore}{\textup{\textsc{EdgeCatcher}}\xspace}

\newcommand{\etc}{\textup{\textsc{EdgeCatcher}}\xspace}
\newcommand{\br}{\textup{\textsc{BR}}\xspace}
\newcommand{\optimist}{\textup{\textsc{OptimisticThenCommitGraph}}\xspace}
\newcommand{\otc}{\textup{\textsc{otcG}}\xspace}

\DeclareMathOperator\suppop{supp}

\newcommand{\supp}[1]{\suppop(#1)}

\newcommand{\scG}{\mathcal{G}}

\newcommand{\te}{\tilde \e}
\newcommand{\ttau}{\tilde \tau}

\newcommand{\idseeitx}{X_t}

\renewcommand{\log}{\ln}

\renewcommand{\kappa}{\mathcal{K}}

\newcommand{\np}{\textsc{NP}}

\newcommand*\samethanks[1][\value{footnote}]{\footnotemark[#1]}

%% file: environments.tex
\theoremstyle{definition}
\newtheorem*{examplex*}{Example}
\newtheorem{examplex}{Example}
\newenvironment{example*}
  {\pushQED{\qed}\begin{examplex*}}
  {\popQED\end{examplex*}}
\newenvironment{example}
  {\pushQED{\qed}\examplex}
  {\popQED\endexamplex}

\theoremstyle{plain}

\newtheorem{lemma}{Lemma}
\newtheorem{theorem}{Theorem}

\newtheorem{definition}{Definition}